\documentclass[12pt]{article}

\usepackage{fullpage,times,url,bm}

\usepackage{amsthm,amsfonts,amsmath,amssymb,mathtools,epsfig,color,float,graphicx,verbatim}
\usepackage{algorithm,algorithmic}
\usepackage{bbm}
\usepackage[square,numbers]{natbib}
\usepackage{epstopdf}
\usepackage{authblk}
\usepackage{comment}

\usepackage{hyperref}
\usepackage{cleveref}[2012/02/15]% v0.18.4; 
\crefformat{footnote}{#2\footnotemark[#1]#3}

\hypersetup{
	colorlinks   = true, %Colours links instead of ugly boxes
	urlcolor     = blue, %Colour for external hyperlinks
	linkcolor    = blue, %Colour of internal links
	citecolor   = black %Colour of citations
}

\newtheorem{theorem}{Theorem}[section]
\newtheorem{proposition}[theorem]{Proposition}
\newtheorem{lemma}[theorem]{Lemma}

\newtheorem{assumption}[theorem]{Assumption}

\newcommand{\reals}{\mathbb{R}}
\newcommand{\E}{\mathbb{E}}

\newcommand{\sign}{\mathrm{sign}}

\newcommand{\relu}[1]{\left[ #1 \right]_+}
\newcommand{\set}[1]{\left\{#1\right\}}

\newcommand{\ba}{\mathbf{a}}
\newcommand{\be}{\mathbf{e}}
\newcommand{\bx}{\mathbf{x}}
\newcommand{\bw}{\mathbf{w}}

\newcommand{\bb}{\mathbf{b}}
\newcommand{\bu}{\mathbf{u}}
\newcommand{\bv}{\mathbf{v}}
\newcommand{\bz}{\mathbf{z}}

\newcommand{\by}{\mathbf{y}}

\newcommand{\bo}{\mathbf{o}}

\newcommand{\btheta}{\boldsymbol{\theta}}

\newcommand{\Lcal}{\mathcal{L}}
\newcommand{\Ocal}{\mathcal{O}}

\newcommand{\Dcal}{\mathcal{D}}
\newcommand{\Fcal}{\mathcal{F}}
\newcommand{\Hcal}{\mathcal{H}}

\newcommand{\Ncal}{\mathcal{N}}

\newcommand{\Pcal}{\mathcal{P}}

\newcommand{\norm}[1]{\left\|#1\right\|}
\newcommand{\inner}[1]{\left\langle#1\right\rangle}
\newcommand{\p}[1]{\left(#1\right)}
\newcommand{\pcc}[1]{\left[#1\right]}
\newcommand{\abs}[1]{\left|#1\right|}

\newcommand{\one}[1]{\mathbbm{1}\left\{#1\right\}}
\newcommand{\pr}{\mathbb{P}}
\newcommand{\crelu}[1]{\left[#1\right]_{+}^{1}}
\DeclareMathOperator{\erf}{erf}
\DeclareMathOperator{\ierf}{erf^{-1}}

\newcommand{\secref}[1]{Sec.~\ref{#1}}
\newcommand{\subsecref}[1]{Subsection~\ref{#1}}

\renewcommand{\eqref}[1]{Eq.~(\ref{#1})}
\newcommand{\lemref}[1]{Lemma~\ref{#1}}

\newcommand{\thmref}[1]{Thm.~\ref{#1}}
\newcommand{\propref}[1]{Proposition~\ref{#1}}
\newcommand{\appref}[1]{Appendix~\ref{#1}}

\makeatletter
\def\moverlay{\mathpalette\mov@rlay}
\def\mov@rlay#1#2{\leavevmode\vtop{%
   \baselineskip\z@skip \lineskiplimit-\maxdimen
   \ialign{\hfil$\m@th#1##$\hfil\cr#2\crcr}}}
\newcommand{\charfusion}[3][\mathord]{
    #1{\ifx#1\mathop\vphantom{#2}\fi
        \mathpalette\mov@rlay{#2\cr#3}
      }
    \ifx#1\mathop\expandafter\displaylimits\fi}
\makeatother

\makeatletter
\newcommand{\printfnsymbol}[1]{%
  \textsuperscript{\@fnsymbol{#1}}%
}
\makeatother

\title{Optimization-Based Separations for Neural Networks}

\author{Itay Safran}
\author{Jason D.\ Lee}
\affil{Princeton University}%\\\texttt{\{isafran,jasonlee\}@princeton.edu}}
\date{}

\ifdefined\usebigfont
\usepackage{times}
\usepackage[fontsize=13pt]{scrextend}
\usepackage[left=1.56in,right=1.56in,top=1.71in,bottom=1.77in]{geometry}
\usepackage[shortlabels]{enumitem}
\setlist[itemize]{leftmargin=*}
\setlist[enumerate]{leftmargin=*}
\else
\fi

\begin{document}

\maketitle

\begin{abstract}
    %Depth separation results propose a possible theoretical explanation for the benefits of deep neural networks over shallower architectures, establishing that the former possess superior approximation capabilities. However, there are no known results in which the deeper architecture leverages this advantage into a provable optimization guarantee. We prove that when the data are generated by a distribution with radial symmetry which satisfies some mild assumptions, gradient descent can efficiently learn ball indicator functions using a depth 2 neural network with two layers of sigmoidal activations, and where the hidden layer is held fixed throughout training. Since it is known that ball indicators are hard to approximate with respect to a certain heavy-tailed distribution when using depth 2 networks with a single layer of non-linearities \citep{safran2017depth}, this establishes what is to the best of our knowledge, the first optimization-based separation result where the approximation benefits of the stronger architecture provably manifest in practice. Our proof technique relies on a random features approach which reduces the problem to learning with a single neuron, where new tools are required to show the convergence of gradient descent when the distribution of the data is heavy-tailed.
    Depth separation results propose a possible theoretical explanation for the benefits of deep neural networks over shallower architectures, establishing that the former possess superior approximation capabilities. However, there are no known results in which the deeper architecture leverages this advantage into a provable optimization guarantee. We prove that when the data are generated by a distribution with radial symmetry which satisfies some mild assumptions, gradient descent can efficiently learn ball indicator functions using a depth 2 neural network with two layers of sigmoidal activations, and where the hidden layer is held fixed throughout training. By building on and refining existing techniques for approximation lower bounds of neural networks with a single layer of non-linearities, we show that there are $d$-dimensional radial distributions on the data such that ball indicators cannot be learned efficiently by any algorithm to accuracy better than $\Omega(d^{-4})$, nor by a standard gradient descent implementation to accuracy better than a constant. These results establish what is to the best of our knowledge, the first optimization-based separations where the approximation benefits of the stronger architecture provably manifest in practice. Our proof technique introduces new tools and ideas that may be of independent interest in the theoretical study of both the approximation and optimization of neural networks.
\end{abstract}

	\section{Introduction}
	In recent years, several theoretical papers have provided a possible explanation for the benefits of using deep neural networks over shallower architectures, by proving \emph{depth separation} results \citep{eldan2016power,telgarsky2016benefits,poole2016exponential,daniely2017depth,yarotsky2017error,liang2017deep,safran2017depth,poggio2017and,safran2019depth,venturi2021depth}. Simply put, these results show the existence of some function that can be approximated efficiently by a neural network with a certain depth, whereas a shallower architecture might require exponentially many more neurons to achieve the same accuracy. Such results may suggest that the reason that deeper networks perform better in practice stems from their superior \emph{approximation} capabilities. In contrast, the network configuration that is achieved in practice when a neural network is trained is strongly dictated by the learning algorithm used in its \emph{optimization}. While it is clear that the ability of a certain architecture to approximate a function efficiently is a necessary condition for being able to learn it successfully, it is not obvious to what extent - if at all - this is also a sufficient condition. Indeed, several recent works have shown that in some cases, the same `extreme' properties that give rise to the inapproximability of these hard functions when using shallower architectures may also prove detrimental for optimization using the deeper architecture \citep{malach2019deeper,malach2021connection}. This suggests that certain results used to demonstrate the benefits of depth such as the one shown in \citet{telgarsky2016benefits} are in a sense `too strong', and that arguably, one of the main incentives for studying the approximation power of neural networks (namely, its necessity for proving a positive optimization result) is absent in such cases. Moreover, if inapproximability using a shallow architecture may in fact turn out to imply the inability to optimize efficiently using the deep architecture in a more general setting, then this would suggest that the approximation capabilities of a neural network have little impact on its optimization, perhaps further demotivating the study of neural network approximation from a practical perspective.
	
	In light of the above, to theoretically explain the practical success of depth via the lens of function approximation, we would need to show that the benefits of depth will also manifest in practice. While a learning algorithm can simply deterministically return the efficient approximation of the hard function by utilizing depth for a specific problem setting, this is unsatisfactory; such a setting would not explain why deeper neural networks achieve better results in practice, since these are typically trained using first-order optimization methods such as gradient descent (GD) or its more sophisticated variants. We are now lead to the following question: 
	\begin{quote}
		\emph{Can we show an optimization-based separation result, in which a simple architecture cannot approximate efficiently %(and therefore there exists no efficient learning algorithm that returns such an architecture)
		(implying the hardness of learning with such architecture), yet a stronger architecture can provably learn with a polynomial sample size, network size and running time, using a first-order optimization method?}
	\end{quote}
	In this paper, we answer this question in the affirmative by proving that a neural network with a single hidden layer with sigmoidal activations and an additional non-linearity in the output neuron can learn ball indicator functions efficiently (functions of the form $\bx\mapsto1$ if and only if $\norm{\bx}_2\le \lambda$ for some $\lambda>0$). In contrast, by using a reduction technique of \citet{safran2017depth}, we show that ball indicators cannot be approximated efficiently to accuracy better than $\Omega(d^{-4})$ using depth 2 neural networks when we are not allowed to use an additional non-linearity in the output neuron. Moreover, using a technique introduced in \citet{daniely2017depth}, we provide new lower bounds for approximating ball indicators, showing that getting better than constant accuracy requires that either the width or the Euclidean norm of the weights of the network grow exponentially in either $d$ or $\frac{1}{\sqrt{\epsilon}}$ ($\epsilon$ being the desired accuracy). By bounding the progress of GD in a standard setting, this leads to a similar exponential lower bound on its running time. Overall, these establish what is to the best of our knowledge, the first separation results where the benefits of the stronger class of networks for function approximation manifests in practice and facilitates efficient optimization. 
	
	The remainder of this paper is structured as follows: After presenting our contributions in more detail below, we turn to discuss related work in \subsecref{subsec:related_work}. In \secref{sec:prelims}, we first introduce relevant definitions and notation before we present our setting and problem formulation in \subsecref{subsec:notation}. Thereafter, our main results are presented in \subsecref{subsec:main_result}, followed by a formal definition of the distributions they use in \subsecref{subsec:construction}. \secref{sec:lower_bounds} contains our approximation lower bounds for ball indicator functions that our main results build upon. In \secref{sec:learning}, we turn to show a positive optimization result for learning ball indicators.

	\subsection{Our Contributions}
	\begin{itemize}
		\item
		We show that there exist a sequence $\{\Dcal_d\}_{d=2}^{\infty}$ of $d$-dimensional heavy-tailed distributions and a sequence of constants $\{\lambda_d\}_{d=2}^{\infty}$ where $\lambda_d\in[1,2]$ for all $d$, such that no neural network of depth 2 with a single layer of non-linearities and width less than $\Omega\p{\exp(\Omega(d))}$ can approximate ball indicator functions with radii $\{\lambda_d\}_{d=2}^{\infty}$ to accuracy better than $\Omega(d^{-4})$ (\thmref{thm:inapproximability}). %We remark that it was previously shown that ball indicator functions are hard to approximate in the high-accuracy regime of $\Omega(d^{-4})$ when using depth 2 neural networks in \citet{safran2017depth}, however some work is required to adapt a similar analysis to our setting in order to prove our result (see the related work subsection for a more detailed comparison).
		
		\item
		In a slightly different approximation setting, we show that there exist a sequence $\{\Dcal'_d\}_{d=2}^{\infty}$ of $d$-dimensional compactly-supported distributions and a sequence of constants $\{\lambda_d\}_{d=2}^{\infty}$ where $\lambda_d\in[1,2]$ for all $d$, such that any neural neural network of depth 2 with a single layer of non-linearities which achieves accuracy $\Ocal(m^{-2})$ for natural $m\ge1$ on ball indicators with radii $\{\lambda_d\}_{d=2}^{\infty}$ must have width or weight norm exponential in either $d$ or $m$ (\thmref{thm:daniely_inapproximability}). %Our lower bound separates the dependence on the input dimension from the accuracy parameter, allowing us to provide a hardness result that persists beyond the high-accuracy regime in our previous lower bound, which holds even for a (sufficiently small) constant accuracy requirement.
		
		\item
		We prove that under certain assumptions on the distribution of the data (Assumption~\ref{asm:d_dist_assumption}), a neural network with two layers of non-linear activations (see \eqref{eq:architecture} for the formal architecture) can efficiently learn data labeled by a ball indicator function with radius in $[1,2]$, and attain arbitrarily small population loss by using GD with a standard initialization (Assumption~\ref{asm:init}), where the hidden layer is held fixed throughout training (\thmref{thm:gd_convergence}). %Our analysis relies on a random features approach which reduces the problem to learning using a single neuron, and introduces new tools which may be of independent interest in the study of the optimization of neural networks. %The main challenge in the proof is to facilitate the analysis to work with heavy-tailed distributions such as the one giving rise to the inapproximability result in \thmref{thm:inapproximability}, which leads to a much more challenging learning problem.
		
		\item
		Welding our lower bounds together with the positive optimization result, we derive our main theorems (\thmref{thm:main} and \thmref{thm:secondary}) which are optimization-based separation results. We show that no algorithm returning a depth 2 network with a single layer of non-linearities can efficiently learn ball indicators with respect to $\{\Dcal_d\}_{d=2}^{\infty}$ to accuracy better than $\Omega(d^{-4})$, whereas optimizing using a standard GD setting on a depth 2 network cannot efficiently learn ball indicators with respect to $\{\Dcal'_d\}_{d=2}^{\infty}$ to more than constant accuracy with overwhelming probability. In contrast, optimizing using a standard GD setting on our stronger architecture which employs more than a single layer of non-linear activations will succeed in learning with respect to both distributions with high probability.
	\end{itemize}
	Next, we turn to discuss and compare related work in the literature that is most relevant to ours.
	
	\subsection{Related Work}\label{subsec:related_work}
	
		\textbf{Depth Separations in Neural Networks.} For concreteness, let us focus here on works that separate depth 2 from 3 where the target function and the domain of approximation are continuous. In their seminal work, \citet{eldan2016power} provided the first such depth separation. The authors show that a certain radial function consisting of a superposition of polynomially many thin shells cannot be approximated to accuracy better than an absolute constant using a depth 2 network, unless the width is exponential in the input dimension. Building on their work, \citet{safran2017depth} provide a more natural separating target function by reducing the more complicated hard function to a ball indicator. Additionally, the authors also demonstrate empirically that depth 3 networks outperform depth 2 networks in learning ball indicators, even if the shallow network has many more trainable weights. While we build on a similar proof strategy as in the aforementioned paper, our work differs from theirs in that they show that for a ball indicator function of any radius, there exists a distribution under which it is hard to approximate using depth 2 networks, whereas our work shows that for a specific distribution there exists a ball indicator function with radius in the interval $[1,2]$ which is hard to approximate using a depth 2 network. Moreover, our main result rigorously proves the learnability of the target function used in their experiment. An additional similar separation is given by \citet{daniely2017depth}, where the approximation is in a compact domain and the target function oscillates polynomially many times in the input dimension, and where an exponential bound on the magnitude of the weights of the depth 2 network is imposed. \citet{venturi2021depth} provide separation results with respect to product distributions on the data and target functions that do not necessarily possess radial symmetry.
	
		\textbf{Unlearnability of Depth Separations.}
		As discussed earlier in the introduction, there are known cases where the same property that gives rise to a negative approximation result for shallow networks is also detrimental for optimization using the deeper architecture. \citet{malach2019deeper} have shown a depth separation result on fractal-like distributions with a natural coarse-to-fine structure that prompts the use of depth to achieve efficient approximation. On the other hand, they show that such distributions are intrinsically difficult to learn using gradient methods since the gradient will be exponentially close to zero with overwhelming probability, even if the architecture being optimized is able to achieve zero loss on the distribution. Building on a similar technique, \citet{malach2021connection} show that if a depth 3 neural network cannot approximate a target function efficiently, then a deeper architecture will also be prone to failure if optimized using gradient methods. As a specific corollary, this implies that the depth separation result shown in \citet{telgarsky2016benefits}, which exploits the ability of deep ReLU networks to compute piece-wise linear functions with a large number of segments, also results in a function that is hard to optimize over using gradient methods. In \citet{vardi2021size}, the authors show that there is a natural proof barrier for proving depth separations between networks of depth 4 and deeper architectures, when the target functions have `benign' properties, such as polynomially-bounded Lipschitz constant and being able to be computed efficiently. This, in a certain sense, suggests the opposite implication of \citep{malach2019deeper,malach2021connection}; that it might be difficult to show depth separations with functions that have some properties that are desired in practice to ensure learnability.
		
		\textbf{Learning a Single Neuron.} In our positive optimization result, by holding the weights of the hidden layer fixed throughout training, we essentially reduce our problem to that of learning with a single neuron using gradient methods. This setting has been studied in several recent works. \citet{yehudai2020learning,vardi2021learning} show the convergence of GD and some of its variants on the population loss objective, with constant or high probability and with or without a bias term, under various distributional assumptions on the data and in the realizable setting. In contrast, our work deals with the non-realizable case, assumes a ball indicator target function, and provides a convergence guarantee with high rather than constant probability for radial distributions. \citet{frei2020agnostic} provide an agnostic (in the non-realizable setting) analysis for learning with finite data while making few assumptions on the problem. However, their result does not imply linear convergence of GD as opposed to ours, and does not guarantee attaining loss which is arbitrarily close to the global minimum.
		
		\textbf{Neural Network Approximation Using Random Features.} Many works used the technique of random features to approximate various target functions, which is potentially used to derive positive optimization results for neural networks \citep{bach2017breaking,ji2019neural,bai2019beyond,allen2019learning,yehudai2019power,chen2020towards,ghorbani2021linearized}. A major difference that sets our work apart from these is that other works perform their approximation on a bounded domain, whereas we show the convergence of random features for heavy-tailed distributions that generate data with a very large norm. On the downside, we merely show the approximation of ball indicator target functions, whereas the works cited above show approximation of a family of polynomials or even more general target function classes.
		
		\textbf{Learning More than a Single Layer of Non-linearities.} The problem of learning a neural network which has more than a single layer of non-linearities has been addressed by several papers in recent years. In \citet{goel2019learning}, the authors study a similar architecture to the one we use to show a positive optimization result, where an additional non-linear activation is employed on the output neuron. In \citet{allen2019learning}, the authors prove that a certain class of depth 3 target functions can be improperly learned using depth 3 networks with a different architecture. \citet{allen2020backward} show that deep networks with quadratic activations can heirarchically learn certain concept classes. However, the learning algorithm used in \citet{goel2019learning} is different than GD, and all three papers crucially rely on the input having unit norm or some other form of boundedness, whereas our work accommodates for certain inputs whose norm distribution does not have finite moments. In contrast, our paper merely establishes convergence for the target class of ball indicator functions, and not for a broader family of target functions analyzed in the aforementioned papers. \citet{chen2020towards} show an optimization-based separation between a depth 3 neural network in a suitably-defined quadratic NTK regime and a network of arbitrary depth trained in the NTK regime, where the former learns polynomials of degree $p$ using width $d^{p/2}$ compared to $d^{p}$ of the latter. However, this separation does not exclude the existence of a different shallow learner that operates in the non-NTK regime, whereas our work excludes efficient learnability in a standard GD setting, or even more generally with any shallow model due to inapproximability, thus providing a stronger separation.

	\section{Preliminaries and Main Results}\label{sec:prelims}
	In this section, we formally define our problem setting before stating our main theorems. We will begin however with introducing some of the notation that will be used throughout the paper in the following subsection.
	
	\subsection{Notation and Terminology}\label{subsec:notation}
	We use $[n]$ as shorthand for the set $\{1,\ldots,n\}$. We denote vectors using bold-faced letters (e.g.\ $\bx$) and matrices or random variables using upper-case letters (e.g.\ $X$). Given a vector $\bx$, we let $\norm{\bx}_2$ denote its Euclidean norm, where the subscript is occasionally omitted when clear from context. Given a square matrix $A$, we denote its spectral norm by $\norm{A}_{\text{sp}}$. Given two non-zero vectors $\bw,\bv\in\reals^d$, we denote the angle between them using $\theta_{\bw,\bv} \coloneqq \arccos\p{\frac{\inner{\bw,\bv}}{\norm{\bw}\cdot\norm{\bv}}}$. We use $\one{\cdot}$ to denote the indicator function. A function $f:\reals^d\to\reals$ is \textit{radial} if for all $\bx,\bx'\in\reals^d$, $\norm{\bx}=\norm{\bx'}$ implies $f(\bx)=f(\bx')$. A function $f:\reals^d\to\reals$ is \textit{log-concave} if $f=\exp(-\varphi)$ where $\varphi:\reals^d\to\reals$ is convex. A distribution is radial (resp., log-concave) if its density function is radial (resp., log-concave). For some $L>0$, a function $f\in C^2(\reals)$ is \textit{$L$-smooth} in a domain $A\subseteq\reals^d$ if $f(\bx)-f(\bx') \le \nabla f(\bx')^{\top}(\bx-\bx') + \frac{L}{2}\norm{\bx-\bx'}^2$ for all $\bx,\bx'\in A$. When optimizing an $L$-smooth objective, a learning rate (step size) $\eta$ is \textit{stable} if $\eta<2/L$. We let $\mathbb{S}_R^{d-1}\coloneqq\set{\bx\in\reals^d:\norm{\bx}=R}$ denote the $d$-dimensional hypersphere of radius $R$ centered at the origin. We use $\Ncal(\mu,\sigma^2)$ to denote a normal random variable with mean $\mu$ and variance $\sigma^2$, and $\Ncal(\boldsymbol{\mu},\Sigma)$ to denote a multivariate normal random variable with mean $\boldsymbol{\mu}$ and covariance matrix $\Sigma$.
	
	Our shallow class of neural networks with inferior approximation capabilities is the following class of depth 2 networks with architecture defined by
	\begin{equation}\label{eq:depth2}
		\bx\mapsto\sum_{j=1}^{r} w_j\sigma\p{\inner{\bu_j,\bx}+b_j}+b_0,
	\end{equation}
	where for terseness we occasionally denote the trainable parameters of this architecture in vectorized form using $\btheta\coloneqq(\bu_1,\ldots,\bu_r,\bw,\bb,b_0)$. In the above architecture, $\sigma$ is some non-linear activation function which satisfies the following assumption which we adopt from \citet{eldan2016power}:
	\begin{assumption}[Polynomially-Bounded Activation]\label{asm:poly_bounded}
		The activation function $\sigma$ is Lebesgue measurable and satisfies
		\[
			\abs{\sigma(x)}\le C_{\sigma}\p{1+\abs{x}^{\alpha_{\sigma}}}
		\]
		for all $x\in\reals$ and for some constants $C_{\sigma},\alpha_{\sigma}>0$.
	\end{assumption}
	We remark that the above assumption is very mild and is satisfied by all standard activation functions used in the literature, which includes common examples such as ReLU and sigmoidal activations, and in particular includes the error function and clipped ReLU activations which we use in the architecture to be defined next: For our stronger class with superior approximation capabilities, we consider the following  `depth 3' neural networks with a simplified structure,
	\begin{equation}\label{eq:architecture}
		\bx\mapsto \crelu{\sum_{j=1}^{r} w_j\erf\p{\inner{\bu_j,\bx}+b_j} }.
	\end{equation}
	In this architecture, we use error function activations $z\mapsto\erf(z)$ in the hidden layer, and a clipped ReLU activation $\crelu{z}\coloneqq\min\set{\relu{z},1}$ in the output neuron, where $\relu{z}\coloneqq\max\set{0,z}$ denotes the ReLU activation. We point out that strictly speaking, this is not a depth 3 neural network since there are no weights applied after the clipped ReLU activation. However, the additional non-linearity as we will see later in this section will prove to be necessary and sufficient for approximating the target functions considered in this paper. Our particular choice of the error function activation is motivated mainly by technical requirements, since it allows us to greatly simplify certain necessary calculations (e.g.\ in the proofs of \thmref{thm:well-behaved_random_features} and \propref{prop:truncation_expectation}). Likewise, the choice of the clipped ReLU is also motivated by technical considerations, and we conjecture our result to hold for a wider class of activation functions, leaving such generalizations to future work.
	
	%We remark that we believe that our results will hold more generally for a wider class of functions, namely non-decreasing and odd sigmoidal activations. This, however, could result in a substantially more complicated proof.
	
	%Likewise, the choice of the clipped ReLU is also motivated by technical considerations, where we conjecture our result to also hold for non-decreasing sigmoidal activations $\sigma$ such that $\lim\limits_{z\to-\infty}\sigma(z)=0$ and $\lim\limits_{z\to\infty}\sigma(z)=1$, and we leave such generalizations to future work.
	
	Since our goal is to separate different architectures according to their asymptotic performance with respect to the input dimension $d\ge2$, our assumptions formally refer to sequences of distributions and target functions rather than individual objects. We assume that we are given access to a finite i.i.d.\ data set $\set{\bx_i,y_i}_{i=1}^n$ sampled from the corresponding distribution $\Dcal_d$ in a sequence $\{\Dcal_d\}_{d=2}^{\infty}$, where the target values are determined according to a sequence of ball indicator functions, $y_i=\one{\norm{\bx_i}\le\lambda_d}$ for $\{\lambda_d\}_{d=2}^{\infty}$ where $\lambda_d\in[1,2]$. Since we will use a random features technique where the hidden layer remains fixed during training to derive our positive optimization result, we will use normal bold-faced letters ($\bx\in\reals^d$) and tilde bold-faced letters ($\hat{\bx}\in\reals^r$) to make the distinction between the $d$-dimensional data we are given and the $r$-dimensional random features produced by the output of the hidden layer in \eqref{eq:architecture}.
	
	\subsection{Main Results}\label{subsec:main_result}
	
	We now turn to state our main theorems. The first shows that there exists a sequence of radial distributions with unbounded support such that efficiently learning a ball indicator to accuracy $\Ocal(d^{-4})$ is not possible using any learning algorithm which must return a network with a weak architecture as defined in \eqref{eq:depth2}, whereas the strong architecture in \eqref{eq:architecture} can efficiently learn such ball indicators using GD.
	
	\begin{theorem}\label{thm:main}
		There exists a sequence of radial distributions $\{\Dcal_d\}_{d=2}^{\infty}$ such that the following holds:
		\begin{itemize}
			\item
			There exist universal constants $c_1,c_2,c_3,c_4>0$ and a sequence of target functions $\{f_d(\bx)\coloneqq\one{\norm{\bx}\le\lambda_d}\}_{d= c_1}^{\infty}$ where $\lambda_d\in[1,2]$ for all $d$, such that for all $d\ge c_1$, any algorithm that returns a depth 2 neural network $N$ with architecture as defined in \eqref{eq:depth2} which employs an activation function $\sigma$ satisfying Assumption~\ref{asm:poly_bounded} must satisfy
			\[
				\E_{\bx\sim\Dcal_d}\pcc{\p{N(\bx)-f_d(\bx)}^2} \ge \frac{c_2}{d^4},
			\]
			unless the algorithm runs in time at least $c_3\exp(c_4d)$.
			\item
			For all $\epsilon,\delta\in(0,1)$ and dimension $d\ge2$, there exist $r,T$ and sample size $n$ that are $\text{poly}(1/\epsilon,\ln(1/\delta),d)$ and learning rate which is $1/\text{poly}(1/\epsilon,\ln(1/\delta),d)$, such that if we run at most $T$ iterations of GD using a neural network with architecture as defined in \eqref{eq:architecture}, where the network's weights are initialized using a standard initialization scheme (see Assumption~\ref{asm:init}) and the hidden layer is held fixed during training. Then for all $\lambda\in[1,2]$, with probability at least $1-\delta$ over the randomness in the initialization of the network and the sampling of the data, the algorithm returns a neural network $N$ such that
			\[
				\E_{\bx\sim\Dcal_d}\pcc{\p{N(\bx)-\one{\norm{\bx}\le\lambda}}^2} \le \epsilon.
			\]
		\end{itemize}
	\end{theorem}
	The proof of the above theorem appears in \appref{app:main_thm_proof}. The main property of the distributions $\{\Dcal_d\}_{d=2}^{\infty}$ used in our theorem is that they are heavy-tailed in the sense that they produce data instances with very large norm. Essentially, we will be penalized severely for not approximating data well even when it's far away from the origin. This however also results in a more challenging positive optimization result to achieve, since data with large norm is harder to learn; for example, this results in an objective with gradients having unbounded norm which significantly complicates our analysis. To partially circumvent this difficulty we use error function activations with image bounded in $[-1,+1]$ which mitigates the input's large norm to some extent.
	
	We stress that in our result, the hardness of learning with a depth 2 neural network applies regardless of whether we are given access to a sufficiently large sample or not. In fact, since the hardness stems from the inability of depth 2 neural networks to approximate the target functions $\{f_d\}_{d=2}^{\infty}$, the hardness result persists even with perfect knowledge of the distributions and of the constants $\{\lambda_d\}_{d=2}^{\infty}$. On the other hand, our result only holds in a high-accuracy regime, and does not exclude the existence of an efficient algorithm which returns an approximation of a ball indicator to accuracy slightly larger than $c_2d^{-4}$, which might be considered sufficient for certain practical applications in machine learning. The dependence on $d^{-4}$ is an artifact of our proof technique (which relies on a reduction from the main result of \citet{eldan2016power}). To complement the above theorem, we present the following result which establishes the hardness of learning ball indicators to any accuracy smaller than a constant using a standard GD setting, for data generated by a certain compactly-supported distribution.
	
	\begin{theorem}\label{thm:secondary}
		There exists a sequence of compactly-supported radial distributions $\{\Dcal'_d\}_{d=2}^{\infty}$ such that the following holds:
		\begin{itemize}
			\item
			There exist constants $c_1,c_2>0$ and a sequence of target functions $\{f_d(\bx)\coloneqq\one{\norm{\bx}\le\lambda_d}\}_{d= 4}^{\infty}$ where $\lambda_d\in[1,2]$ for all $d$, such that for all $d\ge 4$, suppose we run GD on the architecture defined in \eqref{eq:depth2} which employs an activation function $\sigma\in C^2(\reals)$ satisfying Assumption~\ref{asm:poly_bounded} with the following assumptions:
			\begin{itemize}
				\item
				The algorithm runs for $T\ge4$ iterations.
				\item
				The learning rate satisfies $\eta\le\frac2L-\frac1T$ where $L$ is the smoothness of the objective when training the weak architecture on $n$ samples labeled according to a ball indicator.
				\item
				We initialize at some $\btheta_0$ such that $\pr\pcc{\norm{\btheta_0}_2\le p(d,r)} \ge 1-\exp(-d)$, for some polynomial $p$.
			\end{itemize}
			Then with probability at least $1-\exp(-d)$, for all $t\in\{0,\ldots,T\}$, the network $N_{\btheta_t}(\cdot)$ at iteration $t$ satisfies
			\[
				\epsilon\coloneqq\E_{\bx\sim\Dcal'_d}\pcc{\p{N_{\btheta_t}(\bx)-f_d(\bx)}^2}>\frac{1}{400},
			\]
			unless the algorithm runs in time at least
			\[
				c_1\exp\p{c_2\min\set{\frac{1}{\sqrt{\epsilon}}\ln\p{d\sqrt{\epsilon} +2}, d\ln\p{\frac{1}{d\sqrt{\epsilon}} +2}}},
			\]
			where $c_1,c_2>0$ depend solely on $\sigma$.
			
			\item
			For all $\epsilon,\delta\in(0,1)$ and dimension $d\ge3$, there exist $r,T$ and sample size $n$ that are $\text{poly}(1/\epsilon,\ln(1/\delta),d)$ and learning rate which is $1/\text{poly}(1/\epsilon,\ln(1/\delta),d)$, such that if we run at most $T$ iterations of GD using a neural network with architecture as defined in \eqref{eq:architecture}, where the network's weights are initialized using a standard initialization scheme (see Assumption~\ref{asm:init}). Then for all $\lambda\in[1,2]$, with probability at least $1-\delta$ over the randomness in the initialization of the network and the sampling of the data, the algorithm returns a neural network $N$ such that
			\[
				\E_{\bx\sim\Dcal'_d}\pcc{\p{N(\bx)-\one{\norm{\bx}_2\le\lambda}}^2} \le \epsilon.
			\]
		\end{itemize}
	\end{theorem}
	
	%While it is indeed possible to approximate a ball indicator on a compact domain to accuracy $\epsilon>0$ for any \emph{fixed} $\epsilon$ using width polynomial in $d$,\footnote{This can be achieved by first approximating a ball indicator to accuracy $\epsilon/2$ using a $2/\epsilon$-Lipschitz approximation, and then using the main result in \citet{safran2019depth} to approximate the Lipschitz approximation to accuracy $\epsilon/2$ using width polynomial in $d$ (albeit exponential in $1/\epsilon$).} we believe it might be possible to further strengthen our lower bound to hold for an exponent larger than minus $4$ (but still strictly smaller than $0$). Doing so, however, might require the development of new techniques other than the ones known in the literature, which we leave for future work. In any case, we emphasize that our result still separates the two different architectures defined in Eqs.~(\ref{eq:depth2},\ref{eq:architecture}) in the sense that the latter allows efficient learning for all $\epsilon>0$, independent of the input dimension $d$.
	
	The proof of the theorem, which appears in \appref{app:secondary_main_proof}, builds on a reduction to the technique introduced in \citet{daniely2017depth}, showing that an approximation to accuracy $\epsilon$ is only possible if either of the width of the approximating network or the Euclidean norm of its weights is exponential in either $d$ or $\frac{1}{\sqrt{\epsilon}}$. Since the progress of GD and the norm of the weights upon initialization are bounded due to the assumptions in the theorem statement, we have that GD must perform a large number of iterations or that the network must be very wide to achieve a good approximation, resulting in an exponential running time of the algorithm in either of the parameters. This result essentially combines the approximation limitations of our weak architecture and the inability of GD with reasonably-sized steps to move much in weight space to produce a negative result.
	
	Turning to discuss the assumptions made in the theorem in more detail, we first note that our assumptions are mostly mild. The assumption that $\sigma\in C^2(\reals)$ is to simplify the analysis, however since the non-smooth setting typically tends to result in worse optimization guarantees, we do not expect the running time to drastically change for the better if $\sigma$ is assumed to be a ReLU for example. The polynomial boundedness assumption of the norm upon initialization is also very mild, and holds for essentially any initialization scheme used in the literature, which includes in particular the common Xavier initialization (see \citep{glorot2010understanding}), as well as the initialization scheme we use in our positive optimization result (see Assumption~\ref{asm:init}). Lastly, the bound on the learning rate is a necessary condition for the local convergence of GD in a general setting, since there are cases where no non-trivial guarantee can be given otherwise, as we demonstrate in \appref{app:stability}.
	
	It is interesting to note that a similar trade-off between the dimension and accuracy parameters was given in \citet{hsu2021approximation}, where it is shown that a certain high-frequency function cannot be approximated efficiently using depth 2 networks unless the width incurs exponential dependence on either $d$ or $\frac{1}{\epsilon^2}$. Moreover, it can be shown that approximating a ball indicator (in the $L_{\infty}$ norm) on a compact domain can be done efficiently in $d$ if we only wish to obtain constant accuracy $\epsilon>0$.\footnote{This can be achieved by first approximating a ball indicator to accuracy $\epsilon/2$ using a $2/\epsilon$-Lipschitz approximation, and then using the main result in \citet{safran2019depth} to approximate the Lipschitz approximation to accuracy $\epsilon/2$ using width polynomial in $d$ (albeit exponential in the constant $1/\epsilon$).} This provides a partial indication that the lower bound we derive here is somewhat tight, since a lower bound with exponential dependence on both parameters (i.e.\ exponential in $\max\{d,1/\sqrt{\epsilon}\}$) is not possible.
	
	%Finally, unlike the previous theorem, we have that the compactly supported distribution allows us to also train the hidden layer, which is of course closer in nature to how neural networks are trained in practice.
	
	We now turn to provide an explicit construction of the sequences of distributions used in our main theorems in the following subsection.
	
	\subsection{The Distributions Used}\label{subsec:construction}
	
	In this subsection, we formally define the sequence of distributions used in \thmref{thm:main} and \thmref{thm:secondary}. We define a density of a $d$-dimensional distribution by
	\begin{equation}\label{eq:mu_bar_d_density}
		\bar{\mu}_{d}(\bx)\coloneqq\p{\frac{R_d}{\norm{\bx}}}^d J_{d/2}^2(2\pi\alpha\sqrt{d} R_d\norm{\bx}),
	\end{equation}
	where $R_d=\frac{1}{\sqrt{\pi}}\p{\Gamma\p{\frac{d}{2}+1}}^{1/d}$ is the radius of the unit-volume $d$-dimensional Euclidean ball, $\alpha\ge1$ is the absolute constant used in the derivation of the main result in \citet{eldan2016power}, and $J_{\nu}$ is the Bessel function of the first kind, of order $\nu$ (see the aforementioned reference for further discussion about these objects). Such similar density functions (identical up to a linear change of variables) were used to show various inapproximability results for radial functions \citep{eldan2016power,safran2017depth,safran2019depth}.
	We now define the densities of our sequence of distributions $\{\Dcal_d\}_{d=2}^{\infty}$ used in \thmref{thm:main}, which are a mixture of the distribution with density $\bar{\mu}_d$ defined in \eqref{eq:mu_bar_d_density} and a multivariate normal with covariance matrix $\frac{1}{\sqrt{d}}I_d$:
	\begin{equation}\label{eq:mu_d_def}
		\mu_d(\bx)\coloneqq 0.5\bar{\mu}_d(\bx) + 0.5\p{\frac{d}{2\pi}}^{d/2}\exp\p{-\frac{1}{2}d^d\norm{\bx}^2}.
	\end{equation}
	Turning to define the sequence of distributions $\set{\Dcal'_d}_{d=2}^{\infty}$ used in \thmref{thm:secondary}, we define the $d$-dimensional distribution $\Dcal'_d$ to be the distribution of $\bx\in\reals^d$ which is given by the sum
	\begin{equation}\label{eq:daniely_dist_def}
		\bx\coloneqq\bx_1+\bx_2,\hskip 0.2cm\text{where}\hskip 0.2cm \bx_1,\bx_2\in\mathbb{S}_1^{d-1}\hskip 0.2cm\text{are i.i.d.\ and uniformly distributed.}
	\end{equation}
	We denote the density of $\Dcal'_d$ by $\varphi_d$.
	
	Having defined the sequence of distributions used in our main results, we now turn to formally present our approximation lower bounds in the next section.
	
	\section{Lower Bounds for Approximating Ball Indicators}\label{sec:lower_bounds}
	
	In this section, we present our lower bounds for the approximation of ball indicator functions used to derive the optimization lower bounds in our main results. %, with respect to the distributions $\Dcal_d,\Dcal'_d$ defined in the previous section. 
	As discussed in the related work section, the inapproximability of ball indicator functions was previously shown in \citet{safran2017depth}, and our goal here is merely to adapt and extend similar bounds to our setting and assumptions.
	
	The following theorem establishes the inapproximability of ball indicators with respect to the heavy-tailed distribution $\Dcal_d$ defined in the previous section, which is used to prove \thmref{thm:main}.
	\begin{theorem}\label{thm:inapproximability}
		The following holds for some universal constants $c_1,c_2,c_3,c_4>0$, and any network employing an activation function satisfying Assumption~\ref{asm:poly_bounded}: For all $d\ge c_1$, there exist a sequence $\{\lambda_d\}_{d=c_1}^{\infty}$ of constants $\lambda_d\in[1,2]$ and a sequence of radial distributions $\{\Dcal_d\}_{d=2}^{\infty}$ with density functions $\{\mu_d\}_{d=2}^{\infty}$ defined in \eqref{eq:mu_d_def} such that for any depth 2 neural network $N$ with architecture as defined in \eqref{eq:depth2}, we have
		\[
			\int_{\reals^d}\p{N(\bx)-\one{\norm{\bx}\le\lambda_d}}^2\mu_d(\bx)d\bx \ge \frac{c_2}{d^4},
		\]
		unless $N$ has width at least $c_3\exp(c_4d)$.
	\end{theorem}

	The proof of the above theorem, which appears in \appref{app:inapproximability_proof}, relies on the main result of \citet{eldan2016power}. Essentially, it is shown in the aforementioned work that a superposition of $\Ocal(d^2)$ thin shells with varying radii cannot be approximated efficiently using depth 2 neural networks, with respect to a heavy-tailed distribution with density similar to the one defined in \eqref{eq:mu_bar_d_density}. This implies that there must \emph{exist} a ball indicator function that is inapproximable with respect to the distribution with such a density, since otherwise we can efficiently construct an approximation of $\Ocal(d^2)$ thin shells and concatenate them using a depth 2 neural network, contradicting the result of \citet{eldan2016power}.
	
	The following theorem establishes our second inapproximability result of ball indicators with respect to the compactly-supported distribution $\Dcal'_d$ defined in the previous section, which is used to prove \thmref{thm:secondary}.
	
	\begin{theorem}\label{thm:daniely_inapproximability}
		The following holds for any neural network employing an activation function $\sigma$ which satisfies Assumption~\ref{asm:poly_bounded}, and constants $c_1,c_2>0$ that may depend solely on $\sigma$: For all $d\ge 4$, there exist a sequence $\{\lambda_d\}_{d=4}^{\infty}$ of constants $\lambda_d\in[1,2]$ and a sequence of radial distributions $\{\Dcal'_d\}_{d=4}^{\infty}$ defined in \eqref{eq:daniely_dist_def} such that for any natural $m\ge1$ and any depth 2 neural network $N_{\btheta}$ with architecture as defined in \eqref{eq:depth2}, we have
		\[
			\int_{\reals^d}\p{N_{\btheta}(\bx)-\one{\norm{\bx}\le\lambda_d}}^2\varphi_d(\bx)d\bx > \frac{1}{196(m+1)^2},
		\]
		unless $N$ has width $r$ or weight norm $\norm{\btheta}_2$ at least $c_1\exp\p{c_2\min\set{m\ln\p{\frac dm +2}, d\ln\p{\frac md +2}}}$.
	\end{theorem}

	The proof of the above theorem, which appears in \appref{app:daniely_inapprox_proof}, builds on and refines the technique developed in \citet{daniely2017depth}. We first show that approximating a radial function $f:\reals^d\to\reals$ consisting of $2m$ thin shells using a depth 2 neural network must require a large width or a large norm of the weights, since its radial component function (namely, the univariate function $g:\reals\to\reals$ such that $f(\bx)=g(\norm{\bx}_2)$) cannot be approximated well using polynomials of degree at most $m$. Then, a similar reduction argument as the one used in the proof of \thmref{thm:inapproximability} is used to derive a lower bound for a ball indicator function. 
	
	The benefit in the introduction of the parameter $m$ which is independent of $d$ is that it allows us to separate the dependence on the dimension and the accuracy parameters, leading to a hardness of approximation result that persists beyond the high-accuracy regime used in the previous theorem, and which holds even for a (sufficiently small) constant accuracy requirement. On the flip side, since we cannot rule out the existence of a good approximation using a moderately-wide network with exponential weights, this comes at the cost of a somewhat weaker algorithmic implication where the hardness of learning is shown for just the GD algorithm and not for any algorithm as is the case in \thmref{thm:main}. Nevertheless, exponential weights have limitations beyond those that we exploit to show an optimization-based lower bound in \thmref{thm:secondary}. For example, the class of depth 2 neural networks with exponentially-bounded weights has an exponential FAT-shattering dimension and is therefore not statistically learnable in polynomial time (if we consider a broader class of target distributions than ball indicators). We leave the derivation of such statistically-based lower bounds and the generalization to a broader class of optimization algorithms as important future work directions.

	\section{Learning Ball Indicators Using Random Features}\label{sec:learning}

		In this section, we prove that under certain assumptions on the distribution of the data, the family of ball indicator functions defined as $\set{\one{\norm{\bx}\le\lambda}}_{\lambda\in[1,2]}$ can be learned efficiently by GD using a neural network with architecture defined in \eqref{eq:architecture}, with a polynomial network width, sample size and number of GD iterations. Recalling that we use a tilde to distinguish a $d$-dimensional data point $\bx\in\reals^d$ and $r$-dimensional random features $\tilde{\bx}\in\reals^r$ produced by the hidden layer, we now define our empirical objective function which takes the following form
		\begin{equation}\label{eq:obj}
			\hat{F}(\bw) \coloneqq \frac1n\sum_{i=1}^{n}\p{\crelu{\bw^{\top}\tilde{\bx}_i} - y_i}^2.
		\end{equation}
		Additionally, we define the risk of a predictor $\bw$ over a distribution $\Dcal_d$ to be
		\begin{equation*}
			F(\bw) \coloneqq \E_{\bx\sim\Dcal_d}\pcc{\p{\crelu{\bw^{\top}\tilde{\bx}} - \one{\norm{\bx}\le\lambda}}^2} = \E_{\bx\sim\Dcal_d}\pcc{\hat{F}(\bw)}.
		\end{equation*}
		To facilitate our analysis, we make the following assumptions on the distributions of the data and the initialization of our neural network. Specifically, for the distribution of our data, we assume the following.
		
		\begin{assumption}[Data Distribution Assumption]\label{asm:d_dist_assumption}
			The sequence of distributions $\set{\Dcal_d}_{d=2}^{\infty}$ possesses radial symmetry and the density $\gamma_d$ of the distribution of its norm given by the random variable $X_d$ satisfy the following properties:
			\begin{enumerate}
				\item
				There exists a constant $C>0$ such that for all $d$
				\[
					\int_{0}^{2} \gamma_d(x)dx\ge C.
				\]
				\item
				$\gamma_d(x)>0$ for all $x\ge0$, and there exists a polynomial $q$ such that for all $d$ and all $x\in[0.9,2.05]$
				\[
					\gamma_d(x) \le q(d).
				\]
				\item
				There exist $c_1,c_2>0$ such that for all $x\ge c_1$ and all $d$,
				\[
					\pr\pcc{X_d\ge x}\le c_2\cdot\frac{1}{x}.
				\]
			\end{enumerate}
		\end{assumption}
		We remark that the above assumption is mild. The main non-trivial requirement is that $\set{\Dcal_d}_{d=2}^{\infty}$ possess radial symmetry, but the remaining properties merely require that $\gamma_d$ maintains a non-negligible portion of its mass within a constant distance from the origin as $d$ grows, is polynomially-bounded for any $d$, and has a tail that decays at most proportionally to its distance from the origin. Note that this includes heavy-tailed distributions such as Cauchy for example.
		
		Moving to our initialization of the network with architecture defined in \eqref{eq:architecture}, we assume the following:
		\begin{assumption}[Network Initialization]\label{asm:init}
			~\\
			\begin{enumerate}
				\item
				The weights in the hidden layer $U_j\in\reals^d$ and the bias terms $B_j$ satisfy
				\[
					U_j\sim\Ncal\p{0,\frac14\cdot I_d},\hskip 0.5cm B_j\sim\Ncal\p{0,\frac14}.
				\] 
				That is, each weight and bias term in the hidden layer is sampled from i.i.d.\ normal variables with mean zero and variance $\frac14$.
				\item
				The weights of the output neuron $W\in\reals^r$ satisfy
				\[
					W\sim\Ncal\p{0,\frac{1}{r^2}\cdot I_r}.
				\]
				That is, the weights of the output neuron are sampled from i.i.d.\ normal variables with zero mean and variance $\frac{1}{r^2}$.
			\end{enumerate}
		\end{assumption}
		We now turn to present our main result of this section. The following theorem shows that under Assumptions~\ref{asm:d_dist_assumption} and \ref{asm:init}, for any $\epsilon,\delta\in(0,1)$, if $r$ is polynomial in $d,1/\epsilon$ and $\ln(1/\delta)$, and if the sample size $n$ is cubic in $r$ up to poly-logarithmic factors, then $\Ocal(r\ln(r/\epsilon))$ iterations of GD will suffice to achieve a risk of at most $\epsilon$ with probability at least $1-\delta$.
		
		\begin{theorem}\label{thm:gd_convergence}
			Under Assumptions~\ref{asm:d_dist_assumption} and \ref{asm:init}, for all $\epsilon,\delta\in(0,1)$, suppose we run GD on the objective in \eqref{eq:obj} using an architecture as defined in \eqref{eq:architecture} where the hidden layer is held fixed, with the following parameters:
			\begin{itemize}
				\item
				Network width $r\ge \max\set{12000^4,c\epsilon^{-5}\p{q^4(d) +\ln^2(1/\delta)}}$,
				\item
				Sample size $n=\lceil cr^3\log_2^2(r)$$\rceil$,
				\item
				Fixed learning rate $\eta<\frac{\nu}{8r}$, where $\nu=\frac{C}{840\cdot9^4\cdot10^6}$.
			\end{itemize}
			until the iterate $\bw_t$ satisfies $\hat{F}(\bw_t)\le\frac{\epsilon}{2}$. Then with probability at least $1-\delta$, we have that after at most $T=2\eta^{-1}\nu^{-1}\ln\p{r/8\epsilon}$ iterations GD returns $\bw_t$ such that
			\[
				F(\bw_t)\le\epsilon,
			\]
			where $C,q,c_1,c_2$ are defined in Assumption~\ref{asm:d_dist_assumption}, and $c\ge1$ is a constant that depends solely on $C,c_1,c_2$.
		\end{theorem}
		In a nutshell, the proof of the above theorem which appears in \appref{app:proof_gd_convergence}, relies on a random features approach where the weights of the neurons in the hidden layer remain fixed throughout the optimization process, and then showing that the weights learned by GD in the output neuron converge with high probability to a point with risk at most $\epsilon$. We remark that we did not attempt to decrease the constants in the theorem statement, and that even when we fix the hidden layer, the resulting optimization problem is still non-convex in general and may even contain exponentially many sub-optimal local minima \citep{auer1996exponentially,safran2016quality}. In a bit more detail, the proof can be broken down into four parts. In the first part, we show that with high probability, sufficiently many random features will result in the existence of a weight of the output neuron that attains small empirical loss. In the second part, we establish that the random features produced by the initialization scheme defined in Assumption~\ref{asm:init} have well-behaved two-dimensional marginal distributions. This property is important for showing that the gradient of our objective points in the direction of the point with small empirical risk whose existence is shown in the first part. In the third part, we prove that GD converges with high probability to a point with small empirical loss. In the fourth and last part, we show generalization bounds which establish that our small empirical loss also results in a small risk on the distribution of the data.
		
		We now give a more detailed outline of the proof of \thmref{thm:gd_convergence}, focusing on each of the four parts in the following subsections, starting with the ability of random features to approximate our target functions of ball indicators.

		\subsection{Random Feature Approximation via Truncation}\label{subsec:random_features}
		
		In this subsection, we show that random features can approximate ball indicator functions having radius in $[1,2]$. The main idea is that we can truncate a particular random feature (neuron in the hidden layer) by setting its corresponding weight in the output neuron to zero. This way we can manipulate the expectation of the network's output to approximate a desired function, and then use concentration of measure to guarantee a good approximation with high probability. We now present the main theorem of this subsection:
		\begin{theorem}\label{thm:random_features}
			Given any $\delta\in(0,1)$, suppose that $r$ and $n$ satisfy
			\begin{itemize}
				\item
				$r\ge\max\set{12000^4,500^4\max\set{c_1,10}^4\ln^2(8n/\delta)}$,
				\item
				$\sqrt{\frac{\ln(8/\delta)}{2n}} \le p\coloneqq \frac{1800q(d)+500c_2\sqrt{\ln\p{8n/\delta}}}{r^{0.25}}$.
			\end{itemize}
			Then for any $\lambda\in[1,2]$, with probability at least $1-\delta$ over drawing a sample $\{\bx_i,\one{\norm{\bx_i}\le\lambda}\}_{i=1}^n$ from $\Dcal_d$ under Assumption~\ref{asm:d_dist_assumption} with parameters $c_1,c_2,q$, and over the randomness in the initialization of $\norm{\bw_0}$ and the remaining network's weights under Assumption~\ref{asm:init}, there exists $\bv\in\reals^r$ which satisfies the following properties:
			\begin{enumerate}
				\item
				$\norm{\bw_0-\bv}^2<\norm{\bv}^2\le\frac{1}{\sqrt{r}}$.
				\item
				There exists a partition $\bigcup_{k=1}^4 I_k$ of $[0,\infty)$ into disjoint intervals $I_k$ such that $\crelu{\bv^{\top}\tilde{\bx}_j}=\one{\norm{\bx_j}\le \lambda}=y_j$ for all $j$ that satisfy $\norm{\bx_j}\in I_1\cup I_3$, and $\abs{\bv^{\top}\tilde{\bx}_j} \le3$ for all $j$ that satisfy $\norm{\bx_j}\in I_2\cup I_4$. Moreover, we have that
				\[
					\abs{\set{j:\norm{\bx_j}\in I_2\cup I_4}} \le 2pn.
				\]
			\end{enumerate}
		\end{theorem}
		Putting it in simple words, the above theorem establishes that under some technical assumptions on the magnitude of $r$ and $n$, there exists a point $\bv\in\reals^r$ in our optimization space that classifies `most' of the data instances correctly, and thus attains an empirical loss which decays with $r$ and $n$ to zero. Moreover, the theorem also implies that this point is close enough to our initialization point $\bw_0$ to attract GD (see \propref{prop:gradient_dot_product_bound} for a formal statement).
		
		The proof of \thmref{thm:random_features} which appears in \appref{app:random_features_proof} relies on concentration of measure and on the observation that due to the symmetry in Assumption~\ref{asm:init}, a neural network that is initialized in this manner will approximate the zero function in expectation. To bias the initialization of the network to approximate non-trivial functions, we can truncate undesired random features by setting their corresponding weight in the output neuron to zero. For example, by truncating random features with a negative bias term realization, we can bias the function computed by the network to return a positive prediction on data instances with small norm. Note that by doing so we do not change the initialization scheme defined in Assumption~\ref{asm:init}, but rather show that such a manipulation performed in the optimization space of the output neuron is equivalent to initializing using a different initialization scheme, which can effectively approximate various target functions using our architecture. 
	
		\subsection{Radial Distributions Produce Well-Behaved Random Features}\label{subsec:well-behaved_random_features}
		
		Using random features and keeping the hidden layer fixed throughout the optimization essentially reduces our learning problem to that of learning with a single neuron. To facilitate the convergence of GD in such a setting, a useful property to show is that the two-dimensional marginal distribution of the data has a strictly positive density in a neighborhood around the origin \citep{yehudai2020learning,frei2020agnostic}. Such a property is desired since it allows one to show that if the target vector $\bv$ that we wish to approximate and our current GD iterate $\bw_t$ satisfy some certain properties, then the density assumption in the subspace spanned by $\bv,\bw_t$ guarantees that the gradient points in the direction of $\bv$. To establish that this property is satisfied in our setting, we would need to analyze the $r$-dimensional distribution of the random features produced by Assumptions~\ref{asm:d_dist_assumption} and \ref{asm:init}. To this end, we have the following theorem.
		
		\begin{theorem}\label{thm:well-behaved_random_features}
			Suppose that the sequence of distributions $\{\Dcal_d\}_{d=2}^{\infty}$ satisfies Assumption~\ref{asm:d_dist_assumption}(1) with a constant $C>0$, and that a neural network with an architecture defined in \eqref{eq:architecture} is initialized according to Assumption~\ref{asm:init}. Then the sequence of distributions $\{\tilde{\Dcal}_d\}_{d=2}^{\infty}$ of the random features produced by the first hidden layer satisfies the following: %Assumption \ref{asm:r_dist_assumption} with parameters $\alpha=\frac{1}{9}$, $\beta=c\sigma^2(1+b^2)2^{-17}$, $\gamma=1$.
			For any vectors $\bw_1\neq\bw_2$, let $\tilde{\Dcal}_{\bw_1,\bw_2}$ denote the marginal distribution of $\tilde{\bx}$ on the subspace spanned by $\bw_1,\bw_2$ (as a distribution over $\reals^2$). Then any such distribution has a density function $p_{\bw_1,\bw_2}(\bx)$ which satisfies $\inf_{\bx:\norm{\bx}\le\alpha} p_{\bw_1,\bw_2}(\bx)\ge\beta$, where
			\[
				\alpha=\frac{1}{9},\hskip 0.5cm \beta=C\cdot10^{-6}.
			\]
		\end{theorem}
		The proof of the above theorem, which appears in \appref{app:well-behaved_random_features_proof}, relies for the most part on the analysis of log-concave distributions performed in \citet{lovasz2007geometry}. The main challenge in the proof is that we cannot guarantee that the distributions $\{\tilde{\Dcal}_d\}_{d=2}^{\infty}$ are log-concave. Instead, we show that the density of these distributions can be lower bounded by a log-concave density, from which the theorem follows. With Thms.~\ref{thm:random_features} and \ref{thm:well-behaved_random_features} at hand, we can turn to show a convergence result for GD in the next subsection.
		
		\subsection{Convergence of Gradient Descent}\label{subsec:optimization}
		
		In this subsection, we establish the technical tools that facilitate the convergence of GD. We begin with the following proposition, which establishes that as long as some technical inequality is satisfied, then the gradient of the objective in \eqref{eq:obj} points us in the direction of the vector $\bv$ which achieves small empirical loss and whose existence is established by \thmref{thm:random_features}. More formally, we have the following:
		\begin{proposition}\label{prop:gradient_dot_product_bound}
			Under Assumptions~\ref{asm:d_dist_assumption} and \ref{asm:init}, there exist constants $\alpha,\beta>0$ such that for all $\delta\in(0,1)$, with probability at least $1-\delta$ over sampling $n$ data instances from $\Dcal_d$ and over the randomness in the initialization of the network, for all $\bw,\bv\in\reals^r$ that satisfy $\norm{\bw-\bv}\le\norm{\bv}\le0.5$ and the following technical inequality
			\begin{align}
				& \frac1n\sum_{i=1}^n\p{\crelu{\bw^{\top}\tilde{\bx}_i} - \crelu{\bv^{\top}\tilde{\bx}_i} }\cdot \one{\bw^{\top}\tilde{\bx}_i\in(0,1)} \cdot\p{\bw^{\top}\tilde{\bx}_i-\bv^{\top}\tilde{\bx}_i}\nonumber\\
				&\hskip 1cm \ge \frac2n\sum_{i=1}^n\p{\crelu{\bv^{\top}\tilde{\bx}_i} - y_i} \cdot \one{\bw^{\top}\tilde{\bx}_i\in(0,1)} \cdot\p{\bw^{\top}\tilde{\bx}_i-\bv^{\top}\tilde{\bx}_i},\label{eq:small_noise}
			\end{align}
			we have that
			\[
				\inner{\nabla F(\bw),\bw-\bv} \ge \frac12\norm{\bw-\bv}^2\p{\frac{\alpha^4\beta}{210} - 4r\sqrt{\frac{(8r+8)\log_2(n)}{n}} - \sqrt{\frac{2\ln\p{2/\delta}}{n}}}.
			\]
		\end{proposition}
		The proof of the above proposition, which appears in \appref{app:proof_of_gradient_dot_product_bound}, relies on lower bounding the dot product with the gradient using the assumed inequality, and then using the marginal density property from the previous subsection to guarantee that the gradient of the objective is correlated with the direction leading to $\bv$. In a sense, this result can be viewed as a finite data analog in the non-realizable setting of Thm.~4.2 in \citet{yehudai2020learning}. The main difficulty that arises in the finite data case is that in the application of the proposition, $\bw$ is an iterate returned by GD and thus we need to derive the above result uniformly for all the $\bw$'s that satisfy the above condition.
		
		The following proposition establishes that when the technical inequality assumption in \eqref{eq:small_noise} is violated then we have already achieved a small empirical loss.
		\begin{proposition}\label{prop:low_loss_attained}
			Under Assumptions~\ref{asm:d_dist_assumption} and \ref{asm:init}, for all $\delta\in(0,1)$, with probability at least $1-\delta$ over sampling $n$ data instances from $\Dcal_d$ and over the randomness in the initialization of the network, for all $\bw,\bv\in\reals^r$ that satisfy $\norm{\bw-\bv}\le\norm{\bv}\le r^{-0.25}$ and the following technical inequality
			\begin{align}
				& \frac1n\sum_{i=1}^n\p{\crelu{\bw^{\top}\tilde{\bx}_i} - \crelu{\bv^{\top}\tilde{\bx}_i} }\cdot \one{\bw^{\top}\tilde{\bx}_i\in(0,1)} \cdot\p{\bw^{\top}\tilde{\bx}_i-\bv^{\top}\tilde{\bx}_i}\nonumber\\
				&\hskip 1cm < \frac 2n\sum_{i=1}^n \p{\crelu{\bv^{\top}\tilde{\bx}_i} - y_i} \cdot \one{\bw^{\top}\tilde{\bx}_i\in(0,1)} \cdot\p{\bw^{\top}\tilde{\bx}_i-\bv^{\top}\tilde{\bx}_i}\nonumber,
			\end{align}
			if $r$ and $n$ satisfy the following inequalities
			\begin{itemize}
				\item
				$r\ge\max\set{12000^4,500^4\max\set{c_1,10}^4\ln^2(16n/\delta)}$,
				\item
				$\sqrt{\frac{\ln(16/\delta)}{2n}} \le p'\coloneqq \frac{1800q(d)+500c_2\sqrt{\ln\p{16n/\delta}}}{r^{0.25}}$.
			\end{itemize}
			Then we have that
			\begin{align*}
				&\hat{F}(\bw) \le 20p'+ \frac2r + 8\sqrt{\frac{(8r+8)\log_2(n)}{n}} + 2\sqrt{\frac{2\ln\p{4/\delta}}{n}}.
			\end{align*}
		\end{proposition}
		The proof of the above proposition which appears in \appref{app:proof_of_low_loss_attained} is technical, and is based on lower bounding the smaller term in the proposition assumption by the empirical loss. This is achieved by analyzing several different cases depending on the values attained by the dot products of $\bv$ and $\bw$ with the data instances. We then upper bound the larger term in the proposition assumption by using \thmref{thm:random_features}. Combining the two propositions stated in this subsection, the convergence of GD follows from the fact that we converge to $\bv$ at a linear rate by \propref{prop:gradient_dot_product_bound} as long as the technical inquality holds, and if it is violated then \propref{prop:low_loss_attained} guarantees that the empirical loss we achieved is sufficiently small.
		
		The final component in our proof of \thmref{thm:gd_convergence} is to show that the small empirical loss achieved using the above propositions translates to a generalization bound. To this end, we have the following theorem:
		\begin{theorem}\label{thm:generalization}
			For all $\delta\in(0,1)$, with probability at least $1-\delta$ over drawing a sample of size $n$ from any distribution $\tilde{\Dcal}_d$, we have for all $\bw$ satisfying $\norm{\bw}\le1$ that
			\[
				F(\bw) \le \hat{F}(\bw) + 4\sqrt{\frac rn} + 4\sqrt{\frac{2\ln(4/\delta)}{n}}.
			\]
		\end{theorem}
		The proof of the above theorem, which appears in \appref{app:gen_proof}, is simple and relies on standard Rademacher complexity arguments.
		
		%\section{Summary and Discussion}
		
		%In this paper, we show an optimization-based separation results between two classes of neural networks, where the benefit of the stronger class in approximating certain target functions much more efficiently manifests in practice and facilitates efficient optimization. One possible future work direction to strengthen our result would be to provide a tighter bound on the approximation capabilities of neural networks with architecture defined in \eqref{eq:depth2} for approximating ball indicators. As discussed earlier, the best known lower bound is $\Omega(d^{-4})$  and the best known upper bound is $$ which follows from the work of \citet{safran2019depth}.
		
		\bibliographystyle{plainnat}
		\bibliography{citations}

\begin{thebibliography}{39}
\providecommand{\natexlab}[1]{#1}
\providecommand{\url}[1]{\texttt{#1}}
\expandafter\ifx\csname urlstyle\endcsname\relax
  \providecommand{\doi}[1]{doi: #1}\else
  \providecommand{\doi}{doi: \begingroup \urlstyle{rm}\Url}\fi

\bibitem[Allen-Zhu and Li(2020)]{allen2020backward}
Zeyuan Allen-Zhu and Yuanzhi Li.
\newblock Backward feature correction: How deep learning performs deep
  learning.
\newblock \emph{arXiv preprint arXiv:2001.04413}, 2020.

\bibitem[Allen-Zhu et~al.(2019)Allen-Zhu, Li, and Liang]{allen2019learning}
Zeyuan Allen-Zhu, Yuanzhi Li, and Yingyu Liang.
\newblock Learning and generalization in overparameterized neural networks,
  going beyond two layers.
\newblock \emph{Advances in neural information processing systems}, 2019.

\bibitem[Auer et~al.(1996)Auer, Herbster, Warmuth,
  et~al.]{auer1996exponentially}
Peter Auer, Mark Herbster, Manfred~K Warmuth, et~al.
\newblock Exponentially many local minima for single neurons.
\newblock \emph{Advances in neural information processing systems}, pages
  316--322, 1996.

\bibitem[Bach(2017)]{bach2017breaking}
Francis Bach.
\newblock Breaking the curse of dimensionality with convex neural networks.
\newblock \emph{The Journal of Machine Learning Research}, 18\penalty0
  (1):\penalty0 629--681, 2017.

\bibitem[Bai and Lee(2019)]{bai2019beyond}
Yu~Bai and Jason~D Lee.
\newblock Beyond linearization: On quadratic and higher-order approximation of
  wide neural networks.
\newblock \emph{arXiv preprint arXiv:1910.01619}, 2019.

\bibitem[Boucheron et~al.(2005)Boucheron, Bousquet, and
  Lugosi]{boucheron2005theory}
St{\'e}phane Boucheron, Olivier Bousquet, and G{\'a}bor Lugosi.
\newblock Theory of classification: A survey of some recent advances.
\newblock \emph{ESAIM: probability and statistics}, 9:\penalty0 323--375, 2005.

\bibitem[Chen et~al.(2020)Chen, Bai, Lee, Zhao, Wang, Xiong, and
  Socher]{chen2020towards}
Minshuo Chen, Yu~Bai, Jason~D Lee, Tuo Zhao, Huan Wang, Caiming Xiong, and
  Richard Socher.
\newblock Towards understanding hierarchical learning: Benefits of neural
  representations.
\newblock \emph{arXiv preprint arXiv:2006.13436}, 2020.

\bibitem[Daniely(2017)]{daniely2017depth}
Amit Daniely.
\newblock Depth separation for neural networks.
\newblock In \emph{Conference on Learning Theory}, pages 690--696. PMLR, 2017.

\bibitem[{\relax DLMF}()]{NIST:DLMF}
{\relax DLMF}.
\newblock {\it NIST Digital Library of Mathematical Functions}.
\newblock http://dlmf.nist.gov/, Release 1.1.3 of 2021-09-15.
\newblock URL \url{http://dlmf.nist.gov/}.
\newblock F.~W.~J. Olver, A.~B. {Olde Daalhuis}, D.~W. Lozier, B.~I. Schneider,
  R.~F. Boisvert, C.~W. Clark, B.~R. Miller, B.~V. Saunders, H.~S. Cohl, and
  M.~A. McClain, eds.

\bibitem[Eldan and Shamir(2016)]{eldan2016power}
Ronen Eldan and Ohad Shamir.
\newblock The power of depth for feedforward neural networks.
\newblock In \emph{Conference on learning theory}, pages 907--940, 2016.

\bibitem[Frei et~al.(2020)Frei, Cao, and Gu]{frei2020agnostic}
Spencer Frei, Yuan Cao, and Quanquan Gu.
\newblock Agnostic learning of a single neuron with gradient descent.
\newblock \emph{arXiv preprint arXiv:2005.14426}, 2020.

\bibitem[Ghorbani et~al.(2021)Ghorbani, Mei, Misiakiewicz, and
  Montanari]{ghorbani2021linearized}
Behrooz Ghorbani, Song Mei, Theodor Misiakiewicz, and Andrea Montanari.
\newblock Linearized two-layers neural networks in high dimension.
\newblock \emph{The Annals of Statistics}, 49\penalty0 (2):\penalty0
  1029--1054, 2021.

\bibitem[Glorot and Bengio(2010)]{glorot2010understanding}
Xavier Glorot and Yoshua Bengio.
\newblock Understanding the difficulty of training deep feedforward neural
  networks.
\newblock In \emph{Proceedings of the thirteenth international conference on
  artificial intelligence and statistics}, pages 249--256. JMLR Workshop and
  Conference Proceedings, 2010.

\bibitem[Goel and Klivans(2019)]{goel2019learning}
Surbhi Goel and Adam~R Klivans.
\newblock Learning neural networks with two nonlinear layers in polynomial
  time.
\newblock In \emph{Conference on Learning Theory}, pages 1470--1499. PMLR,
  2019.

\bibitem[Goldreich(2003)]{goldreich2003foundations}
Oded Goldreich.
\newblock \emph{Foundations of Cryptography, Volume 2}.
\newblock Cambridge university press, 2003.

\bibitem[Hsu et~al.(2021)Hsu, Sanford, Servedio, and
  Vlatakis-Gkaragkounis]{hsu2021approximation}
Daniel Hsu, Clayton Sanford, Rocco~A Servedio, and Emmanouil-Vasileios
  Vlatakis-Gkaragkounis.
\newblock On the approximation power of two-layer networks of random relus.
\newblock \emph{arXiv preprint arXiv:2102.02336}, 2021.

\bibitem[Ji et~al.(2019)Ji, Telgarsky, and Xian]{ji2019neural}
Ziwei Ji, Matus Telgarsky, and Ruicheng Xian.
\newblock Neural tangent kernels, transportation mappings, and universal
  approximation.
\newblock \emph{arXiv preprint arXiv:1910.06956}, 2019.

\bibitem[Liang and Srikant(2017)]{liang2017deep}
Shiyu Liang and R~Srikant.
\newblock Why deep neural networks for function approximation?
\newblock In \emph{5th International Conference on Learning Representations,
  ICLR 2017}, 2017.

\bibitem[Lov{\'a}sz and Vempala(2007)]{lovasz2007geometry}
L{\'a}szl{\'o} Lov{\'a}sz and Santosh Vempala.
\newblock The geometry of logconcave functions and sampling algorithms.
\newblock \emph{Random Structures \& Algorithms}, 30\penalty0 (3):\penalty0
  307--358, 2007.

\bibitem[Malach and Shalev-Shwartz(2019)]{malach2019deeper}
Eran Malach and Shai Shalev-Shwartz.
\newblock Is deeper better only when shallow is good?
\newblock \emph{Advances in Neural Information Processing Systems},
  32:\penalty0 6429--6438, 2019.

\bibitem[Malach et~al.(2021)Malach, Yehudai, Shalev-Shwartz, and
  Shamir]{malach2021connection}
Eran Malach, Gilad Yehudai, Shai Shalev-Shwartz, and Ohad Shamir.
\newblock The connection between approximation, depth separation and
  learnability in neural networks.
\newblock \emph{arXiv preprint arXiv:2102.00434}, 2021.

\bibitem[Nesterov et~al.(2018)]{nesterov2018lectures}
Yurii Nesterov et~al.
\newblock \emph{Lectures on convex optimization}, volume 137.
\newblock Springer, 2018.

\bibitem[Owen(1956)]{owen1956tables}
Donald~B Owen.
\newblock Tables for computing bivariate normal probabilities.
\newblock \emph{The Annals of Mathematical Statistics}, 27\penalty0
  (4):\penalty0 1075--1090, 1956.

\bibitem[Owen(1980)]{owen1980table}
Donald~Bruce Owen.
\newblock A table of normal integrals: A table.
\newblock \emph{Communications in Statistics-Simulation and Computation},
  9\penalty0 (4):\penalty0 389--419, 1980.

\bibitem[Poggio et~al.(2017)Poggio, Mhaskar, Rosasco, Miranda, and
  Liao]{poggio2017and}
Tomaso Poggio, Hrushikesh Mhaskar, Lorenzo Rosasco, Brando Miranda, and Qianli
  Liao.
\newblock Why and when can deep-but not shallow-networks avoid the curse of
  dimensionality: a review.
\newblock \emph{International Journal of Automation and Computing}, 14\penalty0
  (5):\penalty0 503--519, 2017.

\bibitem[Poole et~al.(2016)Poole, Lahiri, Raghu, Sohl-Dickstein, and
  Ganguli]{poole2016exponential}
Ben Poole, Subhaneil Lahiri, Maithra Raghu, Jascha Sohl-Dickstein, and Surya
  Ganguli.
\newblock Exponential expressivity in deep neural networks through transient
  chaos.
\newblock \emph{Advances in neural information processing systems},
  29:\penalty0 3360--3368, 2016.

\bibitem[Safran and Shamir(2016)]{safran2016quality}
Itay Safran and Ohad Shamir.
\newblock On the quality of the initial basin in overspecified neural networks.
\newblock In \emph{International Conference on Machine Learning}, pages
  774--782, 2016.

\bibitem[Safran and Shamir(2017)]{safran2017depth}
Itay Safran and Ohad Shamir.
\newblock Depth-width tradeoffs in approximating natural functions with neural
  networks.
\newblock In \emph{International Conference on Machine Learning}, pages
  2979--2987. PMLR, 2017.

\bibitem[Safran et~al.(2019)Safran, Eldan, and Shamir]{safran2019depth}
Itay Safran, Ronen Eldan, and Ohad Shamir.
\newblock Depth separations in neural networks: what is actually being
  separated?
\newblock In \emph{Conference on Learning Theory}, pages 2664--2666. PMLR,
  2019.

\bibitem[Saumard and Wellner(2014)]{saumard2014log}
Adrien Saumard and Jon~A Wellner.
\newblock Log-concavity and strong log-concavity: a review.
\newblock \emph{Statistics surveys}, 8:\penalty0 45, 2014.

\bibitem[Shalev-Shwartz and Ben-David(2014)]{shalev2014understanding}
Shai Shalev-Shwartz and Shai Ben-David.
\newblock \emph{Understanding machine learning: From theory to algorithms}.
\newblock Cambridge university press, 2014.

\bibitem[Telgarsky(2016)]{telgarsky2016benefits}
Matus Telgarsky.
\newblock Benefits of depth in neural networks.
\newblock In \emph{Conference on learning theory}, pages 1517--1539. PMLR,
  2016.

\bibitem[Vardi and Shamir(2020)]{vardi2020neural}
Gal Vardi and Ohad Shamir.
\newblock Neural networks with small weights and depth-separation barriers.
\newblock \emph{Advances in Neural Information Processing Systems}, 33, 2020.

\bibitem[Vardi et~al.(2021{\natexlab{a}})Vardi, Reichman, Pitassi, and
  Shamir]{vardi2021size}
Gal Vardi, Daniel Reichman, Toniann Pitassi, and Ohad Shamir.
\newblock Size and depth separation in approximating benign functions with
  neural networks.
\newblock In \emph{Conference on Learning Theory}, pages 4195--4223. PMLR,
  2021{\natexlab{a}}.

\bibitem[Vardi et~al.(2021{\natexlab{b}})Vardi, Yehudai, and
  Shamir]{vardi2021learning}
Gal Vardi, Gilad Yehudai, and Ohad Shamir.
\newblock Learning a single neuron with bias using gradient descent.
\newblock \emph{arXiv preprint arXiv:2106.01101}, 2021{\natexlab{b}}.

\bibitem[Venturi et~al.(2021)Venturi, Jelassi, Ozuch, and
  Bruna]{venturi2021depth}
Luca Venturi, Samy Jelassi, Tristan Ozuch, and Joan Bruna.
\newblock Depth separation beyond radial functions.
\newblock \emph{arXiv preprint arXiv:2102.01621}, 2021.

\bibitem[Yarotsky(2017)]{yarotsky2017error}
Dmitry Yarotsky.
\newblock Error bounds for approximations with deep relu networks.
\newblock \emph{Neural Networks}, 94:\penalty0 103--114, 2017.

\bibitem[Yehudai and Shamir(2019)]{yehudai2019power}
Gilad Yehudai and Ohad Shamir.
\newblock On the power and limitations of random features for understanding
  neural networks.
\newblock In \emph{Advances in Neural Information Processing Systems}, pages
  6594--6604, 2019.

\bibitem[Yehudai and Shamir(2020)]{yehudai2020learning}
Gilad Yehudai and Ohad Shamir.
\newblock Learning a single neuron with gradient methods.
\newblock \emph{arXiv preprint arXiv:2001.05205}, 2020.

\end{thebibliography}

		\appendix
		
		\section{Proofs from \secref{sec:prelims} -- Main Results}
			\subsection{Proof of \thmref{thm:main}}\label{app:main_thm_proof}
			
			\begin{proof}
				Starting with the lower bound in the first item in the theorem statement, we have that it follows from \thmref{thm:inapproximability}, since if an algorithm has returned a depth 2 neural network $N$ satisfying
				\[
					\E_{\bx\sim\Dcal_d}\pcc{\p{N(\bx)-f(\bx)}^2} < \frac{c_2}{d^4},
				\]
				then it must hold that the width of the network is at least $c_3\exp(c_4d)$. Thus for the algorithm to return such an output it must run in time at least $c_3\exp(c_4d)$.\footnote{Formally, we have that a Turing machine returning a network of width at least $c_3\exp(c_4d)$ will require exponential time to output the network's weights on its tape.}
				
				Turning to show the second item, we first argue that if learning is computationally tractable for a certain sequence of distributions, then it is also tractable for any other sequence of distributions which is indistinguishable by polynomial-time sampling (see \citep[Def.~3.2.4]{goldreich2003foundations}). This is true since otherwise we would be able to distinguish the two sequences using the learning algorithm which would succeed on the first sequence but would fail on the second.
				
				Next, we will show that a sequence which is indistinguishable by polynomial-time sampling from the sequence defined by the densities in \eqref{eq:mu_d_def}, satisfies Assumption~\ref{asm:d_dist_assumption}. % with $C=0.25,q=\Ocal(d^3),c_1=2,c_2=1$.
				Consider a random variable $X_d$ which follows a distribution with density given in \eqref{eq:mu_d_def}. Since this is a mixture of radial distributions, it is clearly radial. We will now show that the density of $\norm{X_d}$ is given by
				\[
					\gamma_d(x)\coloneqq0.5\bar{\gamma}_d(x) + 	0.5p_d(x),
				\]
				where
				\begin{equation}\label{eq:density_components}
					\bar{\gamma}_d(x) \coloneqq \frac{d}{x}J_{d/2}^2(2\pi\alpha\sqrt{d}R_dx),\hskip 0.5cm p_d(x) \coloneqq \frac{2(d/2)^{d/2}}{\Gamma(d/2)}x^{d-1}\exp\p{-0.5dx^2}.
				\end{equation}
				For the multivariate normal component, we have that it follows a scaled chi-distribution with $d$ degrees of freedom, given by $\frac{1}{\sqrt{d}}\chi_d$. Since a $\chi_d$ distribution is a special case of the generalized gamma distribution $GG(\sqrt{2},d,2)$, and since scaling a generalized gamma random variable by a constant $\frac{1}{\sqrt{d}}$ results in a generalized gamma random variable with parameters $GG(\sqrt{2/d},d,2)$, we have that its cumulative distribution function is given by
				\[
					P\p{\frac d2,\frac{dx^2}{2}},
				\]
				where $P(\cdot,\cdot)$ denotes the regularized (lower) incomplete gamma function. Its density is therefore given by $p_d(x)$.
				
				Moving to the density of the norm of the $\bar{\mu}_d$ component, we let $A_d \coloneqq \frac{d\pi^{d/2}}{\Gamma(\frac{d}{2}+1)}$ denote the volume of the unit hypersphere in $\reals^d$ and we let $Y$ denote a random variable with density $\bar{\mu}_d$. Changing to polar coordinates, we have
				\begin{align*}
					\pr\pcc{\norm{Y}\le x}&=\int_{\set{\bx:\norm{\bx}\le x}}\bar{\mu}_d(\bx)d\bx =
					\int_{\set{\bx:\norm{\bx}\le x}}\p{\frac{R_d}{\norm{\bx}}}^d J_{d/2}^2(2\pi\alpha\sqrt{d} R_d\norm{\bx})d\bx \\
					&= \int_{0}^{x}\int_{\by\in\mathbb{S}_t^{d-1}} \p{\frac{R_d}{t}}^d J_{d/2}^2\p{2\pi\alpha\sqrt{d} R_dt}d\by dt \\
					&=\int_{0}^{x}A_dt^{d-1}\p{\frac{R_d}{t}}^d J_{d/2}^2\p{2\pi\alpha\sqrt{d} R_dt} dt
					=\int_{0}^{x}\frac{d}{t}J_{d/2}^2(2\pi\alpha\sqrt{d}R_dt)dt.
				\end{align*}
				Thus, the density of the norm of a random variable with density given in \eqref{eq:mu_bar_d_density} is given by $\bar{\gamma}_d(\cdot)$.
				
				We now prove each of the required properties.
				\begin{enumerate}
					\item
					It will suffice to lower bound the norm of the multivariate normal component. Using a standard bound on $P(\cdot,\cdot)$ (see Eq.~(8.10.13) in \citet{NIST:DLMF}) we have
					\[
						P\p{\frac d2,\frac{d}{2}} > \frac12,
					\]
					and thus
					\[
						\pr\pcc{X_d\le2} \ge \pr\pcc{X_d\le1} \ge \frac14.
					\]
					\item
					$p(x)>0$ for all $x>0$ follows immediately from the definition of $p_d(\cdot)$ in \eqref{eq:density_components}. Note that even though $p(0)=0$, we can mix the multivariate normal distribution with another radial $d$-dimensional distribution that has positive density at the origin (e.g.\ the density on the closed ball $B_r(\mathbf{0})=\{\bx:\norm{\bx}\le r\}$ whose pushforward via the function $\bx\mapsto x_1$ onto $[-r,r]$ is the uniform measure). By taking $r$ exponentially small in $d$ and using an appropriately sufficiently small weight, we can make the two distributions statistically indistinguishable by polynomial-time sampling, and therefore we may assume without loss of generality that $p(0)>0$.
					
					To upper bound the density $\gamma_d$, we begin with upper bounding $0.5p_d(x)$. The mode of a generalized gamma random variable $GG(\sqrt{2/d},d,2)$ where $d>1$ is attained at $x=\sqrt{1-1/d}$, and we thus have
					\[
						\sup_x0.5p_d(x) = 	\frac{(d/2)^{d/2}}{\Gamma(d/2)}\p{1-\frac1d}^{(d-1)/2}\exp\p{-\frac d2+0.5}.
					\]
					Plugging $d=2$ in the above yields an upper bound of $1$. Using a standard lower bound on the gamma function we have for all $d\ge3$
					\[
						\Gamma(d/2) > \sqrt{2\pi}\p{\frac d2-1}^{\frac d2-\frac12}\exp\p{-\frac d2+1},
					\]
					which when plugged in the above and simplified a bit yields 
					\[
						\sup_x0.5p_d(x) < 	\frac{1}{\sqrt{2\pi}}\exp(-0.5)\sqrt{\frac d2}\p{1+\frac{1}{d-2}}^{(d-1)/2} \le \sqrt{\frac{d}{2}},
					\]
					where the second inequality is due to $(1+1/(d-2))^{(d-1)/2}\le2$ for all $d\ge3$ and $\frac{1}{\sqrt{2\pi}}\exp(-0.5)\le0.5$. Upper bounding $0.5\bar{\gamma}_d(x)$ we have from Eq.~(10.14.4) in \citet{NIST:DLMF} that
					\[
						\gamma_d(x) = \frac dx J_{d/2}^2(\pi\alpha\sqrt{d}R_dx) \le \frac 	dx\cdot\frac{(\pi\alpha\sqrt{d}R_dx)^d}{\Gamma^2(d/2+1)}.
					\]
					Assuming $x\le(\pi\alpha\sqrt{d}R_d)^{-2}<1$ and recalling that $d\ge2$, the above is at most
					\[
						\frac 	dx\cdot\frac{x^{d/2}}{\Gamma^2(d/2+1)} \le d.
					\]
					Otherwise, if $x>(\pi\alpha\sqrt{d}R_d)^{-2}$, we use Eq.~(10.14.1) in \citet{NIST:DLMF} and \citet[Lemma~1]{eldan2016power} to deduce that
					\[
						\gamma_d(x) = \frac dx 	J_{d/2}^2(\pi\alpha\sqrt{d}R_dx) \le \frac dx < d(\pi\alpha\sqrt{d}R_d)^2 \le \frac{\pi^2\alpha^2}{4}d^3.
					\]
					We conclude that there exists a polynomial $q$ of degree $3$ which satisfies Item 2 in the assumption.
					\item
					By \citet[Lemma~14]{eldan2016power} we have for all $x\ge1$ that
					\begin{equation*}
						\int_{x}^{\infty}\gamma_d(t)dt = 	\int_{x}^{\infty}\frac dt J_{d/2}^2(\pi\alpha\sqrt{d}R_dt)dt \le \int_{x}^{\infty}\frac dt\cdot\frac{1.3}{t\alpha d}dt = \frac{1.3}{\alpha}\pcc{-t^{-1}}_{x}^{\infty} = \frac{1.3}{\alpha x}.
					\end{equation*}
					For the multivariate component, we compute for any $x\ge2$
					\[
						\pr\pcc{\frac{1}{\sqrt{d}}\chi_d\ge x} 	= \pr\pcc{\chi_d^2\ge x^2d} \le (x^2\exp(1-x^2))^{d/2}\le (x^2\exp(1-x^2)) \le \frac1{2x},
					\]
					where the first inequality is a standard tail bound on chi-squared random variables, the second inequality is due to $\sup_x x^2\exp(1-x^2)\le1$ and $d\ge2$, and the last inequality holds for all $x\ge2$. Combining the last two displayed inequalities we have for all $x\ge2$ that
					\[
						\pr\pcc{X_d\ge x} \le 	\frac{1}{2}\p{\frac{1.3}{\alpha x} + \frac1{2x}} \le \frac1x,
					\]
					where we used the fact that $\alpha\ge1$.
					
				\end{enumerate}
			\end{proof}

			\subsection{Proof of \thmref{thm:secondary}}\label{app:secondary_main_proof}
			Before we prove the theorem, we will first state and prove the following lemmas. Throughout this subsection of the appendix, we use $\Lcal(\btheta)$ to denote the objective function obtained by making predictions with an architecture as defined in \eqref{eq:depth2}, over $n$ data instances sampled from the distribution $\Dcal'_d$ defined in \eqref{eq:daniely_dist_def}, and using the squared loss. We also define the set
			\begin{align}
				&\Fcal\p{d,m}\coloneqq \left\{f:\reals_{+}^2\to\reals\mid \exists c_1,c_2>0\hskip 0.2cm\text{s.t.}\hskip 0.2cm\forall d,m>0,\right.\nonumber\\
				&\left.
				\hskip 2.5cm f(d,m)\ge c_1\exp\p{c_2\min\set{m\ln\p{\frac dm +2}, d\ln\p{\frac md +2}}}\right\}, \label{eq:F_d_m}
			\end{align}
			which represents the set of functions that grow exponentially fast in at least one of the parameters $d$ or $m$. The lemma below establishes some useful properties of the function class $\Fcal(d,m)$.
			\begin{lemma}\label{lem:F_d_m_properties}
				Suppose that $f\in\Fcal\p{d,m}$ where $\Fcal\p{d,m}$ is defined in \eqref{eq:F_d_m}. Then
				\begin{enumerate}
					\item
					For any $c,k>0$,
					\[
						c\sqrt[k]{f(d,m)}\in\Fcal\p{d,m}.
					\]
					\item
					For any $c>0$,
					\[
						f(c\cdot d,m)\in\Fcal(d,m)\hskip 0.3cm\text{and}\hskip 0.3cm f(d,c\cdot m)\in\Fcal(d,m).
					\]
					\item
					For any polynomial $p$ in $d,m$, we have
					\[
						\frac{1}{p(d,m)}f(d,m)\in\Fcal\p{d,m}.
					\]
				\end{enumerate}
			\end{lemma}
			
			\begin{proof}~
				\begin{enumerate}
					\item
					This follows immediate from the definition of $\Fcal\p{d,m}$ by taking the constants $c\sqrt[k]{c_1}>0$ and $c\cdot c_2/k>0$ instead of $c_1$ and $c_2$.
					\item
					From symmetry, it suffices to show $f(c\cdot d,m)\in\Fcal(d,m)$.
					First suppose that $c\ge1$, then we have
					\[
						\frac {cd}{m} +2 \ge \frac {d}{m} +2
					\]
					and
					\[
						cd\ln\p{\frac {m}{cd} +2} \ge d\ln\p{\frac {m}{d} +2},
					\]
					since $x\mapsto x\ln\p{a/x+2}$ is increasing in $x$ for all $a,x>0$. We therefore have
					\[
						f(c\cdot d,m)\ge c_1\exp\p{c_2\min\set{m\ln\p{\frac {d}{m} +2}, d\ln\p{\frac {m}{d} +2}}}\in \Fcal\p{d,m}.
					\]
					Suppose that $c<1$. Then
					\[
						m\ln\p{\frac {cd}{m} +2} \ge cm\ln\p{\frac {cd}{cm} +2} = cm\ln\p{\frac {d}{m} +2},
					\]
					where again we used the fact that $x\mapsto x\ln\p{a/x+2}$ is increasing in $x$ for all $a,x>0$. Furthermore, we have
					\[
						cd\ln\p{\frac {m}{cd} +2} \ge cd\ln\p{\frac {cm}{cd} +2} = cd\ln\p{\frac {m}{d} +2}.
					\]
					Combining these two inequalities we obtain
					\[
						f(c\cdot d,m)\ge c_1\exp\p{c\cdot c_2\min\set{m\ln\p{\frac {d}{m} +2}, d\ln\p{\frac {m}{d} +2}}}\in \Fcal\p{d,m}.
					\]
					\item
					Since any polynomial in $d,n$ can be upper bound by $cd^km^k$ for sufficiently large $c>0$ and natural $k$, it suffices to show this item for $p(d,m)=d^km^k$. By symmetry, showing this for $d^k$ and $m^k$ is equivalent, and by applying this rule iteratively for $\frac1m$ (or equivalently for $\frac1d$), it suffices to prove this item for just the case $p(d,m)=\frac1m$. We have by assumption that there exist $c_1,c_2>0$ such that for all $d,m>0$ we have
					\[
						\frac1mf(d,m) \ge \frac1m c_1\exp\p{c_2\min\set{m\ln\p{\frac dm +2}, d\ln\p{\frac md +2}}}.
					\]
					We first handle the case where $m>d$, in which the above is at least
					\begin{align*}
						\frac1m c_1\exp\p{c_2d\ln\p{\frac md +2}} &= \frac1m c_1\p{\frac md +2}^{c_2d} \ge \frac1m c_1\p{\frac md +1}^{c_2d}\\
						&= \frac1m c_1\p{\frac md +1}^{c_2d/2}\p{\frac md +1}^{c_2d/2}\\
						&\ge \frac1mc_1\p{\frac{c_2m}{2}+1}\p{\frac md +1}^{c_2d/2} \ge \frac12c_1c_2\p{\frac md +1}^{c_2d/2}\\
						&=\frac12c_1c_2\exp\p{c_2\frac{d}{2}\ln\p{\frac md+1}}\\
						&\ge \frac12c_1c_2\exp\p{c_2\frac{d}{4}\ln\p{\frac md+2}}\\
						&= \frac12c_1c_2\exp\p{\frac{c_2}{4}\min\set{d\ln\p{\frac md+2},m\ln\p{\frac dm+2}}} \\ &\in \Fcal\p{d,m}.
					\end{align*}
					In the above, the second inequality follows from Bernoulli's inequality and the fourth inequality follows from the inequality $0.5\ln(2+x)/\ln(1+x)\le1$ which holds for all $x\ge1$ and since we assume $m>d$. Turning to the case $m\le d$, we have
					\[
						\frac1mf(d,m) \ge \frac1df(d,m),
					\]
					where by symmetry we have that the proof follows from the same reasoning as in the previous case, where the roles of $d$ and $m$ are switched.
				\end{enumerate}
			\end{proof}
			
			The following lemma bounds the loss of the objective function in terms of the Euclidean norm of the weights and the width of the network.
			\begin{lemma}\label{lem:bounded_loss}
				Suppose that $\sigma$ satisfies Assumption~\ref{asm:poly_bounded}. Then there exists a polynomial $p$ with coefficients and degree that depend solely on $\sigma$ such that
				\[
					\Lcal(\btheta)\le r^2p\p{\norm{\btheta}_2}.
				\]
			\end{lemma}
			
			\begin{proof}
				We begin with bounding the output of the hidden layer. For any data point $\bx_i$ we have from Cauchy-Schwartz that $\inner{\bu_j,\bx_i}\le\norm{\bu_j}_2\norm{\bx_i}_2\le2\norm{\bu_j}$, where we used the fact that $\Dcal'_d$ is supported on a ball of radius $2$. By Assumption~\ref{asm:poly_bounded} we have constants $C_{\sigma},\alpha_{\sigma}>0$ such that
				\[
					\sigma\p{\inner{\bu_j,\bx_i}+b_j} \le C_{\sigma}\p{1+\abs{2\norm{\bu_j}_2+\abs{b_j}}^{\alpha_{\sigma}}} \le C_{\sigma}\p{1+\abs{3\norm{\btheta}_2}^{\alpha_{\sigma}}}.
				\]
				Using the above, we get
				\begin{align*}
					\Lcal(\btheta)&=
					\frac1n\sum_{i=1}^{n}\p{\sum_{j=1}^{r}w_j\sigma\p{\inner{\bu_j,\bx_i}+b_j}+b_0-y_i}^2\\ &\le
					\frac2n\sum_{i=1}^{n}\p{\sum_{j=1}^{r}w_j\sigma\p{\inner{\bu_j,\bx_i}+b_j}+b_0}^2+2\\
					&\le \frac2n\sum_{i=1}^{n}\p{rC_{\sigma}\norm{\btheta}_2 \p{1+\abs{3\norm{\btheta}_2}^{\alpha_{\sigma}}}+\norm{\btheta}_2}^2 + 2\\
					&\le 2r^2\p{\p{C_{\sigma}\norm{\btheta}_2 \p{1+\abs{3\norm{\btheta}_2}^{\alpha_{\sigma}}}+\norm{\btheta}_2}^2 + 1},
				\end{align*}
				where the first inequality uses the inequality $(a+b)^2\le 2(a^2+b^2)$ which holds for all $a,b$. The lemma then follows from the above bound.
			\end{proof}
			
			\begin{proof}[Proof of \thmref{thm:secondary}]
				We start the proof with showing the lower bound. Suppose we train a neural network with architecture defined in \eqref{eq:depth2} using GD for $T\ge4$ iterations. Then we have that the running time of the algorithm is lower bounded by $\max\{r,T\}$. Further assume that the algorithm returned a set of weights $\btheta_T$ such that
				\[
					\Lcal(\btheta_T) \le \epsilon \le \frac{1}{400}.
				\]
				The bulk of the remainder of the lower bound proof focuses on lower bounding $T$ in terms of the norm of $\btheta_T$. Since our bound on $\norm{\btheta_T}_2$ will be shown to be increasing with $T$, it suffices to derive the bound for just the iterate at the final iteration $T$ in order to obtain a lower bound on the loss of all the other iterates $t<T$ as well. We compute
				\begin{equation}\label{eq:grad_sum_bound}
					\norm{\btheta_T-\btheta_0}_2 \le \sum_{t=1}^{T}\norm{\btheta_{t}-\btheta_{t-1}}_2 = \eta\sum_{t=1}^{T}\norm{\nabla\Lcal(\btheta_t)}_2\le\eta \sqrt{T\sum_{t=1}^{T}\norm{\nabla\Lcal(\btheta_t)}_2^2},
				\end{equation}
				where the first inequality is due to the triangle inequality, the equality is due to the update rule $\btheta_{t+1}=\btheta_t-\eta\nabla\Lcal(\btheta_t)$, and the second inequality is due to Cauchy-Schwartz. We now argue that $\Lcal(\cdot)$ is $L$-smooth for some positive $L>0$. This follows from our assumption $\sigma\in C^2(\reals)$ which implies that $\nabla^2\Lcal(\cdot)$ is a continuous mapping, and thus $\norm{\nabla^2\Lcal(\btheta)}_{\text{sp}}$ attains its maximum on any compact domain. In particular, we may define
				\[
					L\coloneqq\max_{\btheta:\norm{\btheta}_2\le cTd^\ell r^\ell}\norm{\nabla^2\Lcal(\btheta)}_{\text{sp}},
				\]
				for sufficiently large $c,\ell>0$ (this norm bound will be justified later in the proof).
				Since $\Lcal(\cdot)$ is $L$-smooth, we have that
				\[
					\Lcal(\btheta_{t+1})-\Lcal(\btheta_t) \le \nabla\Lcal(\btheta_{t})^{\top}(\btheta_{t+1}-\btheta_{t}) + \frac{L}{2}\norm{\btheta_{t+1}-\btheta_{t}}_2^2.
				\]
				Plugging the update rule in the above, simplifying and rearranging yields
				\[
					\p{\eta-\eta^2\frac L2}\norm{\nabla\Lcal(\btheta_{t})}_2^2 \le \Lcal(\btheta_{t})-\Lcal(\btheta_{t+1}).
				\]
				Summing over $t\in[T]$ we have that the right-hand side telescopes, and by the fact that $\Lcal(\btheta_T)\ge0$ we get
				\[
					\sum_{t=1}^{T}\norm{\nabla\Lcal(\btheta_{t})}_2^2 \le \p{\eta-\eta^2\frac L2}^{-1}\Lcal(\btheta_{0}).
				\]
				Plugging the above back in \eqref{eq:grad_sum_bound} we have
				\[
					\norm{\btheta_T-\btheta_0}_2 \le \sqrt{T\Lcal(\btheta_{0})}\cdot\eta\p{\eta-\eta^2\frac L2}^{-0.5} = \sqrt{T\Lcal(\btheta_{0})}\cdot\sqrt{\eta}\p{1-\eta\frac L2}^{-0.5}.
				\]
				Since the function $x\mapsto\sqrt{x}\p{1-x\frac a2}^{-0.5}$ is increasing in $x$ for all $x\in(0,2/a)$ where $a>0$, we have by our assumption $\eta\le\frac2L-\frac1T$ that
				\[
					\sqrt{\eta}\p{1-\eta\frac L2}^{-0.5} \le \sqrt{\frac2L-\frac1T}\p{1-\p{\frac2L-\frac1T}\frac L2}^{-0.5} \le \sqrt{\frac2L}\p{\frac{L}{2T}}^{-0.5} = \frac2L\sqrt{T},
				\]
				which when plugged in the previous inequality yields
				\begin{equation}\label{eq:linear_in_T_bound}
					\norm{\btheta_T-\btheta_0}_2 \le \frac{2T}{L}\sqrt{\Lcal(\btheta_{0})}.
				\end{equation}
				To lower bound $L$, observe that $\Lcal\in C^2(\reals)$ since $\sigma\in C^2(\reals)$ by assumption. Moreover, it is readily seen that in this case we have that $\frac{\partial^2}{\partial^2 b_0}\Lcal(\btheta)=2$. Consider the unit vector $\be_1=(1,0,\ldots,0)\in\reals^{k}$, where $k\coloneqq r(d+2)+1$ is the dimension of our optimization space. If we assume that $\frac{\partial^2}{\partial^2 b_0}\Lcal(\btheta)$ is the first diagonal entry in the Hessian $H(\btheta)\coloneqq\nabla^2\Lcal(\btheta)$, then we have that it equals $2$ for all $\btheta\in\reals^{k}$. Along with the fact that $H(\btheta)$ is symmetric, this implies that for all $\btheta$
				\[
					2=\be_1^{\top}H(\btheta)\be_1 = \bv^{\top}D(\btheta)\bv
				\]
				where $H(\btheta)=O(\btheta)^{\top}D(\btheta)O(\btheta)$ with $D(\btheta)$ diagonal where its entries are the eigenvalues of $H(\btheta)$, $O(\btheta)$ is orthogonal and $\bv=O(\btheta)\be_1$. Since $O(\btheta)$ is an isometry we have $\norm{\bv}_2=1$, implying that
				\[
					2=\sum_{i=1}^{k}d_iv_i^2.
				\]
				Now, if $d_i<\frac{2}{k}$ for all $i$ we have
				\[
					\sum_{i=1}^{k}d_iv_i^2<\frac2k\sum_{i=1}^{k}v_i^2=2,
				\]
				which leads to a contradiction. We therefore have that the largest eigenvalue of $H(\btheta)$ is at least $2/k$, implying $L\ge2/k$ since $\Lcal\in C^2(\reals)$. Plugging this back in \eqref{eq:linear_in_T_bound} we have arrived at
				\[
					\norm{\btheta_T-\btheta_0}_2 \le Tk\sqrt{\Lcal(\btheta_{0})} \le 3drT\sqrt{\Lcal(\btheta_{0})},
				\]
				where we assume $d\ge2$. By the triangle inequality this implies
				\[
					\norm{\btheta_T}_2 \le \norm{\btheta_T-\btheta_0}_2 + \norm{\btheta_0}_2 \le 3drT\sqrt{\Lcal(\btheta_{0})} + \norm{\btheta_0}_2 \le 3drT\p{\sqrt{\Lcal(\btheta_{0})} + \norm{\btheta_0}_2}.
				\]
				From \lemref{lem:bounded_loss} and our initialization assumption, we have that the above can be upper bounded by $cTd^\ell r^\ell$ with probability at least $1-\exp(-d)$ for sufficiently large $c,\ell>0$. Note that this entails a lower bound on $T$ as follows
				\[
					T\ge\frac{\norm{\btheta_T}_2}{cd^{\ell}r^{\ell}},
				\]
				which implies a lower bound on the running time given by
				\[
					\max\{r,T\}\ge\max\set{r, \frac{\norm{\btheta_T}_2}{cd^{\ell}r^{\ell}}} \ge \frac{1}{d^\ell}\min\set{1,\frac1c}\cdot\max\set{r, \frac{\norm{\btheta_T}_2}{r^{\ell}}}.
				\]
				By \lemref{lem:F_d_m_properties}, it will suffice to show that
				\[
					\max\set{r, \frac{\norm{\btheta_T}_2}{r^{\ell}}}\in\Fcal\p{d,\frac{1}{\sqrt{\epsilon}}}.
				\]
				To this end, we first invoke \thmref{thm:daniely_inapproximability} which by our assumption guarantees that
				\[
					\max\set{r, \norm{\btheta_T}_2}\in\Fcal\p{d,m},
				\]
				where $m$ is the largest integer such that $\epsilon\le\frac{1}{196(m+1)^2}$, implying that $m\ge2$ by our assumption $\epsilon\le\frac{1}{400}$. We now split our analysis into several cases. 
				\begin{itemize}
					\item
					First assume that $r>\norm{\btheta_T}_2$. Then in this case we immediately have that
					\[
						\max\set{r,\frac{\norm{\btheta_T}_2}{r^{\ell}}}=r=\max\set{r, \norm{\btheta_T}_2}\in\Fcal\p{d,m}.
					\]
					\item
					Assuming $r\le\norm{\btheta_T}_2$ and $\sqrt{\norm{\btheta_T}_2}r^{-\ell}\ge1$, we have
					\[
						\max\set{r,\frac{\norm{\btheta_T}_2}{r^{\ell}}} \ge \max\set{r, \sqrt{\norm{\btheta_T}_2}} \ge \sqrt{\norm{\btheta_T}_2} \in \Fcal\p{d,m},
					\]
					where the inclusion follows from \lemref{lem:F_d_m_properties}.
					\item
					Assuming $r\le\norm{\btheta_T}_2$ and $\sqrt{\norm{\btheta_T}_2}r^{-\ell}<1$, we have
					\[
						\max\set{r,\frac{\norm{\btheta_T}_2}{r^{\ell}}} \ge r > \sqrt[2\ell]{\norm{\btheta_T}_2}\in \Fcal\p{d,m},
					\]
					where once again the inclusion follows from \lemref{lem:F_d_m_properties}.
				\end{itemize}
				We conclude the proof by recalling that $m$ is the largest integer such that $\frac{1}{196(m+1)^2}\ge\epsilon$ and observing that our bound on $\epsilon$ of $\frac{1}{400}$ implies that
				\[
					3\le m+1\le\frac{1}{14\sqrt{\epsilon}},
				\]
				from which we have $1\le\frac{1}{42\sqrt{\epsilon}}$ and thus
				\[
					2\le\frac{1}{21\sqrt{\epsilon}} = \frac{1}{14\sqrt{\epsilon}} - \frac{1}{42\sqrt{\epsilon}}\le m+1\le\frac{1}{14\sqrt{\epsilon}}.
				\]
				That is, for any $\epsilon\le\frac{1}{400}$ we have an integer $m$ such that $m=\Theta(1/\sqrt{\epsilon})$ and $\max\{r,T\}\in\Fcal\p{d,m}$, and therefore from \lemref{lem:F_d_m_properties} we get
				\[
					\max\{r,T\} \in \Fcal\p{d,\frac{1}{\sqrt{\epsilon}}}.
				\]
				
				Turning to the positive optimization result, we first argue that $\Dcal'_d$ has a radial density which follows from the fact that the sum of two radial distributions is also radial, since each individual distribution is invariant to radial transformations and therefore so is their sum. Next, we prove each item in Assumption~\ref{asm:d_dist_assumption}, where we denote the random variable representing the norm of the distribution $\Dcal'_d$ by $X'_d$, and its density by $\gamma'_d$.
				\begin{enumerate}
					\item
					Since $\Dcal'_d$ is supported on $\set{\bx:\norm{\bx}\le2}$, we have that
					\[
						\int_{0}^{2} \gamma'_d(x)dx = 1.
					\]
					\item
					By \citet[Lemma~A.2]{vardi2020neural} and a straightforward change of variables, we have for all $x\in[0,2]$ that
					\[
						\gamma'_d(x) = \frac{\Gamma\p{\frac{d}{2}}}{\sqrt{\pi}\Gamma\p{\frac{d-1}{2}}}x^{d-2}(1-0.25x^2)^{\frac{d-3}{2}}
					\]
					(where for $x>2$ the density is $0$). Upper bounding the expression above, we compute
					\[
						x^{d-2}(1-0.25x^2)^{\frac{d-3}{2}} = (x^2)^{\frac{d-2}{2}}(1-0.25x^2)^{\frac{d-3}{2}} \le (x^2-0.25x^4)^{\frac{d-3}{2}} \le 1,
					\]
					where the last inequality holds by the assumption $d\ge3$ and the fact that $\max_{x\in[0,2]} x^2-0.25x^4=1$. Using Eq.~5.6.4 in \citet{NIST:DLMF}, we have $\Gamma\p{\frac{d}{2}}/\Gamma\p{\frac{d-1}{2}}<\sqrt{\frac d2}$, allowing us to deduce an upper bound of
					\[
						\gamma'_d(x) \le \sqrt{d}
					\]
					for all $x\in[0,2]$ and $d\ge3$. We remark that $\gamma'_d(x)>0$ holds for all $x\in[0,2]$ except for the boundary, but this poses no problem since we may without loss of generality mix the distribution with another with positive density at the boundary while keeping the analysis unchanged (see the explanation in the beginning of \appref{app:main_thm_proof} for justification).
					\item
					Once again, using the fact that $\Dcal'_d$ is supported on $\set{\bx:\norm{\bx}\le2}$, we have for all $x\ge2$ that
					\[
						\pr\pcc{X'_d\ge x}= 0\le \frac{1}{x}.
					\]
				\end{enumerate}
 			\end{proof}

			\section{Proofs from \secref{sec:lower_bounds} -- Approximation Lower Bounds}
			\subsection{Proof of \thmref{thm:inapproximability}}\label{app:inapproximability_proof}
			\begin{proof}
				Our proof builds on the following main result in \citet{eldan2016power}, which is restated here for completeness in a slightly different manner to fit our setting.
				\begin{theorem}[\citet{eldan2016power}]
					The following holds for some universal constants $c_1,c_2,c_3,c_4>0$, integer $\gamma\ge1$, constant $\alpha\ge1$ and any network employing an activation function satisfying Assumptions 1 and 2 in \citet{eldan2016power}: For all $d\ge c_1$, there exist constants $\lambda_d^{(1)},\ldots,\lambda_d^{(k)}\in[1,2]$ and $\epsilon_d^{(1)},\ldots,\epsilon_d^{(k)}\in\{-1,1\}$ where $k=\gamma d^2$, such that for any neural network $N$ of depth 2 and width at most $c_3\exp(c_4d)$ we have
					\[
						\E_{\bx\sim\Dcal_d}\pcc{\p{ \sum_{i=1}^{k}\epsilon_d^{(i)}\one{\norm{\bx}\le\alpha\sqrt{d}\lambda_d^{(i)}} -N(\bx)}^2}\ge c_2,
					\]
					where the density of $\Dcal_d$ is given by
					\[
						\hat{\mu}(\bx)\coloneqq\p{\frac{R_d}{\norm{\bx}}}^d J_{d/2}^2(2\pi R_d\norm{\bx}).
					\]
				\end{theorem}
				Writing the implication of the above theorem in integral form, we have that
				\[
					\int_{\reals^d}\p{ \sum_{i=1}^{k}\epsilon_d^{(i)}\one{\norm{\bx}\le\alpha\sqrt{d}\lambda_d^{(i)}} -N(\bx)}^2\hat{\mu}(\bx)d\bx \ge c_2.
				\]
				We now perform a change of variables $\bx=\alpha\sqrt{d}\bz$, $d\bx=\frac{1}{(\alpha\sqrt{d})^d}d\bz$ to obtain
				\[
					\int_{\reals^d}\p{ \sum_{i=1}^{k}\epsilon_d^{(i)}\one{\norm{\bz}\le\lambda_d^{(i)}} -N(\alpha\sqrt{d}\bz)}^2\frac{1}{(\alpha\sqrt{d})^d}\hat{\mu}(\alpha\sqrt{d}\bz)d\bz \ge c_2.
				\]
				Since $N(\alpha\sqrt{d}\bz)$ computes the same function as the neural network $N'(\bz)$ where the weights of the neurons in the hidden layer are multiplied by $\alpha\sqrt{d}$, and by defining the measure $\bar{\mu}(\bx)\coloneqq\frac{1}{(\alpha\sqrt{d})^d}\hat{\mu}(\alpha\sqrt{d}\bx)$ (note that this is indeed a measure as readily seen by the change of variables $\by=\alpha\sqrt{d}\bx$, $d\by=\frac{1}{(\alpha\sqrt{d})^d}d\bx$, which yields $\int\bar{\mu}(\bx)d\bx = \int\frac{1}{(\alpha\sqrt{d})^d}\hat{\mu}(\alpha\sqrt{d}\bx)d\bx = \int\hat{\mu}(\by)d\by=1$), we can rewrite the above as
				\begin{equation}\label{eq:scaled_thin_shells}
					\int_{\reals^d}\p{ \sum_{i=1}^{k}\epsilon_d^{(i)}\one{\norm{\bx}\le\lambda_d^{(i)}} -N(\bx)}^2\bar{\mu}(\bx)d\bx \ge c_2,
				\end{equation}
				for any depth 2 neural network $N$ of width at most $c_3\exp(c_4d)$. Fix any $d\ge c_1$. We assume by contradiction that for all $i\in[k]$, we have a depth 2 neural network $N_i$ of width less than $c_3\exp(c_4d)/(\gamma d^2)$ which satisfies
				\[
					\int_{\reals^d}\p{\one{\norm{\bx}\le\lambda_d^{(i)}} - N_i(\bx)}^2\bar{\mu}(\bx)d\bx < \frac{c_2}{\gamma^2d^4}.
				\]
				Letting $\norm{f}_{L(\bar{\mu})} \coloneqq \sqrt{\int_{\reals^d} f^2(\bx)\bar{\mu}(\bx)d\bx}$, we now compute
				\begin{align*}
					&\norm{ \sum_{i=1}^{k}\epsilon_d^{(i)}\one{\norm{\bx}\le\lambda_d^{(i)}} - \sum_{i=1}^{k}\epsilon_d^{(i)}N_i(\bx)}_{L(\bar{\mu})} \le \sum_{i=1}^{k}\abs{\epsilon_d^{(i)}}\cdot\norm{\one{\norm{\bx}\le\lambda_d^{(i)}} -N_i(\bx)}_{L(\bar{\mu})}\\
					&\hskip 1cm
					<\sum_{i=1}^{k}\abs{\epsilon_d^{(i)}}\cdot\sqrt{\frac{c_2}{\gamma^2d^4}} \le \sqrt{c_2},
				\end{align*}
				where the first inequality is the triangle inequality and the second is by our assumption. But since $\sum_{i=1}^{k}\epsilon_d^{(i)}N_i(\bx)$ can be computed by a depth 2 neural network of width less than $c_3\exp(c_4 d)$ by concatenating the $k$ networks $N_i$, $i=1,\ldots,k$, the square of the above contradicts \eqref{eq:scaled_thin_shells}. We thus have that for all $d\ge c_1$, there exists $i\in[k]$ such that for any depth 2 neural network $N$ of width less than $c_3\exp(c_4d)/(\gamma d^2)$, we have
				\[
					\int_{\reals^d}\p{\one{\norm{\bx}\le\lambda_d^{(i)}} - N(\bx)}^2\bar{\mu}(\bx)d\bx \ge \frac{c_2}{\gamma^2d^4},
				\]
				which for the measure $\mu$ defined in \eqref{eq:mu_d_def} implies
				\[
					\int_{\reals^d}\p{\one{\norm{\bx}\le\lambda_d^{(i)}} - N(\bx)}^2\mu(\bx)d\bx \ge \frac{c_2}{2\gamma^2d^4}.
				\]
				Scaling $c_2$ by $0.5\gamma^{-2}$ and since we can take $c'_3,c'_4>0$ small enough so that
				\[
					\frac{c_3\exp(c_4d)}{\gamma d^2} \ge c'_3\exp(c'_4d)
				\]
				for all $d\ge c_1$, the statement of the theorem follows.
			\end{proof}
		
			\subsection{Proof of \thmref{thm:daniely_inapproximability}}\label{app:daniely_inapprox_proof}
			To prove the theorem, we will first need the following lemma and proposition. The lemma below shows that functions that oscillate $2m+2$ many times cannot be approximated well using polynomials of degree at most $m$ with respect to a certain weight function.
			\begin{lemma}\label{lem:best_degree_m}
				For any polynomial $p$ of degree at most $m$, we have
				\[
					\int_{-1}^{1} \p{p(x)-\bar{f}_m(x)}^2\frac{\Gamma\p{\frac{d}{2}}}{\sqrt{\pi}\Gamma\p{\frac{d-1}{2}}}(1-x^2)^{\frac{d-3}{2}}dx \ge \frac{1}{36}.
				\]
			\end{lemma}
			
			\begin{proof}
				Since the integrand is non-negative, we can lower bound the integral in the lemma by
				\[
					\int_{-\frac{1}{2\sqrt{d}}}^{0} \p{p(x)-\bar{f}_m(x)}^2\frac{\Gamma\p{\frac{d}{2}}}{\sqrt{\pi}\Gamma\p{\frac{d-1}{2}}}(1-x^2)^{\frac{d-3}{2}}dx.
				\]
				Assuming $d\ge4$, then $\Gamma\p{\frac{d}{2}}/\Gamma\p{\frac{d-1}{2}}\ge\frac{\sqrt{d}}{2}$ (see Eq.~5.6.4 in \citet{NIST:DLMF}) and since
				\[
					\frac{1}{\sqrt{\pi}}\inf_{x\in[-0.5d^{-0.5},0]}(1-x^2)^{(d-3)/2}\ge\frac49
				\]
				for all $d\ge4$ we have that the above integral is lower bounded by
				\[
					\frac{2\sqrt{d}}{9}\int_{-\frac{1}{2\sqrt{d}}}^{0} \p{p(x)-\bar{f}_m(x)}^2dx.
				\]
				By its definition, $\bar{f}_m$ alternates signs at least $2m$ times on the domain of integration, and since a degree $m$ polynomial can only change signs at most $m$ times, there are at least $m$ intervals each of which of length $\frac{1}{4m\sqrt{d}}$ (we view the intervals at the boundary of the domain of integration as a single interval of this length whenever $\sqrt{d}$ is not an integer) where $p$ does not change its sign. Moreover, on at least $m/2$ of these intervals the signs of $p$ and $\bar{f}_m$ are opposite, and the integral on these intervals is lower bounded by their length $\frac{1}{4m\sqrt{d}}$. Combining this with our previously derived lower bound concludes the lemma.
			\end{proof}
			
			The following proposition shows that a depth 2 neural network cannot approximate a certain dot-product function that oscillates $2m+2$ many times, unless its width or the Euclidean norm of its weights is large.
			\begin{proposition}\label{prop:daniely_lower_bound}
				Suppose that a neural network $N_{\btheta}(\cdot)$ with architecture defined in \eqref{eq:depth2} satisfies
				\[
					r,\norm{\btheta}_2\le\frac{1}{10(1+\sqrt{C_{\sigma}})}N_{d,m}^{\alpha'_{\sigma}/6},
				\]
				where
				\[
					N_{d,m}\coloneqq \frac{(2m+d-2)(m+d-3)!}{m!(d-2)!},
				\]
				$\alpha'_{\sigma}\coloneqq\begin{cases}
					1 & \alpha_{\sigma}\le1\\
					\alpha_{\sigma}^{-1} & \alpha_{\sigma}>1
				\end{cases}$, 
				and $C_{\sigma},\alpha_{\sigma}$ are defined in Assumption~\ref{asm:poly_bounded}
				and $\btheta$ are all the trainable weights of $N_{\btheta}(\cdot)$ in vectorized form. Then
				\[
					\norm{N_{\btheta}(\bx)-\bar{f}_m(\inner{\bx_1,\bx_2})}^2_{L_2(\varphi_d)} > \frac{1}{49}
				\]
			\end{proposition}
			
			\begin{proof}
				By \citet[Thm.~4]{daniely2017depth}, we have that
				\begin{align*}
					&\norm{N_{\btheta}(\bx)-\bar{f}_m(\inner{\bx_1,\bx_2})}^2_{L_2(\varphi_d)}\\
					&\hskip 1.5cm \ge \norm{\Pcal_m \bar{f}_m}_{L_2(\varphi_d)} \p{\norm{\Pcal_m \bar{f}_m}_{L_2(\varphi_d)} - \frac{2\sum_{j=1}^{r}\abs{w_j}\cdot\norm{\sigma\p{\inner{\bu_j,\bx}+b_j}}_{L_2(\varphi_d)}+2\abs{b_0}}{\sqrt{N_{d,m}}}},
				\end{align*}
				where $\Pcal_m \bar{f}_m$ is the best degree $m$ polynomial approximation of $x\mapsto\bar{f}_m(x)$ with respect to the density $w(x)\coloneqq\frac{\Gamma\p{\frac{d}{2}}}{\sqrt{\pi}\Gamma\p{\frac{d-1}{2}}}(1-x^2)^{\frac{d-3}{2}}$. By \lemref{lem:best_degree_m} we have
				\begin{equation}\label{eq:16_lower_bound}
					\norm{N(\bx)-\bar{f}_m(\inner{\bx_1,\bx_2})}^2_{L_2(\varphi_d)} \ge \frac16 \p{\frac16 - \frac{2\sum_{j=1}^{r}\abs{w_j}\cdot\norm{\sigma\p{\inner{\bu_j,\bx}+b_j}}_{L_2(\varphi_d)}+2\abs{b_0}}{\sqrt{N_{d,m}}}}.
				\end{equation}
				We now bound the norm of the function implemented by a neuron
				\begin{align*}
					\norm{\sigma\p{\inner{\bu_j,\bx}+b_j}}_{L_2(\varphi_d)} 
					&= \sqrt{\int_{\reals^d}\p{\sigma\p{\inner{\bu_j,\bx}+b_j}}^2\varphi_d(\bx)d\bx}\\
					&\le \sqrt{\int_{\reals^d} C_{\sigma}\p{1+\abs{2\norm{\bu_j}+\abs{b_j}}^{\alpha_{\sigma}}}^2\varphi_d(\bx)d\bx}\\
					&= \sqrt{C_{\sigma}}\p{1+\abs{2\norm{\bu_j}+\abs{b_j}}^{\alpha_{\sigma}}}\\
					&\le \sqrt{C_{\sigma}}\p{1+\p{3\norm{\btheta}}^{\alpha_{\sigma}}},
				\end{align*}
				where the first inequality holds due to Cauchy-Schwartz and Assumption~\ref{asm:poly_bounded}. From the above and an additional application of Cauchy-Schwartz we get
				\begin{align*}
					2\sum_{j=1}^{r}\abs{w_j}\cdot\norm{\sigma\p{\inner{\bu_j,\bx}+b_j}}_{L_2(\varphi_d)} + 2\abs{b_0} &\le 2r\sqrt{C_{\sigma}}\norm{\btheta} \p{1+\p{3\norm{\btheta}}^{\alpha_{\sigma}}} + 2\norm{\btheta}\\
					&\le 2r\norm{\btheta}\p{1+\sqrt{C_{\sigma}}+3^{\alpha_{\sigma}}\norm{\btheta}^{\alpha_{\sigma}}}.
				\end{align*}
				By our assumption in the proposition statement and the definition of $\alpha'_{\sigma}$ which entails $\alpha'_{\sigma},\alpha_{\sigma}\alpha'_{\sigma}\le1$, we have that $2r\norm{\btheta}\le\frac{1}{50}N_{d,m}^{1/3}$ and $1+\sqrt{C_{\sigma}}+3^{\alpha_{\sigma}}\norm{\btheta}^{\alpha_{\sigma}}\le2N_{d,m}^{1/6}$.
				Plugging these two inequalities in \eqref{eq:16_lower_bound} and simplifying, the proposition follows.
			\end{proof}
			
			With the above lemma and proposition, we are ready to prove the theorem.
			\begin{proof}[Proof of \thmref{thm:daniely_inapproximability}]
				Let $\Dcal'_d$ denote the distribution on $\reals^d$ defined in \eqref{eq:daniely_dist_def}. For natural $m\ge1$, consider the function
				\[
					f_m(z)\coloneqq\begin{cases}
						\sign\p{\sin\p{2\pi m\sqrt{d}z^2}} & z^2\in\pcc{2-\frac{1}{\sqrt{d}},2}\\
						0 & \text{otherwise}
					\end{cases},
				\]
				where we define
				\[
					\sign(z)\coloneqq\begin{cases}
						1 & z\ge0\\
						-1 & z<0
					\end{cases}.
				\]
				By the definition of $f_m(\cdot)$, we have that we can write
				\[
					f_m(\norm{\bx}_2)=\sum_{i=1}^{2m+2}\epsilon_d^{(i)}\one{\norm{\bx}_2\le\lambda_d^{(i)}}
				\]
				for appropriately chosen $\epsilon_d^{(1)},\ldots,\epsilon_d^{(2m+2)}\in\{-1,1\}$ and $\lambda_d^{(1)},\ldots,\lambda_d^{(2m+2)}\in[1,2]$. Suppose that for all $i\in[2m+2]$ there exists a neural network $N_{\btheta_i}(\cdot)$ of width $r_i$, having weights $\btheta_i$, and with architecture defined by \eqref{eq:depth2} such that
				\[
					\norm{N_{\btheta_i}^{(i)}(\bx) - \one{\norm{\bx}_2\le\lambda_d^{(i)}}}_{L_2(\varphi_d)}\le\frac{1}{14(m+1)}.
				\]
				Then we can construct a depth 2 neural network $N_{\btheta'}(\cdot)$ of width $r'$ which is at most $(2m+2)r$, having weights satisfying $\norm{\btheta'}\le\sqrt{2m+2}\cdot\max_i\norm{\btheta_i}$ such that
				\[
					N_{\btheta'}(\bx)=\sum_{i=1}^{2m+2}\epsilon_d^{(i)}N_{\btheta_i}^{(i)}(\bx),
				\]
				by taking all the neural networks $N_{\btheta_i}$ and concatenating them appropriately.
				From the above and the triangle inequality we have
				\begin{equation}\label{eq:150_bound}
					\norm{N_{\btheta'}(\bx)-f_m(\norm{\bx}_2)}_{L_2(\varphi_d)} \le \sum_{i=1}^{2m+2}\abs{\epsilon_d^{(i)}}\cdot\norm{N_{\btheta_i}^{(i)}(\bx) - \one{\norm{\bx}_2\le\lambda_d^{(i)}}}_{L_2(\varphi_d)}\le\frac{1}{7}.
				\end{equation}
				Next, we compute for any $\bx=\bx_1+\bx_2$ such that $\bx_1,\bx_2$ have unit norm and $\norm{\bx}_2^2\in\pcc{2-\frac{1}{\sqrt{d}},2}$ to obtain that
				\[
					f_m(\norm{\bx}_2) = \sign\p{\sin\p{2\pi m\sqrt{d}\norm{\bx}_2^2}} = \sign\p{\sin\p{4\pi m\sqrt{d}(1+\inner{\bx_1,\bx_2})}},
				\]
				for any unit $\bx_1,\bx_2$ such that $\inner{\bx_1,\bx_2}\in\pcc{-\frac{1}{2\sqrt{d}},0}$. Define
				\[
					\bar{f}_m(z)\coloneqq 
					\begin{cases}
						\sign\p{\sin\p{4\pi m\sqrt{d}(1+z)}} & z\in\pcc{-\frac{1}{2\sqrt{d}},0}\\
						0 & \text{otherwise}
					\end{cases}.
				\]
				By the above we have
				\[
					f_m(\norm{\bx}_2)=\bar{f}_m(\inner{\bx_1,\bx_2}),
				\]
				which when plugged in \eqref{eq:150_bound} yields
				\[
					\norm{N_{\btheta'}(\bx)-\bar{f}_m(\inner{\bx_1,\bx_2})}_{L_2(\varphi_d)} \le \frac{1}{7}.
				\]
				By \propref{prop:daniely_lower_bound} we have that
				\[
					\max\set{r',\norm{\btheta'}_2} > \frac{1}{10(1+\sqrt{C_{\sigma}})}N_{d,m}^{\alpha'_{\sigma}/6},
				\]
				implying
				\begin{equation}\label{eq:width_norm_bound}
					\max_i\set{\max\set{r_i,\norm{\btheta_i}_2}} > \frac{1}{20(m+1)(1+\sqrt{C_{\sigma}})}N_{d,m}^{\alpha'_{\sigma}/6} \ge \frac{1}{40(1+\sqrt{C_{\sigma}})m}N_{d,m}^{\alpha'_{\sigma}/6}.
				\end{equation}
				Lower bounding $\frac1mN_{d,m}$, we first assume that $m>d$ to obtain
				\begin{align*}
					\frac1mN_{d,m} &\ge \frac1m\binom{m+d-2}{d-2} \ge \frac1m\p{1+\frac{m}{d-2}}^{d-2} \ge \frac1m\p{1+\frac{m}{d}}^{d/2}\\
					&= \frac1m\p{1+\frac{m}{d}}^{d/4}\cdot\p{1+\frac{m}{d}}^{d/4} \ge \frac1m\p{1+\frac{m}{4}}\cdot\p{1+\frac{m}{d}}^{d/4}\\
					&\ge \frac14\p{1+\frac{m}{d}}^{d/4} = \frac14\exp\p{\frac{d}{4}\ln\p{1+\frac md}} \ge \frac14\exp\p{\frac{d}{8}\ln\p{\frac md +2}}.
				\end{align*}
				In the above; in the first line, the second inequality follows from the inequality $\binom{n}{k}\ge(n/k)^k$ which holds for all natural $n\ge k$ and the second inequality follows from the fact that $x\mapsto(1+a/x)^x$ is increasing in $x\ge0$ for all $a>0$ and by assuming $d\ge4$ which entails $d-2\ge d/2$; in the second line the inequality follows from Bernoulli's inequality; in the third line, the last inequality follows from the fact that $0.5\ln(2+x)/\ln(1+x)\le1$ for all $x\ge1$ and since we assume $m>d$. Now assuming that $m\le d$, we have
				\[
					\frac1mN_{d,m} \ge \frac1dN_{d,m} \ge \frac1d\binom{m+d-2}{m},
				\]
				from which by symmetry and using the same reasoning as in the previous case we get
				\[
					\frac1mN_{d,m} \ge \frac14\exp\p{\frac{m}{8}\ln\p{\frac dm +2}}.
				\]
				Combining these two cases together we obtain a lower bound of
				\[
					\frac1mN_{d,m} \ge \frac14\exp\p{\frac18\min\set{m\ln\p{\frac dm +2}, d\ln\p{\frac md +2}}},
				\]
				which when plugged back in \eqref{eq:width_norm_bound}, along with \lemref{lem:F_d_m_properties}, implies the existence of some $i\in[2m+2]$ such that
				\[
					\max\set{r_i,\norm{\btheta_i}_2} \ge c_1\exp\p{c_2\min\set{m\ln\p{\frac dm +2}, d\ln\p{\frac md +2}}},
				\]
				for some $c_1,c_2>0$, concluding the proof of the theorem.
			\end{proof}
			
		\section{Proofs from \subsecref{subsec:random_features}}\label{app:random_features_proof}
		
		To prove \thmref{thm:random_features}, we will need the following auxiliary lemmas and propositions. The following technical lemma provides useful bounds on the special function Owen's T which we encounter when computing the expectation of truncated random features (see \citet{owen1956tables} for more information about this function).
		\begin{lemma}\label{lem:owen_T_lower_bound}
			Let
			\begin{equation}\label{eq:owens_T}
				T(h,a) \coloneqq \frac{1}{2\pi}\int_0^a\frac{\exp\p{-\frac{1}{2}h^2(1+t^2)}}{1+t^2}dt
			\end{equation}
			denote Owen's T function. Then we have for any $(h,a)\in[0,\infty)^2$ that
			\[
				\frac{1}{2\pi}\exp\p{-\frac12h^2\p{1+a^2}}\arctan(a) \le T(h,a) \le \frac{1}{2\pi}\exp\p{-\frac12h^2}\arctan(a).
			\]
		\end{lemma}
		
		\begin{proof}
			Starting with the lower bound, since $\exp\p{-\frac12h^2\p{1+t^2}}$ is minimized at $t=a$ in the domain of integration $[0,a]$ and since the integrand is positive, we can lower bound $T(h,a)$ by
			\begin{align*}
				\frac{1}{2\pi}\int_0^a\frac{\exp\p{-\frac{1}{2}h^2(1+a^2)}}{1+t^2}dt &= \frac{\exp\p{-\frac{1}{2}h^2(1+a^2)}}{2\pi}\int_0^a\frac{1}{1+t^2}dt\\ &= \frac{\exp\p{-\frac{1}{2}h^2(1+a^2)}}{2\pi}\arctan(a),
			\end{align*}
			where the upper bound follows analogously from the exponent being maximized at $0$ over the integration domain.
		\end{proof}
		
		The following lemma establishes some useful properties of the distribution of the random features that are produced by i.i.d.\ normal random variables.
		\begin{lemma}\label{lem:dist_properties}
			Suppose that $U_j\sim\Ncal(0,\sigma^2\cdot I_d)$ and $B_j\sim\Ncal(0,\sigma^2)$ for all $j\in[r]$. Then the pre-activation output of a neuron in the hidden layer conditioned on the input's norm $\norm{X}$ attaining the value $x$, given by the random variable
			$\p{\inner{U_j,X}+B_j\mid \norm{X}=x}$,\footnote{To avoid ambiguity stemming from the Borel-Kolmogorov paradox, we formally define this random variable using the limit $\lim\limits_{\epsilon\to0}\p{\inner{U_j,X}+B_j\mid \norm{X}\in(x-\epsilon,x+\epsilon)}$. Note that this is a valid definition since by Assumption~\ref{asm:d_dist_assumption}(2) $\norm{X}$ has a strictly positive density.} follows a normal distribution with variance $\sigma^2(1+x^2)$. Moreover, under the conditioning on the input's norm, all such outputs for all the neurons are mutually independent random variables.
		\end{lemma}
		
		\begin{proof}
			Fix some $\bx=(x_1,\ldots,x_d)$ such that $\norm{\bx}=x$, then we have that
			\[
				\inner{U_j,\bx}+B_j = \sum_{k=1}^{d}U_{j,k}x_k+B_j
			\]
			is normally distributed with variance $\sigma^2(1+x^2)$ since the sum of independent normal random variables is normally distributed with the sum of the expectations and variances as its parameters. Moreover, Since the coordinate in each weight $U_{j,k}$ and $B_j$ are mutually independent by our initialization assumption it is readily seen that $\E\pcc{\prod_{j\in A}\p{\inner{U_j,\bx}+B_j}}=0$ for any subset of indices $A\subseteq\{1,\ldots,r\}$. Since uncorrelatedness implies independence for multivariate normal variables, the lemma follows.
		\end{proof}

		The following proposition computes the expectation of the random features produced under Assumption~\ref{asm:init} when we truncate the bias terms.
		
		\begin{proposition}\label{prop:truncation_expectation}
			Suppose that the random variables $U\in\reals^d$ and $B\in\reals$ are distributed according to Assumption~\ref{asm:init}. Then if $B$ is truncated from below at $0$ we have for all $\bx\in\reals^d$ that
			\[
				\E_{U,B}\pcc{\erf\p{\inner{U,\bx}+B}} = \frac{2}{\pi}\arctan\p{\frac{1}{\sqrt{2+\norm{\bx}^2}}}.
			\]
			Further, if $B$ is instead truncated from below at $1/\sqrt{2}$ we have for all $\bx\in\reals^d$ that
			\[
				\E_{U,B}\pcc{\erf\p{\inner{U,\bx}+B}} = \frac{4}{1-\erf(1)} \p{T\p{\sqrt{2},\frac{1}{\sqrt{2+\norm{\bx}^2}}}}.
			\]
		\end{proposition}
		
		\begin{proof}
			We will use of the following identities which appear in \citet{owen1980table}:
			\begin{align}
				& \int_{-\infty}^{\infty} \Phi(c+d z)\phi(z)dz = \Phi\p{\frac{c}{\sqrt{1+d^2}}}, \label{eq:Phi_phi}\\
				& \int \Phi(cz)\phi(z)dz= -T(z,c) + \frac{1}{2}\Phi(z),	\label{eq:owen_T_integral}\\
				& T(0,a) = \frac{1}{2\pi}\arctan(a)\label{eq:owent_zero},
			\end{align}
			where $T(\cdot,\cdot)$ is Owen's T function defined in \eqref{eq:owens_T}, $\phi(z)=(2\pi)^{-0.5}\exp(-0.5z^2)$ is the density function of a standard normal random variable and $\Phi(z)$ is its cumulative distribution function.
			
			Let $U,B$ as in the proposition statement. We have by \lemref{lem:dist_properties} that $\inner{U,\bx}=\norm{\bx}^2\cdot Y$ where $Y\sim\Ncal(0,\sigma^2)$. We compute the expectation iteratively, starting with the expectation over $Y$ by using the law of total expectation as follows
			\begin{equation}\label{eq:lote}
				\E_{Y,B}\pcc{\erf\p{\norm{\bx}^2Y+B}} = \E_{B}\pcc{\E_{Y|B}\pcc{\erf\p{\norm{\bx}^2Y+b|B=b}}}.
			\end{equation}
			Applying the law of unconscious statistician to $Y$, the inner expectation above equals
			\[
				\int_{-\infty}^{\infty}\erf\p{\norm{\bx}^2y+b}\frac{1}{\sigma}\phi\p{\frac{y}{\sigma}}dy,
			\]
			which by a simple change of variables, the identity $\erf(z)=2\Phi(\sqrt{2}z)-1$ and \eqref{eq:Phi_phi}, can be simplified to
			\begin{align*}
				\int_{-\infty}^{\infty}\erf\p{\norm{\bx}\sigma y+b}\phi\p{y}dy &= 2\int_{-\infty}^{\infty}\Phi\p{\sqrt{2}\norm{\bx}\sigma y+\sqrt{2}b}\phi\p{y}dy - \int_{-\infty}^{\infty}\phi(y)dy \\
				&= 2\int_{-\infty}^{\infty}\Phi\p{\sqrt{2}\norm{\bx}\sigma y+\sqrt{2}b}\phi\p{y}dy - 1\\
				&=2\Phi\p{\frac{\sqrt{2}b}{\sqrt{1+2\norm{\bx}^2\sigma^2}}} - 1.
			\end{align*}
			To compute the expectation over the bias term, assume we truncate its values (by setting the corresponding weight in the output neuron to zero) to the interval $[\alpha,\beta]$ for some $\beta>\alpha$. Then plugging the above in \eqref{eq:lote}, we have again from the law of unconscious statistician, the density of a truncated normal and a simple change of variables that this equals
			\begin{align*}
				&\E_{B}\pcc{2\Phi\p{\frac{\sqrt{2}B}{\sqrt{1+2\norm{\bx}^2\sigma^2}}} - 1}\\
				=& 2\int_{\alpha}^{\beta}\Phi\p{\frac{\sqrt{2}b}{\sqrt{1+2\norm{\bx}^2\sigma^2}}}\frac{1}{\sigma\cdot(\Phi(\beta/\sigma)-\Phi(\alpha/\sigma))}\phi\p{\frac{b}{\sigma}}db - 1\\
				=& \frac{2}{\Phi(\beta/\sigma)-\Phi(\alpha/\sigma)}\int_{\alpha/\sigma}^{\beta/\sigma}\Phi\p{\frac{\sqrt{2}\sigma b}{\sqrt{1+2\norm{\bx}^2\sigma^2}}}\phi\p{b}db - 1.
			\end{align*}
			From \eqref{eq:owen_T_integral}, the above equals
			\begin{align*}
				&\frac{2}{\Phi(\beta/\sigma)-\Phi(\alpha/\sigma)}\pcc{-T\p{z,\frac{\sqrt{2}\sigma}{\sqrt{1+2\norm{\bx}^2\sigma^2}}} + \frac{1}{2}\Phi(z)}_{\alpha/\sigma}^{\beta/\sigma} - 1\\
				=& \frac{2}{\Phi(\beta/\sigma)-\Phi(\alpha/\sigma)} \p{T\p{\frac{\alpha}{\sigma},\frac{\sqrt{2}\sigma}{\sqrt{1+2\norm{\bx}^2\sigma^2}}} - T\p{\frac{\beta}{\sigma},\frac{\sqrt{2}\sigma}{\sqrt{1+2\norm{\bx}^2\sigma^2}}}}\\
				=& \frac{4}{\erf(\beta/\sqrt{2}\sigma)-\erf(\alpha/\sqrt{2}\sigma)} \p{T\p{\frac{\alpha}{\sigma},\frac{\sqrt{2}\sigma}{\sqrt{1+2\norm{\bx}^2\sigma^2}}} - T\p{\frac{\beta}{\sigma},\frac{\sqrt{2}\sigma}{\sqrt{1+2\norm{\bx}^2\sigma^2}}}}.
			\end{align*}
			Plugging $\sigma=\frac12$ and taking the limit $\beta\to\infty$, the above reduces to
			\[
				\frac{4}{1-\erf(\sqrt{2}\alpha)} \p{T\p{2\alpha,\frac{1}{\sqrt{2+\norm{\bx}^2}}}},
			\]
			which for the special cases of $\alpha=0$ and $\alpha=1/\sqrt{2}$, by virtue of \eqref{eq:owent_zero}, equals
			\[
				\frac{2}{\pi}\arctan\p{\frac{1}{\sqrt{2+\norm{\bx}^2}}},
			\]
			and
			\[
				\frac{4}{1-\erf(1)} \p{T\p{\sqrt{2},\frac{1}{\sqrt{2+\norm{\bx}^2}}}}.
			\]
			
		\end{proof}
		
		The following proposition establishes some of the crucial properties of the functions we use to approximate ball indicators with random features.
		\begin{proposition}\label{prop:f_xi_properties}
			Let $\xi\in I$ where $I=[0.45,0.472]$ and define for $z\ge0$
			\[
				f_{\xi}(z) \coloneqq \frac{2}{\pi}\arctan\p{\frac{1}{\sqrt{2+z^2}}} - \xi\cdot\frac{4}{1-\erf(1)} \p{T\p{\sqrt{2},\frac{1}{\sqrt{2+z^2}}}}.
			\]
			Then $\set{f_{\xi}}_{\xi\in I}$ satisfy the following properties:
			\begin{enumerate}
				\item\label{item:1}
				For all $\xi\in I$, $\lim\limits_{z\to\infty}f_{\xi}(z)=0$.
				\item\label{item:2}
				For all $\lambda\in[1,2]$ there exists $\xi\in I$ such that $f_{\xi}(\lambda)=0$.
				\item\label{item:3}
				For all $\xi\in I$, $f_{\xi}(z)$ has a global minimum at
				\[
					z^*=\sqrt{\frac{1}{\ln\p{\frac{\xi}{1-\erf(1)}}-1}-2},
				\]
				where $z^*\in[2.8,4.2]$ for any $\xi\in I$. Moreover, $f_{\xi}(z)$ is decreasing for all $z\in(0,z^*)$ and increasing for all $z\in(z^*,\infty)$.
				\item\label{item:4}
				%For all $\xi\in I$ and all $z\ge0$, $f_{\xi}(z)\in\p{-\frac{1}{20},\frac{1}{20}}$.
				For all $\xi\in I$ and all $z\ge0$
				\[
					\frac{2}{\pi}\arctan\p{\frac{1}{\sqrt{2+z^2}}}\in[0,0.5]
				\]
				and
				\[
					\xi\cdot\frac{4}{1-\erf(1)} \p{T\p{\sqrt{2},\frac{1}{\sqrt{2+z^2}}}} \in [0,0.5].
				\]
				\item\label{item:5}
				For all $\xi\in I$ and all $z\ge10$,
				\[
					f_{\xi}(z)\le-\frac{1}{50z}.
				\]
				\item\label{item:6}
				For all $\xi\in I$ and all $z\in[0.9,2.1]$,
				\[
					\abs{\frac{\partial}{\partial z}f_{\xi}(z)} \ge \frac{1}{600}.
				\]
			\end{enumerate}
		\end{proposition}
		
		\begin{proof}
			~
			\begin{enumerate}
				\item
				Since $\frac{1}{\sqrt{2+z^2}}\to0$ as $z\to\infty$, the limit follows from continuity and the definition of $T(\cdot,\cdot)$.
				\item
				Given $\lambda\in[1,2]$, since $T\p{\sqrt{2},\frac{1}{\sqrt{2+\lambda^2}}}>0$, we define
				\[
					\xi=g(\lambda)\coloneqq \frac{1-\erf(1)}{2\pi}\cdot\frac{\arctan\p{\frac{1}{\sqrt{2+\lambda^2}}}}{T\p{\sqrt{2},\frac{1}{\sqrt{2+\lambda^2}}}}.
				\]
				Clearly, $f_{\xi}(\lambda)=0$ by its definition. To bound the set of values $\xi$ takes for all $\lambda\in[1,2]$, we bound the quotient
				\[
					\frac{\arctan\p{\frac{1}{\sqrt{2+\lambda^2}}}}{T\p{\sqrt{2},\frac{1}{\sqrt{2+\lambda^2}}}}
				\]
				from below and above. It will suffice to show the monotonicity of this quotient in the domain $[1,2]$ and compute its values at the boundary. Since $(2+\lambda^2)^{-0.5}$ is monotone, it will suffice to show monotonicity after the change of variables $z\mapsto (2+\lambda^2)^{-0.5}$, which yields
				\[
					\frac{\arctan\p{z}}{T\p{\sqrt{2},z}}.
				\]
				We now need only show that the derivative of the above does not change its sign. Using the fundamental theorem of calculus, consider the numerator of the derivative of the quotient which is given by
				\[
					\frac{1}{1+z^2}T\p{\sqrt{2},z}-\frac{1}{1+z^2}\cdot\frac{1}{2\pi}\arctan(z)\exp(-(1+z^2)),
				\]
				which is non-negative by \lemref{lem:owen_T_lower_bound}. We now conclude by verifying that $g(1),g(2)\in I$ using a symbolic computation package.
				
				\item
				Using the chain rule and the fundamental theorem of calculus, we have
				\begin{align}
					\frac{\partial}{\partial z}f_{\xi}(z) &= -\frac{2}{\pi}\cdot\frac{z}{(3+z^2)\sqrt{2+z^2}} + \xi\cdot\frac{4}{2\pi(1-\erf(1))}\cdot \frac{z\exp\p{-\p{1+\frac{1}{2+z^2}}}}{(2+z^2)^{1.5}\p{1+\frac{1}{2+z^2}}}\nonumber\\
					&= \frac{2z}{\pi(1-\erf(1))\p{3+z^2}\sqrt{2+z^2}}\p{\xi\exp\p{-\p{1+\frac{1}{2+z^2}}} - \p{1-\erf(1)}}.\label{eq:f_xi_derivative}
				\end{align}
				To find the critical point $z^*$, equate the above to zero and solve for $z$ to get
				\[
					z^*=\sqrt{\frac{1}{\ln\p{\frac{\xi}{1-\erf(1)}}-1}-2},
				\]
				which for any $\xi\in I$ is contained inside the interval $[2.8,4.2]$. Lastly, since
				\[
					\exp\p{-\p{1+\frac{1}{2+z^2}}}
				\]
				is increasing in $z$ and since the fraction on the left-hand side of \eqref{eq:f_xi_derivative} is positive for positive $z$, we have from that the derivative is negative in $(0,z^*)$ and positive in $(z^*,\infty)$.
				
				\item
				We have
				\[
					0\le \frac{2}{\pi}\arctan\p{\frac{1}{\sqrt{2+z^2}}} \le \frac{2}{\pi}\arctan\p{\frac{1}{\sqrt{2}}} \le 0.5,
				\]
				since clearly the above is positive for all $z\ge0$.
				Moreover, by \lemref{lem:owen_T_lower_bound}
				\begin{align*}
					0&\le\xi\cdot\frac{4}{1-\erf(1)} \p{T\p{\sqrt{2},\frac{1}{\sqrt{2+z^2}}}}\\  &\le \xi\cdot\frac{4}{1-\erf(1)}\cdot\frac{\exp\p{-1}}{2\pi}\arctan\p{\frac{1}{\sqrt{2+z^2}}}\\ 
					&\le \xi\cdot\frac{4}{1-\erf(1)}\cdot\frac{\exp\p{-1}}{2\pi}\arctan\p{\frac{1}{\sqrt{2}}} \le 0.5,
				\end{align*}
				where the first inequality is immediate from the definition of $T(\cdot,\cdot)$ and in the last inequality we used $\xi\le0.472$.
				\item
				Using \lemref{lem:owen_T_lower_bound} to lower bound $T\p{\sqrt{2},\frac{1}{\sqrt{2+z^2}}}$, we have
				\begin{align*}
					f_{\xi}(z) &\le \frac{2}{\pi}\arctan\p{\frac{1}{\sqrt{2+z^2}}}\\
					&\hskip 1cm
					- \xi\cdot\frac{4}{1-\erf(1)} \p{\frac{1}{2\pi}\exp\p{-\p{1+\frac{1}{2+z^2}}}\arctan\p{\frac{1}{\sqrt{2+z^2}}}}\\
					&= \frac{2}{\pi}\arctan\p{\frac{1}{\sqrt{2+z^2}}}\p{1 - \frac{\xi}{1-\erf(1)}\exp\p{-\p{1+\frac{1}{2+z^2}}}}\\
					&\le \frac{2}{\pi}\arctan\p{\frac{1}{\sqrt{2+z^2}}}\p{1 -  \frac{0.45}{1-\erf(1)}\exp\p{-1.01}}\\
					&\le -\frac{0.08}{\pi}\arctan\p{\frac{1}{\sqrt{2+z^2}}} \le -\frac{0.064}{\pi\sqrt{2+z^2}} \le -\frac{0.02}{z},
				\end{align*}
				where the second inequality follows from $\xi\ge0.45$ and the assumption $z\ge10$ which implies $-\p{1+1/(2+z^2)} \ge -1.01$, the third inequality follows from $1 -  \frac{0.45}{1-\erf(1)}\exp\p{-1.01} \le -0.04$, the fourth from $\arctan(z)\ge 0.8z$ for all $z\in[0,2^{-0.5}]$, and the last inequality follows from $2\le 0.02z^2$ for any $z\ge10$ and from $0.064/(\sqrt{1.02}\pi) \ge 0.02$.
				
				\item
				Plugging $0.9$ or $2.1$ into \eqref{eq:f_xi_derivative}, depending on whether a certain expression in the derivative is monotonically increasing or decreasing, we can lower bound the derivative in absolute value by
				\[
				\abs{\frac{2\cdot0.9}{\pi(1-\erf(1))(3+2.1^2)\sqrt{2+2.1^2}}\p{0.472\exp\p{-\p{1+\frac{1}{2+2.1^2}}} - (1-\erf(1))}}
				\]
				where we also used the fact that $\xi\le0.472$. The lower bound then follows by verifying that the above is at least $0.00169$. 
			\end{enumerate}
		\end{proof}
		With the above propositions, we can now prove \thmref{thm:random_features}.
	
		\begin{proof}[Proof of \thmref{thm:random_features}]
			We first sample $\bw_0=(w_{1,0},\ldots,w_{r,0})$ according to Assumption~\ref{asm:init}, and then we proceed to define $\bv$ in the following manner: Split the $r$ neurons in the hidden layer into two equally sized groups, and set $v_j=r^{-0.75}$ for all the neurons in the first group and $v_j=-\xi_{\lambda}\cdot r^{-0.75}$ in the second group where $\xi_{\lambda}\in[0.45,0.472]$ is a constant that depends on $\lambda$ and is to be specified later. Next, we describe which undesired random features get truncated (by setting their corresponding $v_j$ to $0$). For each hidden neuron in the first group, we truncate bias terms with value below $0$, and for the second group we truncate bias terms below $1/\sqrt{2}$. Furthermore, if the sign of any weight $w_j$ in the output neuron's initialization disagrees with the sign of the corresponding value we had set for $v_i$, then it is also truncated.
			
			Before we proceed to prove the theorem, we will introduce some notation to be used throughout the remainder of the proof. We let $A_1,A_2$ denote the indices of the neurons that do not get truncated in the first and second groups, respectively. For some $j\in[r]$, we let $W_j$ denote the random variable which is the $j$-th coordinate of $\bw_0$. Fix some $i\in[n]$ and let $Y_{j,k}=\erf\p{\inner{U_j,\bx_i}+B_{j,a}}$ denote the random variable which is the output of the $j$-th neuron in the $k$-th group on the $i$-th data instance, where $B_{j,a}$ is a normal random variable with zero mean and variance $0.25$ which is truncated from below at $a$ (i.e.\ $a=0$ for $k=1$ and $a=2^{-0.5}$ for $k=2$). Let $\bv=(v_1,\ldots,v_r)$ denote the coordinates of $\bv$.
			
			Next, we restate the inequalities assumed in the statement of the theorem to be used throughout the proof. Observe that after some calculations, the assumed bound on $r$ implies the following inequalities (if we also assume $n\ge4$)
			\begin{align}
				&r^{-0.25}\le\frac{1}{12000},  \label{eq:lipschitz_deviation}\\
				&10\sqrt{\frac{\ln\p{\frac{8n}{\delta}}}{r}} \le 10\sqrt{\frac{\ln\p{\frac{8n}{\delta}}}{\sqrt{r}}} \le \frac{1}{500} < 1,\label{eq:confidence_bound}\\
				&-\frac{1}{50c_1}\le-10\sqrt{\frac{\ln\p{\frac{8n}{\delta}}}{\sqrt{r}}},\label{eq:c1_bound}\\
				&\sqrt{\frac{\ln\p{\frac{32}{\delta}}}{r}} \le 0.02, \label{eq:non_truncated_deviation}\\
				&1.2^{-r} \le \frac{\delta}{16}. \label{eq:chi_squared_bound_assumption}
			\end{align}
			
			We begin the proof with showing the first item in the theorem statement, where the upper bound follows simply from
			\[
				\norm{\bv}^2 = \sum_{i=1}^{r}v_i^2 \le \sum_{i=1}^{r}r^{-1.5}=\frac{1}{\sqrt{r}},
			\]
			where the inequality is by the definition of $\xi$ which satisfies $\xi\in[0.45,0.472]$ (and appears at a later stage of the proof) and the fact that truncating any coordinate $v_j$ only decreases the norm of $\bv$. Lower bounding $\norm{\bv}$ requires that we first show that the number of neurons that do not get truncated is large enough. To this end, we first observe that the probability of neurons in the first group to not get truncated is exactly $\frac14$. This is because we obtain a bias term realization that is negative with probability $0.5$, the sign of the corresponding coordinate in $\bw_0$ is positive with probability $0.5$ by Assumption~\ref{asm:init}, and since the two events are independent. Similarly, we have that the probability of a neuron in the second group to not get truncated is $0.25-0.25\erf(1)\ge0.03$. This guarantees that with high probability, the number of untruncated neurons in each group is at least a constant fraction of $r$: We have from Hoeffding's inequality that for each group that
			\[
				\pr\pcc{\abs{\frac{2}{r}\sum_{j=1}^{r/2}\one{v_j\neq 0} - \E\pcc{\one{v_j \neq 0}}} > \sqrt{\frac{\ln(32/\delta)}{r}}} \le \frac{\delta}{16},
			\]
			where the expectations for the first and second groups are at least $0.25,0.03$, respectively, since the expectation of an indicator is the probability of the event. By \eqref{eq:non_truncated_deviation} this implies
			\[
				\pr\pcc{\frac{2}{r}\sum_{i=1}^{r/2}\one{v_i\neq0} \ge 0.01 } \ge 1-\frac{\delta}{16}.
			\]
			Taking a union bound over the two truncation groups, we have that 
			\begin{equation}\label{eq:non_truncated_percentage}
				\pr\pcc{\min\set{\abs{A_1},\abs{A_2}}\ge\frac{r}{200}} \ge 1-\frac{\delta}{8}.
			\end{equation}
			
			Assuming the above holds, we can derive the lower bound on $\norm{\bv}$. Focusing on the neurons in the first truncation group, we have by Assumption~\ref{asm:init} that for any $j\in A_1$, 
			\[
				\p{W_j\mid v_j\neq0} \sim \p{W_j\mid W_j>0} \sim \abs{W_j}.
			\]
			That is, given the knowledge that a weight $\bv_j$ did not get truncated, its corresponding weight in the hidden layer $W_j$ follows a half-normal distribution. This is true since the bias term $B_j$ and $W_j$ are independent by Assumption~\ref{asm:init}, and since $v_j\neq0$ implies that the realization of $W_j$ was positive which by symmetry implies it is half-normally distributed. Since the PDF of a half-normal random variable is $\erf\p{\frac{x}{\sigma\sqrt{2}}}$, we have that $\pr\pcc{\abs{W_j}\ge\frac{1}{r}}>0.3$. We can now use Hoeffding's inequality on all the neurons in $A_1$ (note that they are i.i.d.\ even after we are given the knowledge that they were not truncated) and obtain
			\[
				\pr\pcc{\abs{\frac{1}{\abs{A_1}}\sum_{j\in A_1}\one{\abs{W_j}\ge\frac{1}{r}} - \E\pcc{\one{\abs{W_j}\ge\frac{1}{r}}}} > 10\sqrt{\frac{\ln(32/\delta)}{r}}} \le \frac{\delta}{16},
			\]
			which by Eqs.~(\ref{eq:non_truncated_deviation},\ref{eq:non_truncated_percentage}) implies
			\[
					\pr\pcc{\abs{\frac{1}{\abs{A_1}}\sum_{j\in A_1}\one{\abs{W_j}\ge\frac{1}{r}}} \le 0.1} \le \frac{\delta}{16},
			\]
			or equivalently using \eqref{eq:non_truncated_percentage} again
			\begin{equation}\label{eq:w_magnitude_bound}
				\pr\pcc{\abs{\sum_{j\in A_1}\one{\abs{W_j}\ge\frac{1}{r}}} > \frac{\abs{A_1}}{10}\ge \frac{r}{2000}} \ge 1-\frac{\delta}{16}.
			\end{equation}
			With the above at hand we can lower bound $\inner{\bw_0,\bv}$ in a straightforward manner; by Assumption~\ref{asm:init} and the construction of $\bv$, we have that $v_j=r^{-0.75}$ and $w_j\ge\frac1r$ for at least $r/2000$ neurons, therefore
			\begin{equation}\label{eq:dot_product_bound}
				\inner{\bw_0,\bv}\ge \frac{r}{2000}\cdot\frac{1}{r^{0.75}}\cdot\frac1r = \frac{1}{2000r^{0.75}}.
			\end{equation}
			Upper bounding $\norm{\bw_0}^2$, we use a standard bound on Chi-squared random variables to obtain
			\[
				\pr\pcc{\norm{\bw_0}^2\ge \frac2r} \le \p{2\exp(-1)}^{r/2} \le 1.2^{-r} \le \frac{\delta}{16},
			\]
			where the last inequality follows from \eqref{eq:chi_squared_bound_assumption}. Equivalently, the above can be stated as
			\begin{equation}\label{eq:chi_squared_bound}
				\pr\pcc{\norm{\bw_0}^2 < \frac2r}\ge 1 - \frac{\delta}{16},
			\end{equation}
			which along with \eqref{eq:dot_product_bound} yields
			\[
				\norm{\bw_0}^2 < \frac{2}{r} \le \frac{1}{10^3r^{0.75}} \le 2\inner{\bw_0,\bv},
			\]
			where the second inequality follows from \eqref{eq:lipschitz_deviation}. Adding $\norm{\bv}^2$ to both sides of the above inequality and rearranging, the lower bound on $\norm{\bv}^2$ follows.
			
			Moving to the second item in the theorem statement, we now wish to define the intervals $I_1,\ldots,I_4$. To this end, we first define the functions
			\[
				f_{\xi_{\lambda}}(z) \coloneqq \frac{2}{\pi}\arctan\p{\frac{1}{\sqrt{2+z^2}}} - \xi_{\lambda}\cdot\frac{4}{1-\erf(1)} \p{T\p{\sqrt{2},\frac{1}{\sqrt{2+z^2}}}},
			\]
			for $\xi_{\lambda}\in[0.45,0.472]$ that depend on $\lambda$ so that $f_{\xi_{\lambda}}(\lambda)=0$ for all $\lambda\in[1,2]$ (see Item \ref{item:2} in \propref{prop:f_xi_properties}). By Item \ref{item:3} in \propref{prop:f_xi_properties}, we can define two inverse functions which we denote by $f^{-1}_{\xi_{\lambda},1}(z),f^{-1}_{\xi_{\lambda},2}(z)$ on the intervals $(0,z^*)$ and $(z^*,\infty)$, respectively, where it is guaranteed that $z^*\in[2.8,4.2]$.
			We can now define the intervals $I_1,\ldots,I_4$. Starting with $I_1,I_2$, we let
			\[
				I_1\coloneqq\pcc{0, f^{-1}_{\xi_{\lambda},1}\p{2r^{-0.25}}}, \hskip 0.5cm I_2\coloneqq\p{ f^{-1}_{\xi_{\lambda},1}\p{2r^{-0.25}},  f^{-1}_{\xi_{\lambda},1}\p{-r^{-0.25}}}.
			\]
			We remark that the above intervals are well-defined since by Items \ref{item:2} and \ref{item:6} in \propref{prop:f_xi_properties}, we have $f_{\xi_{\lambda}}(0.9)\ge\frac{1}{6000}$ and $f_{\xi_{\lambda}}(2.05)\le-\frac{1}{12000}$, which along with \eqref{eq:lipschitz_deviation} imply that
			\[
				0.9\le f^{-1}_{\xi_{\lambda},1}\p{2r^{-0.25}} < f^{-1}_{\xi_{\lambda},1}\p{-r^{-0.25}}  \le 2.05.
			\]
			To define $I_3$ and $I_4$, we observe that the above also guarantees that $z^*\notin I_1\cup I_2$, and from Item \ref{item:1} in \propref{prop:f_xi_properties} we have that the image of $f^{-1}_{\xi_{\lambda},2}(\cdot)$ is $(0,f^{-1}_{\xi_{\lambda},2}(z^*))$. We can thus define $I_3$ and $I_4$ as the intervals given by
			\[
				I_3\coloneqq \pcc{f^{-1}_{\xi_{\lambda},1}\p{-r^{-0.25}}, f^{-1}_{\xi_{\lambda},2}\p{-10\sqrt{\frac{\ln\p{\frac{8n}{\delta}}}{r}}}}, \hskip 0.5cm
				I_4\coloneqq \p{f^{-1}_{\xi_{\lambda},2}\p{-10\sqrt{\frac{\ln\p{\frac{8n}{\delta}}}{r}}}, \infty}.
			\]
			To see why the above intervals are well-defined, we use Item \ref{item:5} in \propref{prop:f_xi_properties} and \eqref{eq:confidence_bound} to deduce that
			\begin{equation}\label{eq:inequality_chain}
				f_{\xi_{\lambda}}(10) \le -\frac{1}{500} \le -10\sqrt{\frac{\ln\p{\frac{8n}{\delta}}}{\sqrt{r}}} \le -10\sqrt{\frac{\ln\p{\frac{8n}{\delta}}}{r}}.
			\end{equation}
			Since $f_{\xi_{\lambda,2}}^{-1}(\cdot)$ is increasing by its definition (and defined for $z=10$ since $10\ge z^*$), we can apply it to the above inequality to obtain
			\begin{equation}\label{eq:z_star_in_I3}
				10\le f_{\xi_{\lambda,2}}^{-1}\p{-10\sqrt{\frac{\ln\p{\frac{8n}{\delta}}}{r}}},
			\end{equation}
			which verifies that $I_3,I_4$ are well-defined.

			We now turn to show that with high probability, the fraction of data instances with norm in $I_2\cup I_4$ decays to zero as $r$ grows. Recall that $I_2\subseteq[0.9,2.05]$, we can upper bound the length of $I_2$ using Item \ref{item:6} in \propref{prop:f_xi_properties} by $1800r^{-0.25}$. Using Assumption~\ref{asm:d_dist_assumption}, let $q$ denote the polynomial bounding the density $\gamma_d$ on $I_2$. Then
			\begin{equation}\label{eq:I2_upper_bound}
				\pr_{\bx\sim\Dcal_d}\pcc{\norm{\bx}\in I_2} \le 1800q(d)r^{-0.25}.
			\end{equation}
			To bound $\pr_{\bx\sim\Dcal_d}\pcc{\norm{\bx}\in I_4}$, apply $f_{\xi_{\lambda,2}}^{-1}(\cdot)$ to the inequality in \eqref{eq:inequality_chain} (ignoring the last term) to get
			\begin{equation}\label{eq:10_lower_bound}
				f_{\xi_{\lambda,2}}^{-1}\p{-10\sqrt{\frac{\ln\p{\frac{8n}{\delta}}}{\sqrt{r}}}} \ge 10.
			\end{equation}
			We now argue that the constant $c_1$ from Assumption~\ref{asm:d_dist_assumption} satisfies
			\begin{equation}\label{eq:c1_upper_bound}
				c_1 \le f_{\xi_{\lambda,2}}^{-1}\p{-10\sqrt{\frac{\ln\p{\frac{8n}{\delta}}}{\sqrt{r}}}}.
			\end{equation}
			If $c_1\le10$, the above is immediate from \eqref{eq:10_lower_bound}. Otherwise, we can use \eqref{eq:c1_bound} and Item \ref{item:5} in \propref{prop:f_xi_properties} to deduce that
			\[
				f(c_1)\le -\frac{1}{50c_1} \le -10\sqrt{\frac{\ln\p{\frac{8n}{\delta}}}{\sqrt{r}}},
			\]
			where \eqref{eq:c1_upper_bound} follows by applying the increasing function $f_{\xi_{\lambda},2}^{-1}$ to both sides of the inequality. Using \eqref{eq:c1_upper_bound} and Assumption~\ref{asm:d_dist_assumption}, we have the following bound
			\begin{equation}\label{eq:I4_upper_bound}
				\pr_{\bx\sim\Dcal_d}\pcc{\norm{\bx}\in I_4} \le c_2\p{f_{\xi_{\lambda,2}}^{-1}\p{-10\sqrt{\frac{\ln\p{\frac{8n}{\delta}}}{\sqrt{r}}}}}^{-1}.
			\end{equation}
			To bound the expression above, note that \eqref{eq:confidence_bound} entails
			\[
				\frac{r^{0.25}}{500\sqrt{\ln\p{\frac{8n}{\delta}}}}\ge10,
			\]
			which allows us to use Item \ref{item:5} in \propref{prop:f_xi_properties} to deduce that
			\[
				f_{\xi_{\lambda}}\p{\frac{r^{0.25}}{500\sqrt{\ln\p{\frac{8n}{\delta}}}}} \le -10\sqrt{\frac{\ln\p{\frac{8n}{\delta}}}{\sqrt{r}}}.
			\]
			Applying the increasing function $f_{\xi_{\lambda},2}^{-1}$ to both sides of the inequality above and rearranging while noting that $f_{\xi_{\lambda,2}}^{-1}(z)>0$ for all $z$ in its domain which does not change the sign upon division, we have
			\[
				\frac{1}{f_{\xi_{\lambda,2}}^{-1}\p{-10\sqrt{\ln\p{\frac{8n}{\delta}}/\sqrt{r}}}} \le \frac{500\sqrt{\ln\p{\frac{8n}{\delta}}}}{r^{0.25}},
			\]
			which when plugged in \eqref{eq:I4_upper_bound} and combined with \eqref{eq:I2_upper_bound} using a union bound implies
			\[
				\pr_{\bx\sim\Dcal_d}\pcc{\norm{\bx}\in I_2\cup I_4} \le p \coloneqq \frac{1800q(d)+500c_2\sqrt{\ln\p{\frac{8n}{\delta}}}}{r^{0.25}}.
			\]
			Since the data instances are i.i.d., we can use Hoeffding's inequality to bound the probability of getting a significantly larger portion of the data in the intervals $I_2,I_4$ as follows
			\[
				\pr_{\bx_i\sim\Dcal_d^n}\pcc{\abs{\frac1n\sum_{i=1}^{n}\one{\norm{\bx_i}\in I_2\cup I_4} - \E\pcc{\one{\norm{\bx_i}\in I_2\cup I_4}}} > \sqrt{\frac{\ln(8/\delta)}{2n}}} \le \frac{\delta}{4},
			\]
			which by the assumption $\sqrt{\frac{\ln(8/\delta)}{2n}} \le p$ in the theorem statement implies
			\begin{equation}\label{eq:I2_I4_pr_bound}
				\pr_{\bx_i\sim\Dcal_d^n}\pcc{\sum_{i=1}^{n}\one{\norm{\bx_i}\in I_2\cup I_4} > 2pn} \le \frac{\delta}{4}.
			\end{equation}
			To show that with high probability, data instances with norm in $I_1\cup I_3$ are classified correctly and data instances in $I_2\cup I_4$ have a bounded misclassification margin, we would need to obtain sufficient concentration of the hidden neurons around their means.	We thus now turn to bound the uniform deviation over the data for each group of neurons from its expected value. Using Hoeffding's inequality on the random variables $Y_{i,k}$ for each $k$ separately (which is justified since given $\norm{\bx_j}$, the outputs of the neurons in the hidden layer are mutually independent by \lemref{lem:dist_properties}), we have from Item \ref{item:4} in \propref{prop:f_xi_properties},  that $\pr\pcc{Y_{i,1}\in[0,0.5]}=1$, and therefore
			\[
				\pr\pcc{\abs{\sum_{i\in A_1}r^{-0.75}Y_{i,1} - r^{-0.75}\E\pcc{Y_{i,1}}} >\sqrt{\frac{25\ln(8n/\delta)}{\sqrt{r}}} } \le \frac{\delta}{4n}.
			\]
			Likewise, from a similar bound on the average of $Y_{i,2}$ where $\pr\pcc{\xi_{\lambda} Y_{i,2}\in[0,0.5]}=1$ by Item \ref{item:4} in \propref{prop:f_xi_properties}, we have
			\[
				\pr\pcc{\abs{\sum_{i\in A_2}\xi_{\lambda}r^{-0.75}Y_{i,2} - \xi_{\lambda}r^{-0.75}\E\pcc{Y_{i,2}}} >\sqrt{\frac{25\ln(8n/\delta)}{\sqrt{r}}} } \le \frac{\delta}{4n}.
			\]
			Combining the above two bounds using a union bound and the triangle inequality and taking another union bound over all the $n$ instances in the data, it then follows from \propref{prop:truncation_expectation} that
			\begin{equation}\label{eq:ball_uniform_bound}
				\pr\pcc{\sup_{j\in[n]}\abs{\bv^{\top}\tilde{\bx}_j - r^{0.25}f_{\xi_{\lambda}}(\norm{\bx_j})} \le 10\sqrt{\frac{\ln(8n/\delta)}{\sqrt{r}}}} \ge 1-\frac{\delta}{2}.
			\end{equation}
			It only remains to show that $\crelu{\bv^{\top}\tilde{\bx}_j}=\one{\norm{\bx_j}\le \lambda}=y_j$ for any $\bx_j$ such that $\norm{\bx_j}\in I_1\cup I_3$, and that $\abs{\bv^{\top}\tilde{\bx}_j}\le3$ for any $\bx_j$ such that $\norm{\bx_j}\in I_2\cup I_4$. To this end, we will consider each interval separately:
			\begin{itemize}
				\item
				Suppose that $\norm{\bx_j}\in I_1$, then since $f_{\xi_{\lambda}}$ is decreasing on $I_1$ we have that $r^{0.25}f_{\xi_{\lambda}}(\norm{\bx_j})\ge r^{0.25}f_{\xi_{\lambda}}\p{f_{\xi_{\lambda,1}}^{-1} \p{2r^{-0.25}}}=2$. Combining this with Eqs.~(\ref{eq:confidence_bound},\ref{eq:ball_uniform_bound}) we have
				\[
					\bv^{\top}\tilde{\bx}_j \ge r^{0.25}f_{\xi_{\lambda}}(\norm{\bx_j}) - 10\sqrt{\frac{\ln(8n/\delta)}{\sqrt{r}}} \ge 2 - 1 = 1.
				\] 
				We thus have
				\[
					\crelu{\bv^{\top}\tilde{\bx}_j} = 1 =  \one{\norm{\bx_j}\le \lambda} = y_j.
				\]
				
				\item
				Suppose that $\bx_j\in I_2$, then since $f_{\xi_{\lambda}}$ is decreasing on $I_2$ we have that
				\[
					r^{0.25}f_{\xi_{\lambda}}(\norm{\bx_j}) \in (-1,2),
				\]
				which together with Eqs.~(\ref{eq:confidence_bound},\ref{eq:ball_uniform_bound}) implies
				\[
					\bv^{\top}\tilde{\bx}_j \in (-2,3).
				\]
				\item
				Suppose that $\bx_j\in I_3$. Since $z^*\in I_3$ by \eqref{eq:z_star_in_I3}, we have that $f_{\xi_{\lambda}}$ is decreasing and then increasing on $I_3$ by Item \ref{item:3} in \propref{prop:f_xi_properties}. Therefore, $f_{\xi_{\lambda}}$ attains its maximum over $I_3$ at the boundary, and we have
				\begin{align*}
					f_{\xi_{\lambda}}(\norm{\bx_j}) &\le \max\set{f_{\xi_{\lambda}}\p{f^{-1}_{\xi_{\lambda},1}\p{-r^{-0.25}}}, f_{\xi_{\lambda}}\p{f^{-1}_{\xi_{\lambda},2}\p{-10\sqrt{\frac{\ln\p{\frac{8n}{\delta}}}{r}}}}}\\
					&= \max\set{-r^{-0.25}, -10\sqrt{\frac{\ln\p{\frac{8n}{\delta}}}{r}}}.
				\end{align*}
				From the above and Eqs.~(\ref{eq:confidence_bound},\ref{eq:ball_uniform_bound}), we have
				\[
					\bv^{\top}\tilde{\bx}_j \le r^{0.25}f_{\xi_{\lambda}}(\norm{\bx_j}) + 10\sqrt{\frac{\ln\p{\frac{8n}{\delta}}}{\sqrt{r}}} \le \max\set{-1, -10\sqrt{\frac{\ln\p{\frac{8n}{\delta}}}{\sqrt{r}}}} + 10\sqrt{\frac{\ln\p{\frac{8n}{\delta}}}{\sqrt{r}}} \le 0,
				\]
				implying that
				\[
					\crelu{\bv^{\top}\tilde{\bx}_j} = 0 =  \one{\norm{\bx_j}\le \lambda} = y_j.
				\]
				\item
				Suppose that $\bx_j\in I_4$, then since $f_{\xi_{\lambda}}$ is increasing on $I_4$ we have that
				\[
					r^{0.25}f_{\xi_{\lambda}}(\norm{\bx_j}) \in (-1,0),
				\]
				which together with Eqs.~(\ref{eq:confidence_bound},\ref{eq:ball_uniform_bound}) implies
				\[
					\bv^{\top}\tilde{\bx}_j \in (-2,1).
				\]
			\end{itemize}
			We conclude the proof of \thmref{thm:random_features} by applying a union bound to the events in Eqs.~(\ref{eq:non_truncated_percentage},\ref{eq:w_magnitude_bound},\ref{eq:chi_squared_bound},\ref{eq:I2_I4_pr_bound},\ref{eq:ball_uniform_bound}).
		\end{proof}

		%\subsubsection{Approximating Ball Indicators}
		
		%\subsubsection{Convergence to Expectation}
		
		\section{Proofs from \subsecref{subsec:well-behaved_random_features}}\label{app:well-behaved_random_features_proof}
		
		The proof of \thmref{thm:well-behaved_random_features} relies on the following lemmas. The first lemma below provides a technical inequality which allows us to lower bound the density of the distribution of our random features by a log-concave density.
		
		\begin{lemma}\label{lem:monotonicity}
			The function
			\[
				g(x,z)=\frac{\exp\p{\ierf(z)^2\p{1-\frac{2}{1+x^2}}}}{\sqrt{2+2x^2}}
			\]
			restricted to the domain $(x,z)\in\reals\times[-1/9,1/9]$ is non-increasing in $x$.
		\end{lemma}
		
		\begin{proof}
			Differentiating with respect to $x$ we obtain (by hand or by using a symbolic computation package)
			\[
			\frac{\partial}{\partial x}g(x,z) = \frac{x\exp\p{\ierf(z)^2\p{1-\frac{2}{1+x^2}}}\p{2\ierf(z)^2 - 0.5(1+x^2)}}{0.5^3(2+2x^2)^{2.5}}.
			\]
			The above is non-positive if
			\[
			2\ierf(z)^2 - 0.5(1+x^2) \le 0.
			\]
			It can be verified that $\sup_{z\in[-1/9,1/9]}\ierf(z)^2\le0.01$, and therefore
			\[
				2\ierf(z)^2 - 0.5(1+x^2) \le 0.02 - 0.5 < 0,
			\]
			which concludes the proof of the lemma.
		\end{proof}
		
		The following lemma leverages the analysis done in \citet{lovasz2007geometry} to assert that log-concave distributions have well-behaved two-dimensional marginal distributions in any directions.
		
		\begin{lemma}\label{lem:logconcave_density}
			Suppose that $\tilde{X}_1,\ldots,\tilde{X}_r$ are i.i.d.\ log-concave random variables with $\E[\tilde{X}_1]=0$ and $\E[\tilde{X}_1^2]=\rho^2$. Then the marginal distribution of the random vector $(\tilde{X}_1,\ldots,\tilde{X}_r)$ on the subspace spanned by any $\bw_1\ne \bw_2$ has a (two-dimensional) density $p(\cdot)$ satisfying
			\[
			\inf_{\bx:\norm{\bx}\le\frac{1}{9\rho}}p(\bx)\ge \rho^22^{-16}.
			\]
		\end{lemma}
		
		\begin{proof}
			Let $O$ denote the orthogonal transformation such that the change of variables $\tilde{\bx}\mapsto O\tilde{\bx}$ followed by a marginalization of the last $r-2$ coordinates results in the random two-dimensional vector having density $p$. We first claim that this density is log-concave; this holds since affine transformations and marginalizations preserve log-concavity \citep{saumard2014log}. Next, we have that the distribution of this vector is given by the random variables
			\begin{equation*}
				X_1 = \sum_{j=1}^{r} o_{1,j}\tilde{X}_j,~~~
				X_2 = \sum_{j=1}^{r} o_{2,j}\tilde{X}_j,
			\end{equation*}
			where $o_{i,j}$ is the $(i,j)$-th entry of $O$. Consider the density $\int_{\reals^2}p(\bx)d\bx=1$. We perform a change of variables $x_1\mapsto\frac{1}{\rho}y_1$ and $x_2\mapsto\frac{1}{\rho}y_2$ which entails $d\bx=\frac{1}{\rho^2}d\by$, and we have
			\[
			1=\int_{\reals^2}p(\bx)d\bx = \int_{\reals^2}\frac{1}{\rho^2}p\p{\frac{1}{\rho}\by}d\by.
			\]
			That is, $\frac{1}{\rho^2}p\p{\frac{1}{\rho}\bx}$ is the density function of the random vector $\p{\frac{1}{\rho}X_1,\frac{1}{\rho}X_2}$. We will show that this random vector is in isotropic position which means it has zero mean and identity covariance matrix:
			We clearly have for $i\in\{1,2\}$ that
			\[
			\E\pcc{\frac{1}{\rho}X_i} = \sum_{j=1}^{r} \frac{o_{i,j}}{\rho}\E\pcc{\tilde{X}_j} = 0,
			\]
			and to compute the covariance matrix of $\p{\frac{1}{\rho}X_1,\frac{1}{\rho}X_2}$, we begin with the diagonal entries to obtain
			\[
			\E\pcc{\frac{1}{\rho^2}X_i^2} =  \frac{1}{\rho^2}\p{\sum_{j=1}^{r}o_{i,j}^2\E\pcc{\tilde{X}_j^2} + \sum_{j_1\ne j_2}o_{i,j_1}o_{i,j_2}\E\pcc{\tilde{X}_{j_1}\tilde{X}_{j_2}}} = \frac{1}{\rho^2}\sum_{j=1}^{r}o_{i,j}^2\rho^2=1,
			\]
			where we used the facts that $\tilde{X}_{j}$ are i.i.d.\ and that the rows of $O$ are orthogonal. To compute the off-diagonal entries, we have
			\[
			\E\pcc{\frac{1}{\rho^2}X_1X_2} = \frac{1}{\rho^2}\p{\sum_{j=1}^{r}o_{1,j}o_{2,j}\E\pcc{\tilde{X}_j^2} + \sum_{j_1\ne j_2}o_{1,j_1}o_{2,j_2}\E\pcc{\tilde{X}_{j_1}\tilde{X}_{j_2}}} = \frac{1}{\rho^2}\sum_{j=1}^{r}o_{1,j}o_{2,j}\rho^2=0,
			\]
			where again we used the facts that $\tilde{X}_{j}$ are i.i.d.\ and that the rows of $O$ are orthogonal. We now have by \citet[Thm.~5.14]{lovasz2007geometry} that the density $\frac{1}{\rho^2}p\p{\frac{1}{\rho}\bx}$ satisfies
			\[
			\frac{1}{\rho^2}p\p{\frac{1}{\rho}\bv} \ge 2^{-18\norm{\bv}}\frac{1}{\rho^2}p(\mathbf{0}) \ge 2^{-18\norm{\bv}-14}
			\]
			for all $\bv$ such that $\norm{\bv}\le\frac19$. The above implies that for all $\bx$ with $\norm{\bx}\le\frac{1}{9\rho}$, we have
			\[
			p(\bx)\ge \rho^22^{-18\rho\norm{\bx}-14} \ge \rho^22^{-16},
			\]
			concluding the proof of the lemma.
		\end{proof}
		
		We are ready to prove \thmref{thm:well-behaved_random_features}.
		
		\begin{proof}[Proof of \thmref{thm:well-behaved_random_features}]
			We begin with evaluating the distribution of the output of an arbitrary neuron in the first hidden layer, given by $\erf(\inner{U_j,X}+B_j)$, where $X\sim\Dcal_d$, $B_j\sim \Ncal(0,0.25)$ and $U_j\sim \Ncal(0,0.25)$. We first observe that the outputs of two neurons are not independent in general. Indeed, if $\Dcal_d$ generated an instance with a very large norm, then $\erf(\inner{U_j,X}+B_j)$ will be very close to either $-1$ or $1$, depending on the sign of $\inner{U_j,X}$. Given the information that a neuron outputs a value close to $1$ in absolute value, it is far less likely for other neurons to attain values close to zero. Fortunately, by \lemref{lem:dist_properties} we have that
			\[
				\p{\inner{U_j,X}+B_j\mid \norm{X}=x}\sim \Ncal(0,0.25(1+x^2)).
			\]
			In words, $\inner{U_j,X}+B_j$ conditioned on $\norm{X}$ taking the value $x$ follows a normal distribution with zero mean and variance $0.25(1+x^2)$. This conditioning is useful, since \lemref{lem:dist_properties} also implies that under this conditioning, the output of the neurons are mutually independent random variables, which would prove integral in our analysis. We now compute the distribution of the random variable $\tilde{X}_j=\erf(\inner{U_j,X}+B_j)$, also conditioned on $\norm{X}=x$. We have for any $z\in(-1,1)$
			\begin{align}
				\pr\pcc{\tilde{X}_j\le z\mid \norm{X}=x} &= \pr\pcc{\inner{U_j,X}+B_j\le \ierf(z)\mid \norm{X}=x} \nonumber\\
				&= 0.5+0.5\erf\p{\frac{\ierf(z)}{0.5\sqrt{2+2x^2}}}.\label{eq:anti_derivative}
			\end{align}
			Differentiating with respect to $z$ using the fact that $\frac{\partial}{\partial z} \ierf(z) = \frac{\sqrt{\pi}}{2}\exp\p{[\ierf(z)]^2}$, we have that the density of $\tilde{X}_j$ given that $\norm{X}=x$ is given by
			\begin{equation}\label{eq:density}
				f(z|x) = \frac{\exp\p{\ierf(z)^2\p{1-\frac{2}{1+x^2}}}}{\sqrt{2+2x^2}}.
			\end{equation}
			Using the above, we can express the density of the random vector $(\tilde{X}_1,\ldots,\tilde{X}_r)$ as follows
			\begin{equation}\label{eq:norm_marg}
				\tilde{f}(\tilde{x}_1,\ldots,\tilde{x}_r) = \int_{0}^{\infty} \bar{f}(\tilde{x}_1,\ldots,\tilde{x}_r,x)dx = \int_{0}^{\infty}\prod_{j=1}^{r}f(\tilde{x}_j\mid x)\cdot \gamma_d(x)dx,
			\end{equation}
			where $\bar{f}$ is the joint density of $\tilde{X}_1,\ldots,\tilde{X}_r$ and $\norm{X}$, and we used the conditional mutual dependence of the $\tilde{X}_j$'s given that $\norm{X}=x$ which \lemref{lem:dist_properties} guarantees. We are now interested in lower bounding the above expression once it is marginalized to an arbitrary two-dimensional subspace. To this end, we may assume without loss of generality that $\bw_1,\bw_2$ are orthogonal and of unit length (e.g.\ by applying a Gram–Schmidt process which does not change their span). Letting $\tilde{\bx}=(\tilde{x}_1,\ldots,\tilde{x}_r)$, $\mathbf{o}_j\in\reals^r$ denote the $i$-th row of an orthogonal matrix $O\in\reals^{r\times r}$ and $A\subseteq\reals^{r-2}$ denote the transformed integration domain $[-1,1]^{r-2}$ after applying the transformation defined by the first $r-2$ rows of $O$. We need to lower bound the following expression
			\begin{align}
				&\inf_{\tilde{\bx}:\norm{\tilde{\bx}}\le\frac19}\inf_{O:O\text{ orthogonal}} \int_{A}   \tilde{f}(\inner{\mathbf{o}_1,\tilde{\bx}},\ldots,\inner{\mathbf{o}_r,\tilde{\bx}})d\tilde{x}_1\ldots d\tilde{x}_{r-2}\nonumber\\
				=& \inf_{\tilde{\bx}:\norm{\tilde{\bx}}\le\frac19}\inf_{O:O\text{ orthogonal}} \int_{A}    \int_{0}^{\infty}\prod_{j=1}^{r}f(\inner{\mathbf{o}_j,\tilde{\bx}}\mid x)\cdot \gamma_d(x)dxd\tilde{x}_1\ldots d\tilde{x}_{r-2}\nonumber\\
				\ge& \inf_{\tilde{\bx}:\norm{\tilde{\bx}}\le\frac19}\inf_{O:O\text{ orthogonal}} \int_{A}    \int_{0}^{2}\prod_{j=1}^{r}f(\inner{\mathbf{o}_j,\tilde{\bx}}\mid x)\cdot \gamma_d(x)dxd\tilde{x}_1\ldots d\tilde{x}_{r-2}\nonumber\\
				\ge& \inf_{\tilde{\bx}:\norm{\tilde{\bx}}\le\frac19}\inf_{O:O\text{ orthogonal}} \int_{A}   \int_{0}^{2}\prod_{j=1}^{r}f(\inner{\mathbf{o}_j,\tilde{\bx}}\mid x=1)\cdot \gamma_d(x)dxd\tilde{x}_1\ldots d\tilde{x}_{r-2}\nonumber\\
				\ge& C\inf_{\tilde{\bx}:\norm{\tilde{\bx}}\le\frac19}\inf_{O:O\text{ orthogonal}} \int_{A}   \prod_{j=1}^{r}f(\inner{\mathbf{o}_j,\tilde{\bx}}\mid x=1) d\tilde{x}_1\ldots d\tilde{x}_{r-2},\label{eq:orthogonal_transformation_marginalization}
			\end{align}
			where the equality is by \eqref{eq:norm_marg}, the first inequality is because the integrand is non-negative, the second inequality is by \lemref{lem:monotonicity} since $\inner{\mathbf{o}_j,\tilde{\bx}}\le\norm{\mathbf{o}_j}\cdot\norm{\tilde{\bx}}\le1/9$ from Cauchy-Schwartz, and the last inequality is due to $\Dcal_d$ satisfying Assumption~\ref{asm:d_dist_assumption}(1) with the constant $C>0$.
			Turning to analyze the density function of the marginal given by
			\[
				p(\tilde{x}_{r-1},\tilde{x}_r) =  \int_{A}\prod_{j=1}^{r}f(\inner{\mathbf{o}_j,\tilde{\bx}}\mid x=1) d\tilde{x}_1\ldots d\tilde{x}_{r-2},
			\]
			observe that it is the density function obtained by an orthogonal transformation and marginalization on $r$ i.i.d.\ random variables with density $f(z\mid x=1)$. Conveniently, we have from \eqref{eq:density} that $f(z\mid x=1) = 0.5$ for all $z\in[-0.5,0.5]$. Namely, this is a uniform distribution on the interval $[-0.5,0.5]$, which has variance $\frac{1}{12}$ and is clearly log-concave in $z$ by definition. \lemref{lem:logconcave_density} then yields that \eqref{eq:orthogonal_transformation_marginalization} is lower bounded by
			\[
				C\cdot\frac{1}{12}\cdot2^{-16} \ge C\cdot10^{-6},
			\]
			concluding the proof of \thmref{thm:well-behaved_random_features}.
		\end{proof}

		\section{Gradient Descent Convergence Proofs}\label{app:optimization_proofs}
	
		The proofs of Propositions~(\ref{prop:gradient_dot_product_bound},\ref{prop:low_loss_attained}) build on the following lemmas, where the first which is presented below, bounds the growth function of the hypothesis class of predictors that are intersections of halfspaces. This lemma will allow us to bound the empirical Rademacher complexity of certain function classes that we will encounter in the proofs of the propositions.
		
		\begin{lemma}\label{lem:halfspaces_intersection_growth}
			Let $r\ge2$. Let $\Hcal_{r,m}$ denote the hypothesis class of intersections of $m$ halfspaces in $r$-dimensional Euclidean space (i.e.\ $\bx\mapsto1$ if and only if it is contained inside all $m$ halfspaces). Then the growth function of $\Hcal_{r,m}$ denoted $\tau_{\Hcal_{r,m}}(\cdot)$ satisfies
			\[
				\tau_{\Hcal_{r,m}}(n) \le n^{m(r+1)}
			\]
			for all $n>r+1$.
		\end{lemma}
		
		\begin{proof}
			Let $\Hcal_r$ denote the class of predictors in $\reals^r$ parameterized by $\bw\in\reals^r$ and $b\in\reals$ such that $h_{\bw,b}(\bx)=\one{\bw^{\top}\bx+b\ge0}$.
			It is a well-known fact that the VC-dimension of $\Hcal_r$ is $r+1$. By the Sauer-Shelah lemma \citep[Lemma~6.10]{shalev2014understanding} and the assumption $r\ge2$, we have
			\[
				\tau_{\Hcal_{r}}(n) \le \p{\frac{en}{r+1}}^{r+1} \le n^{r+1}.
			\]
			Next, observe that for any $h\in\Hcal_{r,m}$ and a set $C=\{\tilde{\bx}_1,\ldots,\tilde{\bx}_n\}$, the number of different labellings on $C$ produced by hypotheses in $\Hcal_{r,m}$ is uniquely determined by a set of $m$ halfspaces. Since each such halfspace can produce at most $n^{r+1}$ different labellings on the set $C$, the size of the different labellings attained by hypotheses in $\Hcal_{r,m}$ is therefore upper bounded by $\p{n^{r+1}}^m=n^{m(r+1)}$, concluding the proof of the lemma.
		\end{proof}
	
		The following lemma bounds the empirical Rademacher complexity of certain function classes that control the empirical risk we obtain over our data.
		
		\begin{lemma}\label{lem:rademacher_uniform_convergence}
			Fix $\bv\in\reals^r$, assume $n>r+1$ and define the function classes
			\begin{align*}
				&\Fcal_1\coloneqq\set{\tilde{\bx}\mapsto \one{\bw^{\top}\tilde{\bx}\ge1}\cdot\one{\bv^{\top}\tilde{\bx}\le0} + \one{\bw^{\top}\tilde{\bx}\le0}\cdot\one{\bv^{\top}\tilde{\bx}\ge1}:\bw\in\reals^r},\\
				&\Fcal_2\coloneqq\set{\tilde{\bx}\mapsto \one{\bw^{\top}\tilde{\bx}\in(0,1)}\cdot\one{\bv^{\top}\tilde{\bx}\in(0,1)} \cdot \p{ \frac{(\bw-\bv)^{\top}}{\norm{\bw-\bv}} \tilde{\bx}}^2:\bw\in\reals^r, \bw\neq\bv }.
			\end{align*}
			Let $f_{\bw}\in\Fcal_1$. Then for any distribution $\tilde{\Dcal}$ we sample $\tilde{\bx}_i$ from, we have with probability at least $1-\delta$ that
			\[
				\sup_{\bw\in\reals^r}\abs{\frac{1}{n}\sum_{i=1}^{n}f_{\bw}(\tilde{\bx}_i) -\E_{\tilde{\bx}_i\sim\tilde{\Dcal}}\pcc{f_{\bw}(\tilde{\bx}_i)}} \le 4\sqrt{\frac{(8r+8)\log_2(n)}{n}} + \sqrt{\frac{2\ln\p{2/\delta}}{n}}.
			\]
			Further, if $f_{\bw}\in\Fcal_2$
			then for any distribution $\tilde{\Dcal}$ we sample $\tilde{\bx}_i$ from, which satisfies $\norm{\tilde{\bx}_i}\le \sqrt{r}$ almost surely, we have with probability at least $1-\delta$ that
			\[
				\sup_{\bw\in\reals^r}\abs{\frac{1}{n}\sum_{i=1}^{n}f_{\bw}(\tilde{\bx}_i) -\E_{\tilde{\bx}_i\sim\tilde{\Dcal}}\pcc{f_{\bw}(\tilde{\bx}_i)}} \le 4r\sqrt{\frac{(8r+8)\log_2(n)}{n}} + \sqrt{\frac{2\ln\p{2/\delta}}{n}}.
			\]
		\end{lemma}
		
		\begin{proof}
			We shall use a Rademacher complexity argument to prove the lemma. Starting with $\Fcal_1$, we have by standard Rademacher complexity arguments that
			\begin{equation}\label{eq:rad_upper_bound}
				\sup_{\bw\in\reals^r}\abs{\frac{1}{n}\sum_{i=1}^{n}f_{\bw}(\tilde{\bx}_i) -\E_{\tilde{\bx}_i\sim\tilde{\Dcal}}\pcc{f_{\bw}(\tilde{\bx}_i)}} \le 2R_n\p{\Fcal_1\p{\tilde{\bx}_1,\ldots,\tilde{\bx}_n}} + \sqrt{\frac{2\ln\p{2/\delta}}{n}}
			\end{equation}
			\citep[Thm.~3.2]{boucheron2005theory}, where $R_n\p{\Fcal_1\p{\tilde{\bx}_1,\ldots,\tilde{\bx}_n}} \coloneqq \E\pcc{\sup_{f_{\bw}\in\Fcal_1}\abs{\frac{1}{n} \sum_{i=1}^{n}\xi_i f_{\bw}(\tilde{\bx}_i) }}$ is the empirical Rademacher complexity of $\Fcal_1$, and the expectation is over $\xi_1,\ldots,\xi_n$ which are i.i.d.\ Rademacher random variables. Next, we have that the empirical Rademacher complexity of $\Fcal_1$ can be upper bounded by
			\[
				\sqrt{\frac{2\log_2(\tau_{\Fcal_1}(n))}{n}},
			\]
			where $\tau_{\Fcal_1}(\cdot)$ is the growth function of $\Fcal_1$ \citep[Eq.~(6)]{boucheron2005theory}. Since the class $\Fcal_1$ is contained inside the class of intersections of $4$ halfspaces, its growth function is upper bounded by the growth function of the class of such intersections, and thus by virtue of \lemref{lem:halfspaces_intersection_growth} the above is at most
			\[
				\sqrt{\frac{(8r+8)\log_2(n)}{n}}.
			\]
			Plugging this back in \eqref{eq:rad_upper_bound}, we arrive at
			\[
				\sup_{\bw\in\reals^r}\abs{\frac{1}{n}\sum_{i=1}^{n}f_{\bw}(\tilde{\bx}_i) -\E_{\tilde{\bx}_i\sim\tilde{\Dcal}}\pcc{f_{\bw}(\tilde{\bx}_i)}} \le 2\sqrt{\frac{(8r+8)\log_2(n)}{n}} + \sqrt{\frac{2\ln\p{2/\delta}}{n}},
			\]
			where the bound on $\Fcal_1$ now follows from \citet[Lemma~26.6]{shalev2014understanding} after a simple scaling to accommodate for the difference between classifiers mapping to $\{0,1\}$ and $\{-1,1\}$.
			
			Turning to bound the difference for $\Fcal_2$, by \eqref{eq:rad_upper_bound} it suffices to upper bound the empirical Rademacher complexity of $\Fcal_2$. For any $i\in[n]$, define the functions $\phi_i(x)\coloneqq x\p{\frac{(\bw-\bv)^{\top}}{\norm{\bw-\bv}} \tilde{\bx}_i}^2$ and for $\ba\in\reals^n$ let $\boldsymbol{\phi}(\ba)=(\phi_1(a_1),\ldots,\phi_n(a_n))$. We now have that
			\[
				R_n\p{\Fcal_2\p{\tilde{\bx}_1,\ldots,\tilde{\bx}_n}} \le R_n\p{\boldsymbol{\phi}\circ\Fcal_1\p{\tilde{\bx}_1,\ldots,\tilde{\bx}_n}} \le r\cdot R_n\p{\Fcal_1\p{\tilde{\bx}_1,\ldots,\tilde{\bx}_n}},
			\]
			where the first inequality is by the definitions of $\Fcal_1$ and $\Fcal_2$ and the fact that changing the orientation (or omitting some) of the halfspaces defining $\Fcal_1$ results in the same growth function bound, and the second inequality follows from the contraction lemma \citep[Lemma~26.9]{shalev2014understanding} where the Lipschitz constant of $\phi_i$ is upper bounded by $r$ for all $i$ due to Cauchy-Schwartz and the fact that $\norm{\tilde{\bx}_i}\le\sqrt{r}$ almost surely. The bound on $\Fcal_2$ then follows immediately from our previously derived bound on $\Fcal_1$.
		\end{proof}
		
		\begin{lemma}\label{lem:region_expectation_inequality}
			Suppose $\mathbf{0}\neq\bw,\bv\in\reals^r$ such that $\theta\coloneqq\theta_{\bw,\bv}>0$, $\norm{\bw-\bv}\le\norm{\bv}\le r^{-0.25}$ and $r\ge10^6$. Let $A_{1,3}=(-\infty,0]\times[1,\infty)$ and $A_{3,1}=[1,\infty)\times(-\infty,0]$. Then under Assumption~\ref{asm:init}, we have
			\[
				\E_{\tilde{\bx}}\pcc{ \one{\p{\bw^{\top}\tilde{\bx},\bv^{\top}\tilde{\bx}}\in A_{1,3}\cup A_{3,1}} } \le \frac1r.
			\]
		\end{lemma}
	
		\begin{proof}
			Since the expectation of an indicator function is the probability of the event, we have
			\begin{equation}\label{eq:expectation_to_probability}
				\E_{\tilde{\bx}}\pcc{ \one{\p{\bw^{\top}\tilde{\bx},\bv^{\top}\tilde{\bx}}\in A_{1,3}\cup A_{3,1}} } = \pr_{\tilde{\bx}}\pcc{ \one{\p{\bw^{\top}\tilde{\bx},\bv^{\top}\tilde{\bx}}\in A_{1,3}\cup A_{3,1}}}.
			\end{equation}
			In what follows, we consider the projections of $\bw,\bv$, denoted by $\hat{\bw},\hat{\bv}$, onto the two-dimensional space spanned by them, so that $\hat{\bx}\in\reals^2$ is sampled from the marginal distribution $\tilde{\Dcal}_{\bw,\bv}$. This is justified since the expression we are bounding in the lemma only depends on dot products of $\tilde{\bx}$ with $\bw,\bv$ which do not change when we perform such a projection.
			
			We shall now derive a bound on the above probability by bounding $\norm{\hat{\bx}}$ with high probability over the randomness in sampling it from the marginal distribution. Recall that by \lemref{lem:dist_properties} we have that conditioned on $\norm{\tilde{\bx}}$, each coordinate $x_j$ of $\hat{\bx}$ is distributed according to $\sum_{i=1}^{n}o_{j,i}\tilde{X}_i$, where $\tilde{X}_i$ are i.i.d.\ with density given in \eqref{eq:density} and $\bo_j=(o_{j,1},\ldots,o_{j,n})$, $j\in\{1,2\}$, are orthogonal. Since $\tilde{X}_i$ is bounded in $[-1,1]$ and $\E\pcc{\tilde{X}_i}=0$ regardless of the realization of $\norm{\tilde{\bx}}$, we can use Hoeffding's inequality to deduce that for all $j\in\{1,2\}$,
			\[
				\pr_{\tilde{\bx}\sim\tilde{\Dcal}}\pcc{|x_j|\ge\frac{r^{0.25}}{2\sqrt{2}\theta}} \le 2\exp\p{-\frac{\sqrt{r}}{4\theta^2}}.
			\]
			By a union bound and since the ball of radius $\frac{r^{0.25}}{2\theta}$ contains the square $\pcc{-\frac{r^{0.25}}{2\sqrt{2}\theta},\frac{r^{0.25}}{2\sqrt{2}\theta}}^2$, we have
			\begin{equation}\label{eq:theta_tail_bound}
				\pr_{\hat{\bx}\sim\tilde{\Dcal}_{\bw,\bv}}\pcc{\norm{\hat{\bx}}\ge\frac{r^{0.25}}{2\theta}} \le \pr_{\hat{\bx}\sim\tilde{\Dcal}_{\bw,\bv}}\pcc{\hat{\bx}\notin \pcc{-\frac{r^{0.25}}{2\sqrt{2}\theta},\frac{r^{0.25}}{2\sqrt{2}\theta}}^2 } \le 4\exp\p{-\frac{\sqrt{r}}{4\theta^2}}.
			\end{equation}
			Now, define
			\[
				\bar{\bw}\coloneqq \frac{\norm{\hat{\bv}}}{\norm{\hat{\bw}}}\hat{\bw},\hskip 0.5cm \bar{\bv}\coloneqq \frac{\norm{\hat{\bw}}}{\norm{\hat{\bv}}}\hat{\bv}.
			\]
			Since $\theta\coloneqq\theta_{\bw,\bv}=\theta_{\bar{\bw},\hat{\bv}}=\theta_{\hat{\bw},\bar{\bv}}$, we have
			\begin{equation}\label{eq:theta_norm_bound}
				\norm{\bar{\bw}-\hat{\bv}}^2=\norm{\bar{\bw}}^2+\norm{\hat{\bv}}^2-2\norm{\bar{\bw}}\norm{\hat{\bv}}\cos\p{\theta} = 2\norm{\hat{\bv}}^2(1-\cos(\theta)) \le \frac{\theta^2}{\sqrt{r}},
			\end{equation}
			where in the inequality we used $\norm{\hat{\bv}}^2\le1/\sqrt{r}$ and $1-\cos(\theta)\le\frac{\theta^2}{2}$ which holds for all $\theta\in\reals$ by a Taylor expansion. From a similar argument we have
			\begin{equation}\label{eq:theta_norm_bound2}
				\norm{\hat{\bw}-\bar{\bv}} \le 2\norm{\hat{\bw}}^2(1-\cos(\theta)) \le \frac{4\theta^2}{\sqrt{r}},
			\end{equation}
			where we used the triangle inequality and our assumption to obtain $\norm{\hat{\bw}}\le\norm{\hat{\bw}-\hat{\bv}}+\norm{\hat{\bv}}\le2\norm{\hat{\bv}}$.
			Using the above shorthands, we also have that
			\[
				\p{\hat{\bw}^{\top}\hat{\bx},\hat{\bv}^{\top}\hat{\bx}}\in A_{1,3} \iff \p{\bar{\bw}^{\top}\hat{\bx},\hat{\bv}^{\top}\hat{\bx}}\in A_{1,3},
			\]
			and
			\[
				\p{\hat{\bw}^{\top}\hat{\bx},\hat{\bv}^{\top}\hat{\bx}}\in A_{3,1} \iff \p{\hat{\bw}^{\top}\hat{\bx},\bar{\bv}^{\top}\hat{\bx}}\in A_{3,1}.
			\]
			This holds true since $x\in(-\infty,0]$ if and only if $ax\in(-\infty,0]$ for all $a>0$. A necessary condition for $\p{\bar{\bw}^{\top}\hat{\bx},\hat{\bv}^{\top}\hat{\bx}}\in A_{1,3}$ to hold is that
			\[
				\abs{\bar{\bw}^{\top}\hat{\bx}-\hat{\bv}^{\top}\hat{\bx}} \ge 1,
			\]
			since otherwise, if the two dot products are less than that apart then the inclusion of either in one of the intervals $(-\infty,0],[1,\infty)$ excludes the other from the second interval. We now compute using Cauchy-Schwartz and the square roots of Eqs.~(\ref{eq:theta_norm_bound},\ref{eq:theta_norm_bound2}) to obtain
			\[
				\abs{\bar{\bw}^{\top}\hat{\bx}-\hat{\bv}^{\top}\hat{\bx}} \le \norm{\bar{\bw}-\hat{\bv}}\cdot\norm{\hat{\bx}} \le r^{-0.25}\theta\norm{\hat{\bx}},
			\]
			and
			\[
				\abs{\hat{\bw}^{\top}\hat{\bx}-\bar{\bv}^{\top}\hat{\bx}} \le \norm{\hat{\bw}-\bar{\bv}}\cdot\norm{\hat{\bx}} \le 2r^{-0.25}\theta\norm{\hat{\bx}},
			\]
			That is, $\norm{\hat{\bx}}<0.5r^{0.25}/\theta$ implies $\p{\hat{\bw}^{\top}\hat{\bx},\hat{\bv}^{\top}\hat{\bx}}\notin A_{1,3}\cup A_{3,1}$. Along with \eqref{eq:theta_tail_bound}, this means that
			\[
				\pr_{\hat{\bx}\sim\tilde{\Dcal}_{\bw,\bv}}\pcc{\p{\hat{\bw}^{\top}\hat{\bx},\hat{\bv}^{\top}\hat{\bx}}\in A_{1,3}\cup A_{3,1}} \le \pr_{\hat{\bx}\sim\tilde{\Dcal}_{\bw,\bv}}\pcc{\norm{\hat{\bx}}\ge\frac{r^{0.25}}{2\theta}} \le 4\exp\p{-\frac{\sqrt{r}}{4\theta^2}}.
			\]
			Plugging the above in \eqref{eq:expectation_to_probability} and returning to $r$-dimensional space, we have arrived at
			\begin{equation*}\label{eq:exp_tail}
				\E_{\tilde{\bx}}\pcc{ \one{\p{\bw^{\top}\tilde{\bx},\bv^{\top}\tilde{\bx}}\in A_{1,3}\cup A_{3,1}} } \le 4\exp\p{-\frac{\sqrt{r}}{4\theta^2}} \le 4\exp\p{-\frac{\sqrt{r}}{4\pi^2}}\le\frac1r,
			\end{equation*}
			where the last inequality follows from the inequality $4\exp(-\sqrt{x}/4\pi^2)\le1/x$ which holds for all $x\ge10^6$.
		\end{proof}
		
		Having established the required tools to prove Propositions~(\ref{prop:gradient_dot_product_bound},\ref{prop:low_loss_attained}), we now turn to do so.
	
		\subsection{Proof of \propref{prop:gradient_dot_product_bound}}\label{app:proof_of_gradient_dot_product_bound}
	
		\begin{proof}
			First note that the proposition statement is trivially true if $\bw=\bv$. Assume $\bw\neq\bv$ for the remainder of the proof and define $\epsilon_i\coloneqq\crelu{\bv^{\top}\tilde{\bx}_i} - y_i$. We compute
			\begin{align*}
				\inner{\nabla F(\bw),\bw-\bv}&= \frac1n\sum_{i=1}^n\p{\crelu{\bw^{\top}\tilde{\bx}_i} - y_i}\cdot \one{\bw^{\top}\tilde{\bx}_i\in(0,1)} \cdot\p{\bw^{\top}\tilde{\bx}_i-\bv^{\top}\tilde{\bx}_i}\\
				&= \frac1n\sum_{i=1}^n\p{\crelu{\bw^{\top}\tilde{\bx}_i} - \crelu{\bv^{\top}\tilde{\bx}_i} + \epsilon_i }\cdot\one{\bw^{\top}\tilde{\bx}_i\in(0,1)}\cdot\p{\bw^{\top}\tilde{\bx}_i-\bv^{\top}\tilde{\bx}_i}\\
				&= \frac1n\sum_{i=1}^n\p{\crelu{\bw^{\top}\tilde{\bx}_i} - \crelu{\bv^{\top}\tilde{\bx}_i}  }\cdot\one{\bw^{\top}\tilde{\bx}_i\in(0,1)}\cdot\p{\bw^{\top}\tilde{\bx}_i-\bv^{\top}\tilde{\bx}_i}\\%\label{eq:v_term}\\
				&\hskip 1cm +\frac1n\sum_{i=1}^n\epsilon_i \cdot\one{\bw^{\top}\tilde{\bx}_i\in(0,1)}\cdot\p{\bw^{\top}\tilde{\bx}_i-\bv^{\top}\tilde{\bx}_i}.%\label{eq:epsilon_term}.
			\end{align*}
			By \eqref{eq:small_noise}, the above is at least
			\[
				\frac{1}{2n}\sum_{i=1}^n\p{\crelu{\bw^{\top}\tilde{\bx}_i} - \crelu{\bv^{\top}\tilde{\bx}_i}  }\cdot\one{\bw^{\top}\tilde{\bx}_i\in(0,1)}\cdot\p{\bw^{\top}\tilde{\bx}_i-\bv^{\top}\tilde{\bx}_i}.
			\]
			Since $\crelu{\cdot}$ is non-decreasing, $\p{\crelu{a} - \crelu{b}}\cdot \p{a-b} \ge 0$ for all $a,b\in\reals$, and therefore each summand in the above is non-negative and we can omit summands in which $\bv^{\top}\hat{\bx}_i\notin(0,1)$, to lower bound the above by
			\begin{align*}
				&\frac{1}{2n}\sum_{i=1}^n\p{\crelu{\bw^{\top}\tilde{\bx}_i} - \crelu{\bv^{\top}\tilde{\bx}_i}  }\cdot\one{\bw^{\top}\tilde{\bx}_i\in(0,1)}\cdot\one{\bv^{\top}\tilde{\bx}_i\in(0,1)}\cdot\p{\bw^{\top}\tilde{\bx}_i-\bv^{\top}\tilde{\bx}_i}\\
				&\hskip 1cm= \frac{1}{2n}\norm{\bw-\bv}^2\sum_{i=1}^n\p{\frac{(\bw-\bv)^{\top}}{\norm{\bw-\bv}}\tilde{\bx}_i}^2\cdot\one{\bw^{\top}\tilde{\bx}_i\in(0,1)}\cdot\one{\bv^{\top}\tilde{\bx}_i\in(0,1)}.
			\end{align*}
			By \lemref{lem:rademacher_uniform_convergence}, with probability at least $1-\delta$, the above is lower bounded uniformly for all $\bw$ by
			\begin{align}
				&\frac12\norm{\bw-\bv}^2\left(\E_{\tilde{\bx}}\pcc{\p{\frac{(\bw-\bv)^{\top}}{\norm{\bw-\bv}}\tilde{\bx}_i}^2\cdot\one{\bw^{\top}\tilde{\bx}\in(0,1)}\cdot\one{\bv^{\top}\tilde{\bx}\in(0,1)}}\right. \nonumber\\
				&~~~~~~~~~~~~~~~~~~~~ \left.- 4r\sqrt{\frac{(8r+8)\log_2(n)}{n}} - \sqrt{\frac{2\log\p{2/\delta}}{n}}\right).\label{eq:expectation_lower_bound}
			\end{align}
			Let $P$ denote the orthogonal projection matrix projecting onto the subspace spanned by $\bw,\bv$. Since the expectation above is non-negative, we can lower bound it by
			\begin{equation}\label{eq:gilad_bound}
				\E_{\tilde{\bx}}\pcc{\p{\frac{(\bw-\bv)^{\top}}{\norm{\bw-\bv}}\tilde{\bx}}^2\cdot\one{\norm{P\tilde{\bx}}\le\alpha}\cdot\one{\bw^{\top}\tilde{\bx}\in(0,1)}\cdot\one{\bv^{\top}\tilde{\bx}\in(0,1)}}.
			\end{equation}
			We now argue that $\norm{\bw-\bv}^2\le\norm{\bv}^2$ implies that $\theta_{\bw,\bv}\le\frac{\pi}{2}$. This follows from
			\[
				\norm{\bw}^2-2\inner{\bw,\bv}+\norm{\bv}^2=\norm{\bw-\bv}^2\le\norm{\bv}^2,
			\]
			from which we have $\inner{\bw,\bv}\ge\frac{\norm{\bw}^2}{2}$, i.e.\ the dot product between $\bw$ and $\bv$ is strictly positive and therefore
			\begin{equation}\label{eq:angle_bound}
				\theta_{\bw,\bv}\le\frac{\pi}{2}.
			\end{equation}
			Next, we have from the triangle inequality and our assumption that $\norm{\bw}\le\norm{\bw-\bv}+\norm{\bv}\le2\norm{\bv}\le1$. Since $P$ is an orthogonal projection that keeps $\bw,\bv$ unchanged, we have that $\bw^{\top}\tilde{\bx}=P\bw^{\top}\tilde{\bx}=\bw^{\top}P\tilde{\bx}$ and $\bv^{\top}\tilde{\bx}=P\bv^{\top}\tilde{\bx}=\bv^{\top}P\tilde{\bx}$. We thus have for any $\tilde{\bx}$ in \eqref{eq:gilad_bound} by using Cauchy-Schwartz and \thmref{thm:well-behaved_random_features} which guarantees $\alpha\le1$, that $\abs{\bw^{\top}\tilde{\bx}} \le \norm{\bw}\cdot\norm{P\tilde{\bx}} \le \alpha\le\frac19$. Letting $\hat{\bw},\hat{\bv}\in\reals^2$ denote the projections of $\bw,\bv$ onto the two-dimensional subspace spanned by them, we have that \eqref{eq:gilad_bound} equals
			\begin{align}
				&\E_{\tilde{\bx}}\pcc{\p{\frac{(\bw-\bv)^{\top}}{\norm{\bw-\bv}}\tilde{\bx}}^2\cdot\one{\norm{P\tilde{\bx}}\le\alpha}\cdot\one{\bw^{\top}\tilde{\bx}>0}\cdot\one{\bv^{\top}\tilde{\bx}>0}}\nonumber\\
				=& \E_{\hat{\bx}\sim\tilde{\Dcal}_{\bw,\bv}}\pcc{\p{\frac{(\hat{\bw}-\hat{\bv})^{\top}}{\norm{\hat{\bw}-\hat{\bv}}}\hat{\bx}}^2\cdot\one{\norm{\hat{\bx}}\le\alpha}\cdot\one{\hat{\bw}^{\top}\hat{\bx}>0}\cdot\one{\hat{\bv}^{\top}\hat{\bx}>0}}\label{eq:marginal}\\
				\ge& \inf_{\bu\in\reals^2:\norm{\bu}=1}\E_{\hat{\bx}\sim\tilde{\Dcal}_{\bw,\bv}}\pcc{\p{\bu^{\top}\hat{\bx}}^2\cdot \one{\norm{\hat{\bx}}\le\alpha} \cdot\one{\hat{\bw}^{\top}\hat{\bx}>0}\cdot \one{\hat{\bv}^{\top}\hat{\bx}>0}}\nonumber\\
				\ge& \beta\inf_{\bu\in\reals^2:\norm{\bu}=1}\int_{\reals^2}\p{\bu^{\top}\hat{\bx}}^2\cdot \one{\norm{\hat{\bx}}\le\alpha} \cdot\one{\hat{\bw}^{\top}\hat{\bx}>0}\cdot \one{\hat{\bv}^{\top}\hat{\bx}>0}d\hat{\bx}\label{eq:beta_lower_bound}\\
				\ge&  \frac{\alpha^4\beta}{8\sqrt{2}}\sin^3\p{\frac{\pi}{8}}\ge \frac{\alpha^4\beta}{210} \label{eq:gilad_lemma}.
			\end{align}
			In the above, \eqref{eq:marginal} follows from the fact that all the dot products are in the subspace spanned by $\bw,\bv$, so we can take the expectation over the marginal distribution $\tilde{\Dcal}_{\bw,\bv}$; \eqref{eq:beta_lower_bound} follows from \thmref{thm:well-behaved_random_features}, since the marginal density of $\tilde{\Dcal}_{\bw,\bv}$ is lower bounded by $\beta$ for all $\hat{\bx}$ with norm at most $1/9$; and the inequalities in \eqref{eq:gilad_lemma} are by \citet[Lemma~B.1]{yehudai2020learning} applied to the vectors $\bw$ and $\bv$ and by \eqref{eq:angle_bound}. Plugging the lower bound on the expectation we obtained above in \eqref{eq:expectation_lower_bound}, the proposition follows.
		\end{proof}
	
		\subsection{Proof of \propref{prop:low_loss_attained}}\label{app:proof_of_low_loss_attained}
		
		\begin{proof}
			We begin with upper bounding the loss of making predictions using $\bw$ when the target values are determined by $\bv$:
			\begin{equation}\label{eq:pred_v_with_w}
				\frac1n\sum_{i=1}^{n}\p{\crelu{\bw^{\top}\tilde{\bx}_i} - \crelu{\bv^{\top}\tilde{\bx}_i}}^2,
			\end{equation}
			in terms of the following surrogate loss
			\begin{equation}\label{eq:surragate_loss}
				\frac1n\sum_{i=1}^{n}\p{\crelu{\bw^{\top}\tilde{\bx}_i} - \crelu{\bv^{\top}\tilde{\bx}_i}}^2\cdot\one{\bw^{\top}\tilde{\bx}_i\in(0,1)}.
			\end{equation}
			To this end, we partition the sums into $9$ parts, depending on the values the dot products take in the following manner: Define the intervals $I'_1=(-\infty,0]$, $I'_2=(0,1)$ and $I'_3=[1,\infty)$, and for $j_1,j_2\in\{1,2,3\}$ define the sets $A_{j_1,j_2}=I'_{j_1}\times I'_{j_2}$. We now wish to upper bound
			\[
				\frac1n\sum_{i=1}^{n}\p{\crelu{\bw^{\top}\tilde{\bx}_i} - \crelu{\bv^{\top}\tilde{\bx}_i}}^2 = \frac1n\sum_{j_1,j_2\in\{1,2,3\}} ~\sum_{i:\p{\bw^{\top}\tilde{\bx}_i,\bv^{\top}\tilde{\bx}_i}\in A_{j_1,j_2}} \p{\crelu{\bw^{\top}\tilde{\bx}_i} - \crelu{\bv^{\top}\tilde{\bx}_i}}^2.
			\]
			We will do so by analyzing three different cases:
			\begin{itemize}
				\item
				We begin with the observation that the sums are identical over the sets $A_{2,1},A_{2,2},A_{2,3}$ and $A_{1,1},A_{3,3}$. This is because for instances $\tilde{\bx}_i$ with dot products in $A_{1,1},A_{3,3}$, we have that $\crelu{\bw^{\top}\tilde{\bx}_i} = \crelu{\bv^{\top}\tilde{\bx}_i}$, and therefore all such summands are identically zero in both \eqref{eq:pred_v_with_w} and \eqref{eq:surragate_loss}; and for instances with dot products in $A_{2,1},A_{2,2},A_{2,3}$ we have that $\one{\bw^{\top}\tilde{\bx}_i\in(0,1)}=1$, so each summand in both equations is identical.
				
				\item
				Turning to analyze the sums for $A_{1,2},A_{3,2}$, we use \thmref{thm:random_features} with a confidence of $\delta/2$ to obtain
				\[
					\abs{\set{i:\p{\bw^{\top}\tilde{\bx}_i,\bv^{\top}\tilde{\bx}_i}\in A_{1,2}\cup A_{3,2}}} \le \frac{3600q(d)+1000c_2\sqrt{\ln\p{16n/\delta}}}{r^{0.25}}n.
				\]
				This is true since $\bv^{\top}\tilde{\bx}_i\in (0,1)$ implies that $\bv$ misclassifies  $\tilde{\bx}_i$ which can only happen on the above fraction of the data instances. This allows us to derive the following upper bound
				\begin{align*}
					&\frac1n\sum_{i:\p{\bw^{\top}\tilde{\bx}_i,\bv^{\top}\tilde{\bx}_i}\in A_{1,2}\cup A_{3,2}}\p{\crelu{\bw^{\top}\tilde{\bx}_i} - \crelu{\bv^{\top}\tilde{\bx}_i}}^2 \le \frac1n\sum_{i:\p{\bw^{\top}\tilde{\bx}_i,\bv^{\top}\tilde{\bx}_i}\in A_{1,2}\cup A_{3,2}} 1\\
					&\hskip 1cm
					\le \frac{3600q(d)+1000c_2\sqrt{\ln\p{16n/\delta}}}{r^{0.25}}.
				\end{align*}
				\item
				Lastly, assuming that $\theta_{\bw,\bv}>0$ we have uniformly for any $\bw\in\reals^r$ with probability at least $1-\frac{\delta}{2}$, the sums over the sets $A_{1,3},A_{3,1}$ are at most
				\begin{align*}
					&\frac1n\sum_{i:\p{\bw^{\top}\tilde{\bx}_i,\bv^{\top}\tilde{\bx}_i}\in A_{1,3}\cup A_{3,1}}\p{\crelu{\bw^{\top}\tilde{\bx}_i} - \crelu{\bv^{\top}\tilde{\bx}_i}}^2
					=
					\frac1n\sum_{i=1}^{n}\one{i:\p{\bw^{\top}\tilde{\bx}_i,\bv^{\top}\tilde{\bx}_i}\in A_{1,3}\cup A_{3,1}}\\
					&\hskip 0.6cm\le
					\E_{\tilde{\bx}}\pcc{ \one{\p{\bw^{\top}\tilde{\bx}_i,\bv^{\top}\tilde{\bx}_i}\in A_{1,3}\cup A_{3,1}} } + 4\sqrt{\frac{(8r+8)\log_2(n)}{n}} + \sqrt{\frac{2\ln\p{4/\delta}}{n}}\\
					&\hskip 0.6cm\le \frac1r + 4\sqrt{\frac{(8r+8)\log_2(n)}{n}} + \sqrt{\frac{2\ln\p{4/\delta}}{n}}.
				\end{align*}
				where the equality is due to $\crelu{\bw^{\top}\tilde{\bx}_i} - \crelu{\bv^{\top}\tilde{\bx}_i}=1$ for any $\tilde{\bx}_i$ with dot products in $A_{1,3}\cup A_{3,1}$, the first inequality is due to \lemref{lem:rademacher_uniform_convergence} using a confidence of $\delta/2$, and the last inequality is due to \lemref{lem:region_expectation_inequality}.
			\end{itemize}
			We can now combine the above three cases using a union bound to deduce that with probability at least $1-\delta$
			\begin{align}
				&\frac1n\sum_{i=1}^{n}\p{\crelu{\bw^{\top}\tilde{\bx}_i}- \crelu{\bv^{\top}\tilde{\bx}_i}}^2\nonumber\\
				&\hskip 2cm  - \frac1r - 4\sqrt{\frac{(8r+8)\log_2(n)}{n}} - \sqrt{\frac{2\log\p{4/\delta}}{n}} - \frac{3600q(d)+1000c_2\sqrt{\ln\p{16n/\delta}}}{r^{0.25}}\nonumber\\
				&\hskip 1cm \le  \frac1n\sum_{i=1}^{n}\p{\crelu{\bw^{\top}\tilde{\bx}_i} - \crelu{\bv^{\top}\tilde{\bx}_i}}^2\cdot\one{\bw^{\top}\tilde{\bx}_i\in(0,1)}\nonumber\\
				&\hskip 1cm \le  \frac1n\sum_{i=1}^{n}\p{\crelu{\bw^{\top}\tilde{\bx}_i} - \crelu{\bv^{\top}\tilde{\bx}_i}}\cdot\one{\bw^{\top}\tilde{\bx}_i\in(0,1)} \cdot \p{\bw^{\top}\tilde{\bx}_i-\bv^{\top}\tilde{\bx}_i}, \label{eq:w_v_alpha_beta_bound}
			\end{align}
			where in the last inequality we used that fact that $\crelu{\cdot}$ is $1$-Lipschitz. We remark that in the case where $\theta_{\bw,\bv}=0$, the only possible non-empty dot product sets are $A_{1,1},A_{2,2},A_{3,3}$ and possibly also at most one of either $A_{2,3},A_{3,2}$, therefore the upper bound we derived above is also valid in the case where $\theta_{\bw,\bv}=0$. This holds true since we can simply skip the third case which analyzes the sums with dot products in $A_{1,3},A_{3,1}$ which requires that $\theta_{\bw,\bv}>0$ to invoke \lemref{lem:region_expectation_inequality}.
			
			Turning to upper bound the expression 
			\[
				\frac 2n\sum_{i=1}^n \p{\crelu{\bv^{\top}\tilde{\bx}_i} - y_i} \cdot \one{\bw^{\top}\tilde{\bx}_i\in(0,1)} \cdot\p{\bw^{\top}\tilde{\bx}_i-\bv^{\top}\tilde{\bx}_i},
			\]
			we have from \thmref{thm:random_features} that the above loss can be decomposed into the loss over four intervals $I_1,\ldots,I_4$ and upper bounded as follows (note that we previously assumed that the theorem implications hold so we don't need to further evaluate the confidence over which they do)
			\begin{align*}
				&\frac 2n\sum_{j=1}^4\sum_{i:\norm{\bx_i}\in I_j} \p{\crelu{\bv^{\top}\tilde{\bx_i}} - y_i} \cdot \one{\bw^{\top}\tilde{\bx}_i\in(0,1)} \cdot\p{\bw^{\top}\tilde{\bx}_i-\bv^{\top}\tilde{\bx}_i}\\
				&\hskip 1cm 
				= \frac 2n\sum_{i:\norm{\bx_i}\in I_2\cup I_4} \p{\crelu{\bv^{\top}\tilde{\bx}_i} - y_i} \cdot \one{\bw^{\top}\tilde{\bx}_i\in(0,1)} \cdot\p{\bw^{\top}\tilde{\bx}_i-\bv^{\top}\tilde{\bx}_i}\\
				&\hskip 1cm 
				\le \frac 2n\sum_{i:\norm{\bx_i}\in I_2\cup I_4} 4 \le 8\frac{3600q(d)+1000c_2\sqrt{\ln\p{16n/\delta}}}{r^{0.25}},
			\end{align*}
			where the first inequality is due to $\crelu{\bv^{\top}\tilde{\bx}_i} - y_i \le 1$, $\one{\bw^{\top}\tilde{\bx}_i\in(0,1)}\le1$ and $\abs{\bw^{\top}\tilde{\bx}_i-\bv^{\top}\tilde{\bx}_i}\le4$ for all $i$ such that $\norm{\bx_i}\in I_2\cup I_4$. Combining the above with \eqref{eq:w_v_alpha_beta_bound} and the inequality in the proposition statement, we have arrived at the following bound
			\begin{align}
				&\frac1n\sum_{i=1}^{n}\p{\crelu{\bw^{\top}\tilde{\bx}_i}- \crelu{\bv^{\top}\tilde{\bx}_i}}^2 \le 9\cdot\frac{3600q(d)+1000c_2\sqrt{\ln\p{16n/\delta}}}{r^{0.25}}  \nonumber\\ &\hskip 6.5cm
				+\frac1r + 4\sqrt{\frac{(8r+8)\log_2(n)}{n}} + \sqrt{\frac{2\ln\p{4/\delta}}{n}}. \label{eq:w_v_loss_bound}
			\end{align}
			We now use \thmref{thm:random_features} one last time to upper bound the squared loss of $\bv$ and obtain
			\[
				\frac1n\sum_{i=1}^{n}\p{\crelu{\bv^{\top}\tilde{\bx}_i} - y_i}^2 
				\le
				\frac1n \sum_{i:\norm{\bx_i}\in I_2\cup I_4}1
				\le
				\frac{3600q(d)+1000c_2\sqrt{\ln\p{16n/\delta}}}{r^{0.25}}.
			\]
			Combining the above with \eqref{eq:w_v_loss_bound} using the inequality $(a+b)^2\le2a^2+2b^2$ with $a=\crelu{\bw^{\top}\tilde{\bx}_i}- \crelu{\bv^{\top}\tilde{\bx}_i}$ and $b=\crelu{\bv^{\top}\tilde{\bx}_i} - y_i$, the proposition follows.
		\end{proof}

		\section{Proofs of \thmref{thm:gd_convergence} and \thmref{thm:generalization}}
		
		\subsection{Proof of \thmref{thm:generalization}}\label{app:gen_proof}
		\begin{proof}
			Define the function class
			\[
				\Fcal\coloneqq\set{\tilde{\bx}\mapsto\bw^{\top}\tilde{\bx}:\norm{\bw}\le1},
			\]
			and define the loss
			\[
				\ell(y,y')=\p{\crelu{y} - \crelu{y'}}^2.
			\]
			Using standard Rademacher complexity arguments \citep[Thm.~26.5]{shalev2014understanding} and the fact that $\ell$ is upper bounded by $1$, we have with probability at least $1-\delta$ that
			\[
				F(\bw) \le \hat{F}(\bw) + 2R_n\p{\ell\circ\Fcal\p{\tilde{\bx}_1,\ldots,\tilde{\bx}_n}} + 4\sqrt{\frac{2\ln(4/\delta)}{n}}.
			\]
			where $R_n\p{\Fcal\p{\tilde{\bx}_1,\ldots,\tilde{\bx}_n}} \coloneqq \E\pcc{\sup_{f_{\bw}\in\Fcal}\abs{\frac{1}{n} \sum_{i=1}^{n}\xi_i f_{\bw}(\tilde{\bx}_i) }}$ is the empirical Rademacher complexity of $\Fcal$, and the expectation is over $\xi_1,\ldots,\xi_n$ which are i.i.d.\ Rademacher random variables. To bound the empirical Rademacher complexity, we have from \citet[Lemma~26.10]{shalev2014understanding} that
			\[
				R_n\p{\Fcal\p{\tilde{\bx}_1,\ldots,\tilde{\bx}_n}} \le \frac{\max_i\norm{\tilde{\bx}_i}}{\sqrt{n}} \le \sqrt{\frac{r}{n}}.
			\]
			From the above, the contraction lemma \citep[Lemma~26.9]{shalev2014understanding} and since $\ell$ is $2$-Lipschitz, we conclude
			\[
				F(\bw) \le \hat{F}(\bw) + 4\sqrt{\frac rn} + 4\sqrt{\frac{2\ln(4/\delta)}{n}}.
			\]
		\end{proof}

		\subsection{Proof of \thmref{thm:gd_convergence}}\label{app:proof_gd_convergence}
		
		\begin{proof}
			Before we delve into the proof of the theorem, we will first establish several technical inequalities that are implied by the assumptions in the theorem statement and that will be used throughout the proof. Let $\alpha,\beta\in(0,1)$ denote the constants guaranteed by applying \thmref{thm:well-behaved_random_features} which is justified by Assumption~\ref{asm:d_dist_assumption}, and denote 
			\[
				p\coloneqq \frac{1800q(d)+500c_2\sqrt{\ln\p{64n/\delta}}}{r^{0.25}}.
			\]
			Then we have that
			\begin{align}
				&r\ge\max\set{12000^4,500^4\max\set{c_1,10}^4\ln^2(64n/\delta)},\label{eq:r_lower_bound}\\
				&\sqrt{\frac{\ln(64/\delta)}{2n}} \le p \le \frac{\epsilon}{16},\label{eq:p_bound}\\
				&4r\sqrt{\frac{(8r+8)\log_2(n)}{n}} +  \sqrt{\frac{2\ln\p{8/\delta}}{n}} \le \frac{\alpha^4\beta}{420},\label{eq:alpha_beta}\\
				&4\sqrt{\frac rn} + 4\sqrt{\frac{2\ln(16/\delta)}{n}} \le \frac{\epsilon}{2},\label{eq:rademacher_gen_bound}\\
				&40p + \frac2r + 8\sqrt{\frac{(8r+8)\log_2(n)}{n}} + 2\sqrt{\frac{2\ln\p{16/\delta}}{n}} \le \frac{\epsilon}{2}.\label{eq:9_regions_bound}
			\end{align}
			%It would suffice to substitute $r\le$
			%r\ge \max\set{12000^4,c\epsilon^{-4}\p{q^4(d) +\ln^2(1/\delta)}}$,
			%Sample size $n=r^3\log_2^2(r)$,
			We first note that \eqref{eq:9_regions_bound} implies that $40p\le\epsilon/2$ and in particular $p\le\epsilon/16$, which establishes the second inequality in \eqref{eq:p_bound}. For the first inequality in \eqref{eq:p_bound}, we compute
			\[
				\sqrt{\frac{\ln(64/\delta)}{2n}} \le \frac{1}{r^{0.25}} \cdot \frac{\sqrt{\ln(64/\delta)}}{r^{0.25}} \le \frac{1800q(d)+500c_2\sqrt{\ln\p{64n/\delta}}}{r^{0.25}} = p,
			\]
			where in the above, the first inequality is due to $2n\ge r$ and the second inequality is due to our assumption in the theorem statement which implies $r^{0.25}\ge \max\{q(d),c_2\}$ if $c>0$ is sufficiently large. Next, we compute
			\[
				\frac{\ln(64n/\delta)}{\sqrt{r}} = \frac{1}{r^{0.4}} \p{\frac{\ln(64c)}{r^{0.1}} + \frac{\ln(r^3\log_2^2(r))}{r^{0.1}}} + \frac{\ln(1/\delta)}{\sqrt{r}}.
			\]	
			In the above we first have from the assumption in the theorem statement that $r^{-0.4}\le c^{-0.4}\epsilon^2$. We can therefore bound the parenthesized expression by an absolute constant to establish that it can be made arbitrarily small. Upper bound the term
			\[
				\frac{\ln(64c)}{r^{0.1}}\le \frac{\ln(64c)}{c^{0.1}}\sqrt{\epsilon} \le 6,
			\]
			which follows from $\sup_x \ln(64x)/x^{0.25}\le6$ and $\epsilon\le1$. For the second term in the parentheses, we compute
			\[
				\frac{\ln(r^3\log_2^2(r))}{r^{0.1}} = \frac{3\ln(r)}{r^{0.1}} + \frac{2\ln(\log_2(r))}{r^{0.1}} \le 12 + 3 = 15,
			\]
			which follows from $\sup_x 3\ln(x)/x^{0.1}\le12$ and $\sup_x 2\ln(\log_2(x))/x^{0.1}\le3$. The bound $\ln(1/\delta)/\sqrt{r}\le \frac{\epsilon^2}{\sqrt{c}}$ is immediate from the assumption in the theorem statement that $r\ge c\epsilon^{-5}\ln^2(1/\delta) \ge c\epsilon^{-4}\ln^2(1/\delta)$. Combining the bounds on the three terms we have shown that
			\begin{equation}\label{eq:ln_n_over_r}
				\frac{\ln(64n/\delta)}{\sqrt{r}} \le \epsilon^2\p{\frac{21}{c^{0.4}} + \frac{1}{\sqrt{c}}}.
			\end{equation}
			Squaring the above, rearranging and scaling $c$ to be sufficiently large (while recalling that $\epsilon\le1$) implies \eqref{eq:r_lower_bound}.
			
			To establish the remaining inequalities in Eqs.~(\ref{eq:alpha_beta}-\ref{eq:9_regions_bound}), we would first need to upper bound the following two terms:
			\begin{equation}\label{eq:constant_or_epsilon}
				2\sqrt{\frac{2\ln\p{16/\delta}}{n}} \le \frac2r\sqrt{\frac{2\ln\p{16/\delta}}{r}} \le \frac{2\epsilon}{c},
			\end{equation}
			where in the above, the first inequality is by $n\ge r^3$ for $c>1$ and the second inequality is by our assumption that $r\ge2\ln(16/\delta)$ when $c$ is sufficiently large, and that $r\ge c\epsilon^{-5} \ge c\epsilon^{-1}$. Next, we upper bound the term
			\begin{align*}
				4r\sqrt{\frac{(8r+8)\log_2(n)}{n}} &= 4\sqrt{\frac{(8r+8)\log_2\p{cr^3\log_2^2(r)}}{cr\log_2^2(r)}}\\
				&\le \frac{16}{\sqrt{c}}\sqrt{\frac{\log_2\p{cr^3\log_2^2(r)}}{\log_2^2(r)}}\\
				&= \frac{16}{\sqrt{c}}\sqrt{\frac{\log_2^2(c) + 3\log_2(r) + 2\log_2(\log_2(r)) }{\log_2^2(r)}}.
			\end{align*}
			By the change of variables $x=\log_2(r)$ we have that the above equals
			\[
				\frac{16}{\sqrt{c}}\sqrt{\frac{\log_2^2(c) + 3x + 2\log_2(x) }{x^2}} \le \frac{16}{\sqrt{c}}\sqrt{\log_2^2(c) + 3 + 1 } \le \frac{16}{c^{0.25}}\sqrt{5+3+1} \le \frac{48}{c^{0.25}},
			\]
			where the first inequality follows by our assumption $r\ge2$ which implies $\log_2(r)\ge1$ and thus $3/x\le3$, and since $\sup_x 2\log_2(x)/x^2\le1$; and the second inequality follows from $\sup_x \log_2^2(x)/\sqrt{x}\le5$. Combining this with \eqref{eq:constant_or_epsilon}, we can thus take $c>0$ sufficiently large so that the left-hand side of \eqref{eq:alpha_beta} is smaller than any constant, thus satisfying the inequality in the aforementioned equation. Moreover, in the above we have also shown (when combined with \eqref{eq:constant_or_epsilon} again) that
			\[
				8\sqrt{\frac{(8r+8)\log_2(n)}{n}} + 2\sqrt{\frac{2\ln\p{16/\delta}}{n}} \le \frac{96}{c^{0.25}r} + \frac{2\epsilon}{c} \le \frac{96\epsilon}{c^{1.25}} + \frac{2\epsilon}{c},
			\]
			where the second inequality uses our assumption in the theorem statement that $r\ge c\epsilon^{-1}$. Thus, for sufficiently large $c>0$ we have that the above is at most $\epsilon/4$ which upper bounds the last two summands in the left-hand side of \eqref{eq:9_regions_bound} and also implies \eqref{eq:rademacher_gen_bound}. It only remains to show that $40p+2/r\le\epsilon/4$ to establish the last remaining inequality in \eqref{eq:9_regions_bound}. To this end, compute
			\[
				p = \frac{1800q(d)+500c_2\sqrt{\ln\p{64n/\delta}}}{r^{0.25}} = \frac{1800q(d)}{r^{0.25}} + 500c_2\sqrt{\frac{\ln\p{64n/\delta}}{\sqrt{r}}}.
			\]
			Bounding the term $1800q(d)/r^{0.25}$, we have from our assumption in the theorem statement $r\ge c\epsilon^{-4}q^4(d)$ that this is at most $1800c^{-0.25}\epsilon$. Using \eqref{eq:ln_n_over_r} to bound the term inside the square root, we obtain
			\[
				p\le 1800c^{-0.25}\epsilon + 500c_2\sqrt{\frac{21}{c^{0.4}} + \frac{1}{\sqrt{c}}}\cdot\epsilon,
			\]
			which for a sufficiently large $c>0$ is smaller than $\epsilon/320$ and therefore $40p\le\epsilon/8$. The $2/r$ term is also at most $\epsilon/8$ by our assumption, concluding the derivation of the inequalities in Eqs.~(\ref{eq:r_lower_bound}-\ref{eq:9_regions_bound}).
			
			Turning to the proof of the theorem, we can apply Thms.~(\ref{thm:random_features},\ref{thm:generalization}) and Propositions~(\ref{prop:gradient_dot_product_bound},\ref{prop:low_loss_attained}) with a confidence parameter of $\delta/4$ in each, by using Eqs.~(\ref{eq:r_lower_bound},\ref{eq:p_bound}), and further assume that their implications hold throughout the remainder of the proof. Note that this happens with probability of at least $1-\delta$ by a union bound.
			
			We first show that for some $t\in\{0,1,\ldots,T\}$, GD finds a point $\bw_t$ such that 
			\begin{equation}\label{eq:half_epsilon}
				\hat{F}(\bw_t)\le \frac{\epsilon}{2}.
			\end{equation}
			If the assumption in \propref{prop:gradient_dot_product_bound} is violated by $\bw_0$, then we have from \propref{prop:low_loss_attained} and \eqref{eq:9_regions_bound} that $\hat{F}(\bw_0)\le\frac{\epsilon}{2}$. Otherwise, we have from \thmref{thm:well-behaved_random_features} that $\nu=\frac{\alpha^4\beta}{840}\le1$, and thus from \propref{prop:gradient_dot_product_bound} and \eqref{eq:alpha_beta} we get
			\[
				\inner{\nabla F(\bw_0),\bw_0-\bv} \ge \nu\norm{\bw_0-\bv}^2,
			\]
			With the above inequality, we can follow a similar approach as in \citet[Thm.~5.3(2)]{yehudai2020learning} and deduce that for $t=0$,
			\begin{align*}
				\norm{\bw_{t+1} - \bv}^2 &= \norm{\bw_t-\eta\nabla F(\bw_t)-\bv}^2\\
				&= \norm{\bw_{t} - \bv}^2-2\eta\inner{\nabla F(\bw_t),\bw_t-\bv} +\eta^2\norm{\nabla F(\bw_t)}^2\\
				&\le\norm{\bw_t-\bv}^2(1-\eta\nu) + \eta^2\norm{\nabla F(\bw_t)}^2.
			\end{align*}
			We shall now derive a bound on the norm of the gradient in the above expression as follows
			\begin{align*}
				\norm{\nabla F(\bw_t)} &=
				\norm{\frac2n\sum_{i=1}^n\p{\crelu{\bw_t^{\top}\tilde{\bx}_i} - y_i}\cdot \one{\bw_t^{\top}\tilde{\bx}_i\in(0,1)} \cdot\tilde{\bx}_i}\\
				&\le\frac2n\sum_{i=1}^n\abs{\p{\crelu{\bw_t^{\top}\tilde{\bx}_i} - y_i}\cdot \one{\bw_t^{\top}\tilde{\bx}_i\in(0,1)}} \cdot\norm{\tilde{\bx}_i} \le 2\sqrt{r}.
			\end{align*}
			By our assumption $\eta<\frac{\nu}{8r}$ we have
			\[
				1-\eta\nu+4\eta^2r < 1- \frac{\eta\nu}{2}<1,
			\]
			which implies
			\begin{equation}\label{eq:iter}
				\norm{\bw_{t+1}-\bv}^2 \le (1-\eta\nu + 4\eta^2r)\norm{\bw_{t}-\bv}^2 \le \p{1-\frac{\eta\nu}{2}}\norm{\bw_{t}-\bv}^2.
			\end{equation}
			In particular $\norm{\bw_{t+1}-\bv}^2\le\norm{\bw_{t}-\bv}^2<\norm{\bv}^2$, and therefore
			\[
				\norm{\bw_{t+1}}\le\norm{\bw_{t+1}-\bv}+\norm{\bv} \le 2\norm{\bv} \le 2r^{-0.25}.
			\]
			Now, if the inequality in \eqref{eq:small_noise} is violated, then by \propref{prop:low_loss_attained} and \eqref{eq:9_regions_bound} we have that $\hat{F}(\bw_1)\le \frac{\epsilon}{2}$ and thus we have shown that \eqref{eq:half_epsilon} holds for $t=1$. Otherwise, we have that all the conditions in \propref{prop:gradient_dot_product_bound} hold for $\bw_{t+1}$, and we can therefore apply both propositions iteratively until \eqref{eq:small_noise} is violated and \eqref{eq:half_epsilon} holds for some $t<T$, or until $T$ iterations of GD have been performed. In the latter case, we can apply \eqref{eq:iter} iteratively to obtain
			\[
				\norm{\bw_T-\bv}^2 \le 	\p{1-\frac{\eta\nu}{2}}^{T}\norm{\bw_0-\bv}^2\le \exp\p{-\ln\p{\frac{r}{8\epsilon}}}\norm{\bw_0-\bv}^2 \le \frac{\epsilon}{8r},
			\]
			where the second inequality is due to $T=2\eta^{-1}\nu^{-1}\ln\p{r/8\epsilon}$ and the inequality $(1-1/x)^x\le\exp(-1)$ which holds for all $x\ge1$, and the last inequality is due to $\norm{\bw_t},\norm{\bv}\le0.5$ and the triangle inequality. Using the above, we can compute
			\begin{align*}
				\frac1n\sum_{i=1}^{n}\p{\crelu{\bw_T^{\top}\tilde{\bx}_i}- \crelu{\bv^{\top}\tilde{\bx}_i}}^2 &\le \frac1n\sum_{i=1}^{n}\p{\bw_T^{\top}\tilde{\bx}_i - \bv^{\top}\tilde{\bx}_i}^2\\
				& \le \frac1n\sum_{i=1}^{n}\norm{\bw_T-\bv}^2\cdot\norm{\tilde{\bx}_i}^2 \le \frac{\epsilon}{8},
			\end{align*}
			where the first inequality uses the fact that $\crelu{\cdot}$ is $1$-Lipschitz and the second inequality is by Cauchy-Schwartz. We can now bound the squared loss of $\bv$ by using \thmref{thm:random_features} to get
			\[
				\frac1n\sum_{i=1}^{n}\p{\crelu{\bv^{\top}\tilde{\bx}_i} - y_i}^2 
				\le
				\frac1n \sum_{i:\norm{\bx_i}\in I_2\cup I_4}1
				\le
				\frac{3600q(d)+1000c_2\sqrt{\ln\p{32n/\delta}}}{r^{0.25}}\le\frac{\epsilon}{8},
			\]
			where the last inequality is implied by \eqref{eq:p_bound}. Combining the last two displayed inequalities using the inequality $(a+b)^2\le2a^2+2b^2$ with $a=\crelu{\bw_T^{\top}\tilde{\bx}_i}- \crelu{\bv^{\top}\tilde{\bx}_i}$ and $b=\crelu{\bv^{\top}\tilde{\bx}_i} - y_i$, we have that
			\[
				\hat{F}(\bw_T) \le \frac{\epsilon}{2}.
			\]
			We finish the proof by showing that this results in a generalization bound. Using \thmref{thm:generalization} and Eqs.~(\ref{eq:rademacher_gen_bound},\ref{eq:half_epsilon}), we have for some $t\in\{0,1,\ldots,T\}$ that
			\[
				F(\bw_t) \le \hat{F}(\bw_t) + 4\sqrt{\frac rn} + 4\sqrt{\frac{2\ln(16/\delta)}{n}} \le \frac{\epsilon}{2} + \frac{\epsilon}{2} = \epsilon.
			\]
			The proof of the theorem is complete.
		\end{proof}

		\section{Learning Rate Stability is Insufficient for Efficient Learning}\label{app:stability}
		
		In this appendix, we show that although the stability of the learning rate is sufficient to guarantee the asymptotic convergence of GD, it is not sufficient for convergence in polynomial time. For an $L$-smooth function $f$, it is well-known that GD with a fixed and stable step size $\eta$, meaning that $\eta<\frac2L$, converges to a local minimum asymptotically (e.g.\ \citep[Thm.~2.1.14]{nesterov2018lectures}). Moreover, there are examples for which no non-trivial guarantee can be given if $\eta\ge\frac2L$. We now provide an example of a $2$-smooth function where $T\ge4$ iterations of GD with a step size of $\eta>\frac2L-\frac1T$ result in only constant progress towards the global minimum.
		
		Consider the objective $f(x)=x^2$ with the global minimum $x^*=0$. Then the gradient step update for $\eta>\frac2L-\frac1T=1-\frac1T\ge\frac34$ is given by
		\[
			x_{t+1}=x_t-2\eta x_t=(1-2\eta)x_t.
		\]
		Iteratively applying the above recursion we obtain
		\[
			\abs{x_T-x^*}=\abs{x_T}=\abs{(1-2\eta)^Tx_0} = (2\eta-1)^T\abs{x_0}\ge \p{1-\frac2T}^T\abs{x_0} \ge \frac{1}{16}\abs{x_0}.
		\]
		In the above, the third equality is due to $\eta\ge\frac34$ which guarantees that $2\eta-1>0$, the first inequality follows from the fact that $(2\eta-1)^T$ is increasing for all $\eta>0.5$, and the last inequality is due to $(1-2/x)^x\ge\frac{1}{16}$ for all $x\ge4$.
		
		It follows from the above that GD with a step size larger than $\frac2L-\frac1T$ cannot get to less than a constant fraction of its initialization distance from the global minimizer. Thus, even if we allow the step size to be just a small margin away from the edge of stability $2/L$, convergence cannot be guaranteed in general. This results in the necessary condition that $\eta\le\frac2L-\frac1T$ which is a stronger assumption than merely requiring that the learning rate is stable. We also remark that while the example given in this appendix is of a strongly convex function, our analysis holds locally and therefore also applies in the non-convex regime, given that the objective possesses a local minimum with a finite smoothness parameter $L$ in its neighborhood.
		
		%\subsection{Training the Hidden Layer Doesn't Hurt}
		
		%\note{Here we will hopefully show that training the hidden layer using the same step size (and assuming the hidden layer is wide enough) then we still converge as if it remains fixed.}

\end{document}